\documentclass[letterpaper]{article} 
\usepackage{aaai2027}  
\usepackage[hyphens]{url}  
\usepackage{graphicx} 
\usepackage{natbib}  
\usepackage{caption} 
\usepackage{algorithm}
\usepackage{algorithmic}

\usepackage{newfloat}
\usepackage{listings}
\DeclareCaptionStyle{ruled}{labelfont=normalfont,labelsep=colon,strut=off} 
\floatstyle{ruled}
\newfloat{listing}{tb}{lst}{}
\floatname{listing}{Listing}

\usepackage{amsmath}
\usepackage{amsthm}
\usepackage{amssymb}
\usepackage{stackengine}
\usepackage{tikz}
\usetikzlibrary{arrows.meta, positioning}

\makeatletter
\newtheorem*{rep@theorem}{\rep@title}
\newcommand{\newreptheorem}[2]{%
\newenvironment{rep#1}[1]{%
 \def\rep@title{#2 \ref{##1}}%
 \begin{rep@theorem}}%
 {\end{rep@theorem}}}
\makeatother

\newtheorem{theorem}{Theorem}
\newreptheorem{theorem}{Theorem}
\newtheorem{definition}[theorem]{Definition}

\def\delequal{\mathrel{\ensurestackMath{\stackon[1pt]{=}{\scriptstyle\Delta}}}}

\usepackage{booktabs}

\title{General Probabilities of Causation with Causal Knowledge}
\author {
    Xin Shu,
    Zhen Lei,
    Ang Li\corresponding
}
\affiliations {
    Department of Computer Science, Florida State University\\
    xs24a@fsu.edu, zl24a@fsu.edu, angli@cs.fsu.edu
}

\begin{document}

\maketitle

\begin{abstract}
Probabilities of causation (PoCs) characterize individual causal responses that cannot be directly observed and therefore generally require partial identification. Tian and Pearl first derived theoretically sharp bounds for binary PoCs, including the probability of necessity (PN), the probability of sufficiency (PS), and the probability of necessity and sufficiency (PNS). Mueller et al. subsequently tightened the bounds for binary PNS by incorporating causal information encoded in covariates and mediators. More recently, Li and Pearl, as well as Shu et al., extended PoCs to multivalued settings and derived corresponding theoretical bounds. These developments naturally raise the question of whether additional causal knowledge can further tighten the bounds in multivalued settings. This paper addresses this question by deriving tighter bounds for multivalued PoCs through the incorporation of causal information encoded in covariates and mediators. We illustrate the theoretical results with toy examples, while simulation studies further demonstrate that the proposed bounds are tighter than existing nonbinary bounds.

\end{abstract}

\section{Introduction}
The modern study of probabilities of causation (PoCs) began with ~\citet{pearl1999probabilities}, who introduced three counterfactual measures of causation for binary treatments and outcomes within the Structural Causal Model (SCM) framework~\cite{galles1998axiomatic,halpern2000axiomatizing}: $\mathrm{PN}$, $\mathrm{PS}$, and $\mathrm{PNS}$.  \citet{tian2000probabilities} subsequently derived theoretically sharp bounds for these quantities by combining experimental and observational data using Balke’s linear programming framework~\cite{balke1995probabilistic}. 
\citet{li2019unit,li2022unit} later 
applied them to unit-level utility analysis.
Although the bounds derived by Tian and Pearl are sharp given the available distributional information, subsequent research has shown that incorporating covariates, mediators, and structural restrictions encoded by causal graphs can yield narrower bounds on probabilities of causation~\cite{kuroki2011statistical,dawid2017probability,mueller2021causeseffectslearningindividual}.

Recent research has extended probabilities of causation to settings with nonbinary treatments and outcomes. \citet{zhang2022partial} and \citet{li2022bounds} developed numerical methods based on nonlinear programming for computing bounds on nonbinary probabilities of causation. \citet{li2024probabilities} subsequently derived recursive theoretical bounds. More recently, \citet{shu2026identificationprobabilitiescausationrecursive} derived closed form bounds by combining experimental and observational data. 

Despite these advances, whether causal graph information can further tighten the bounds of PoCs in nonbinary settings remains an open question. The answer does not follow immediately from the binary case, because a multivalued PoC involves a joint probability of several potential outcomes and therefore introduces a substantially more complex counterfactual structure. Moreover, the information contributed by an auxiliary variable depends on whether it is a pretreatment covariate, a partial mediator, or a pure mediator.


Consider a pharmaceutical company developing a new drug to prevent diabetes, which is increasingly common among younger populations. Its effectiveness can be evaluated using drug-use frequency and diabetes status, both of which may take multiple values. The probability of necessity and sufficiency ($\mathrm{PNS}$) characterizes such individual response patterns by jointly considering the individual's potential diabetes outcomes under different frequencies of taking the drug. However, such bounds may remain wide (e.g., $[0.1,0.9]$) and therefore be of limited use for decision making if they do not incorporate the causal structure underlying an individual's response. In practice, additional variables often encode information about this causal structure.
For example, lifestyle regularity may influence both the frequency of taking the drug and diabetes status, appetite may partially mediate the effect of the drug, and blood glucose regulation may fully mediate this effect. A causal graph distinguishes these roles and imposes corresponding restrictions on feasible counterfactual response patterns. By incorporating these restrictions together with the available experimental and observational distributions, causal graph information can rule out response patterns that are mathematically possible but structurally incompatible, thereby yielding narrower and more informative $\mathrm{PNS}$ bounds.

Building on the nonbinary bounds of \citet{shu2026identificationprobabilitiescausationrecursive}, this paper develops bounds with multivalued treatments and outcomes by incorporating causal graph information. We exploit causal graphs involving non-descendant covariates, including confounders and backdoor adjustment sets \cite{pearl1993aspects, pearl1995causal}, as well as partial and pure mediators, to impose structural constraints on the feasible joint counterfactual distributions. For each causal structure, we derive bounds by integrating the available experimental and observational distributions with graph-based constraints on the feasible joint counterfactual distributions. These results extend the use of causal structure from binary to multivalued settings and show when auxiliary variables can exclude structurally incompatible response patterns, thereby tightening existing bounds on individual causal responses.

\subsection{Key Contributions}
\begin{itemize}
    \item We extend existing bounds on PoCs that incorporate causal knowledge from the binary to the multivalued setting. Specifically, we derive bounds not only for $\mathrm{PNS}$ but also for $\mathrm{PSub}$, $\mathrm{PRep}$, and $\mathrm{PN}$ for general treatment and outcome settings. Unlike the binary case, the multivalued setting requires characterizing joint counterfactual probabilities across multiple treatment levels, introducing substantial new technical challenges while naturally including the binary setting as a special case.

    \item We derive results for the four representative forms of multivalued PoCs introduced by \citet{shu2026identificationprobabilitiescausationrecursive}, which are sufficient to represent any discrete PoC. Consequently, the proposed framework immediately applies to a broad class of multivalued PoCs, substantially extending the scope and theoretical impact of incorporating causal knowledge.

    \item We consider multiple causal structures involving covariates that are not descendants of the treatment, covariates satisfying the back-door criterion, partial mediators, and pure mediators.
\end{itemize}

\section{Preliminaries}\label{related work}
In this section, we review the basic concepts of probabilities of causation (PoC) and the relevant notation used in Structural Causal Models (SCMs)~\cite{galles1998axiomatic,halpern2000axiomatizing}. $Y_x=y$ denotes the counterfactual statement that ``variable $Y$ would have the value $y$ had $X$ been set to $x$.'' Here, $X$ denotes the treatment variable, while $Y$ denotes the outcome or effect. For simplicity, we use $y_x$ as shorthand for $Y_x=y$. Since this paper considers nonbinary settings, we use ${y_j}_{x_j}$ to denote $Y_{x_j}=y_j$ throughout. In general, $P({y_j}_{x_j})$ refers to a probability obtained from experimental data, whereas $P({x_j},{y_j})$ refers to a probability obtained from observational data.

Let $X$ and $Y$ be binary variables in a causal model $M$. Let $x$ and $y$ denote the events $X=\mathrm{true}$ and $Y=\mathrm{true}$, respectively, and let $x'$ and $y'$ denote the events $X=\mathrm{false}$ and $Y=\mathrm{false}$, respectively. The three basic probabilities of causation are defined below~\cite{pearl1999probabilities}:
\begin{definition}[Probabilities of Causation] \cite{pearl1999probabilities}
\begin{align*}
\text{PNS}\delequal P(y_x,y'_{x'})
\text{PN} \delequal  P(y'_{x'}|x,y),
\text{PS} \delequal P(y_x|y',x'),
\end{align*}
\end{definition}

In the multivalued setting, let the treatment variable $X$ and the outcome variable $Y$ each take $n$ values. ${PNS}(k)=P({y_1}_{x_1},\ldots,{y_k}_{x_k})$, where $k\le n$, was first introduced by \citet{li2024probabilities}, represents the probability that the same individual would exhibit outcome $y_j$ if the treatment were set to $x_j$ for every $j\in\{1,\ldots,k\}$. It therefore characterizes a joint individual causal response across multiple treatment levels.
\begin{align*}
&\textbf{Probability of necessity and sufficiency(k):}\\
&\mathrm{PNS}(k) = P({y_{1}}_{x_{1}},...,{y_{k}}_{x_{k}}) \\
&\textbf{Probability of substitute(k, p): } p \ne j \text{ for } 1\le j \le k\\
&\mathrm{PSub}(k,p) = P({y_{1}}_{x_{1}},...,{y_{k}}_{x_{k}},x_p)\\
&\textbf{Probability of replacement(k, q):}\\
&\mathrm{PRep}(k,q) = P({y_{1}}_{x_{1}},...,{y_{k}}_{x_{k}},y_q) \\
&\textbf{Probability of necessity(k, p, q): } p \ne j \text{ for } 1\le j \le k\\
&\mathrm{PN}(k,p,q) = P({y_{1}}_{x_{1}},...,{y_{k}}_{x_{k}},x_p,y_q)
\end{align*}

Later, \citet{shu2026identificationprobabilitiescausationrecursive} derived closed-form bounds for nonbinary probabilities of causation. The four main theorems are presented above. Further details, together with the equivalence and replaceability theorems, are omitted here because of space limitations and can be found in \citet{shu2026identificationprobabilitiescausationrecursive}.

In this paper, we derive narrower bounds for these four PoCs by incorporating information from causal graphs involving covariates and mediators. Under the specified graphical conditions, the proposed nonbinary bounds tighten existing bounds. We establish this improvement mathematically and use simulations to demonstrate and quantify the extent of the tightening.

\section{Bounds with Causal Diagram}
In this section, we assume that variables $X$ and $Y$ each take $n$ values, so that $|X|=|Y|=n$, with $k\leq n$. The generalization was formally established by \citet{shu2026identificationprobabilitiescausationrecursive}. Proofs of the following theorems are provided in the appendix.

For each graphical structure, we derive improved versions of all four existing bounds for multivalued probabilities of causation ($\mathrm{PNS}(k)$, $\mathrm{PSub}(k,p)$, $\mathrm{PRep}(k,q)$, $\mathrm{PN}(k,p,q)$).

\subsection{Non-descendant Covariates}

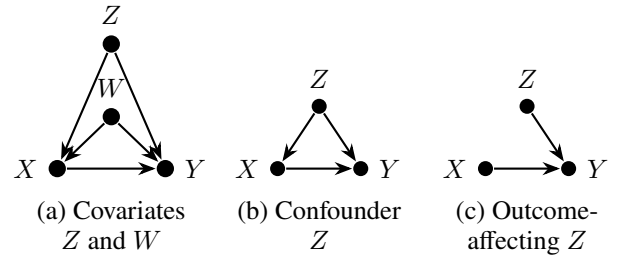
\begin{figure}[!b]
\centering
\begin{tikzpicture}[>=Stealth, thick, scale=0.55]
\begin{scope}[shift={(-5,0)}]
    \node[circle, fill=black, inner sep=2.3pt,
          label=above:$Z$] (Za) at (0,3) {};
    \node[circle, fill=black, inner sep=2.3pt,
          label=above:$W$] (Wa) at (0,1.25) {};
    \node[circle, fill=black, inner sep=2.3pt,
          label=left:$X$] (Xa) at (-1.3,0) {};
    \node[circle, fill=black, inner sep=2.3pt,
          label=right:$Y$] (Ya) at (1.3,0) {};

    \draw[->] (Za) -- (Xa);
    \draw[->] (Za) -- (Ya);
    \draw[->] (Wa) -- (Xa);
    \draw[->] (Wa) -- (Ya);
    \draw[->] (Xa) -- (Ya);

    \node[below, align=center] at (0,-0.55)
          {(a) Covariates\\ \(Z\) and \(W\)};
\end{scope}

\begin{scope}[shift={(0,0)}]
    \node[circle, fill=black, inner sep=2pt,
          label=above:$Z$] (Zb) at (0,1.5) {};
    \node[circle, fill=black, inner sep=2pt,
          label=left:$X$] (Xb) at (-1,0) {};
    \node[circle, fill=black, inner sep=2pt,
          label=right:$Y$] (Yb) at (1,0) {};

    \draw[->] (Zb) -- (Xb);
    \draw[->] (Zb) -- (Yb);
    \draw[->] (Xb) -- (Yb);

    \node[below, align=center] at (0,-0.55)
          {(b) Confounder \\ \(Z\)};
\end{scope}

\begin{scope}[shift={(5,0)}]
    \node[circle, fill=black, inner sep=2pt,
          label=above:$Z$] (Zc) at (0,1.5) {};
    \node[circle, fill=black, inner sep=2pt,
          label=left:$X$] (Xc) at (-1,0) {};
    \node[circle, fill=black, inner sep=2pt,
          label=right:$Y$] (Yc) at (1,0) {};

    \draw[->] (Xc) -- (Yc);
    \draw[->] (Zc) -- (Yc);

    \node[below, align=center] at (0,-0.55)
          {(c) Outcome-\\ affecting \(Z\)};
\end{scope}

\end{tikzpicture}

\caption{In all three graphs, \(Z\) contains no descendants of
\(X\). In graphs (b) and (c), \(Z\) additionally satisfies the
back-door criterion.}
\label{fig:z-not-descendant}
\end{figure}

Given a causal diagram $G$, we restrict $Z$ to be a set of variables containing no descendants of $X$ in $G$. Fig.~\ref{fig:z-not-descendant} (a) illustrates such a causal structure. This assumption ensures that $Z$ is not affected by interventions on $X$ and allows the conditional quantity $P({y_j}_{x_j}\mid z)$ to be measured. Here, unless otherwise specified, $j \in \{1,\ldots,n\}$ throughout the rest of the paper.

\begin{theorem}\label{nnk1}
\begin{align*}
\sum_{z}{\max \left \{
\begin{array}{cc}
0, \\
\\
\displaystyle \sum_{j = 1}^{k}P({y_{j}}_{x_{j}}\mid z) - k + 1, \\
\\
\text{For }i \in  \{1, ..., k\}:\\
\quad \displaystyle \sum_{\substack{1 \le j \le k \\ j \ne i}}
\Big[P({y_{j}}_{x_{j}}\mid z)+\\
+P({x_{j}\mid z})-P({x_{j}, y_{j}}\mid z)\Big]+\\
+ P({x_{i}, y_{i}}\mid z) - k + 1
\end{array}
\right \}\nonumber
} \times P(z) \\\le \mathrm{PNS}(k)
\end{align*}
\begin{align*}
\sum_{z}{\min \left \{
\begin{array}{cc}
\displaystyle \sum_{j = 1}^{k}P({x_{j}, y_{j}}\mid z) +\\
+\displaystyle\sum_{j = k+1}^{n}P({x_{j}}\mid z), \\
\\
\text{For }j \in \{1,...,k\}:\\
P({y_{j}}_{x_{j}}\mid z),\\
\\
\text{For }m\in\{1,...,k-1\},\\
t_j\in\{1,...,k\}:\\
\displaystyle \frac{1}{m}\Big[\sum_{j = 0}^{m}P({y_{t_j}}_{x_{t_j}}\mid z) - \\
-P({x_{t_j}, y_{t_j}\mid z})\Big]
\end{array} 
\right \}\nonumber
} \times P(z) \\ \ge \mathrm{PNS}(k)
\end{align*}
\end{theorem}

Here, we derive bounds on the conditional z-specific $\mathrm{PNS}(k)$ for each possible value $z$ and then take a weighted sum of these stratum-specific bounds to obtain tighter bounds on the $\mathrm{PNS}(k)$. Bounds on $\mathrm{PSub}(k), \mathrm{PRep}(k), \mathrm{PN}(k)$ can be derived similarly.

\begin{theorem} \label{nnk+x_p1}
$p \ne j$ for $1\le j \le k$,
\begin{align*}
\sum_{z}{\max \left \{
\begin{array}{cc}
0, \\
\\
\displaystyle \sum_{j=1}^{k}\Big[P({y_{j}}_{x_{j}}\mid z)+P({x_{j}}\mid z)-\\-P({x_{j}, y_{j}}\mid z)\Big] + P({x_{p}}\mid z) - k
\end{array}
\right \}\nonumber
} \times P(z) \\ \le \mathrm{PSub}(k,p)
\end{align*}
\begin{align*}
\sum_{z}{\min \left \{
\begin{array}{cc}
P({x_{p}}\mid z),\\
\\
\text{For } j \in \{1,...,k\}:\\
P({y_{j}}_{x_{j}}\mid z) - P({x_{j}, y_{j}}\mid z)\\
\end{array} 
\right \}\nonumber
} \times P(z) \\ \ge \mathrm{PSub}(k,p)\\
\end{align*}
\end{theorem}

\begin{theorem} \label{nnk+y_q1}
\begin{align*}
\sum_{z}{\max \left \{
\begin{array}{cc}
0, \\
\\
\displaystyle \sum_{j=1}^{k}\Big[P({y_{j}}_{x_{j}}\mid z)+\\
+P({x_{j}}\mid z)-P({x_{j}, y_{j}}\mid z)\Big]+\\
+\displaystyle \sum_{\substack{k+1 \le j \le n \\ j \ne q}}{P({x_{j}, y_{q}}\mid z)} +\\+ P({x_{q}, y_{q}}\mid z) - k,\\
\\
\text{If } q \in \{1,...,k\}:\\
\displaystyle \sum_{\substack{1 \le j \le k \\ j \ne q}}\Big[P({y_{j}}_{x_{j}}\mid z)+\\
+P({x_{j}}\mid z)-P({x_{j}, y_{j}}\mid z)\Big]+\\
+ P({x_{q}, y_{q}}\mid z) - (k-1) 
\end{array}
\right \}\nonumber
} \times P(z) \\ \le \mathrm{PRep}(k,q)
\end{align*}
\begin{align*}
\sum_{z}{\min \left \{
\begin{array}{cc}
\text{If } q \in \{1,...,k\}:\\
P({y_{q}}_{x_{q}}\mid z), \\
\\
P({x_{q}, y_{q}}\mid z) +\\
+\displaystyle \sum_{\substack{k+1 \le j \le n \\ j \ne q}}P({x_{j}, y_{q}}\mid z),\\
\\
\text{For }j \in \{1,...,k\},\text{ and } j \ne q\\
P({y_{j}}_{x_{j}}\mid z) - P({x_{j}, y_{j}}\mid z),\\
\end{array} 
\right \}\nonumber
} \times P(z) \\ \ge \mathrm{PRep}(k,q)\\
\end{align*}
\end{theorem}

\begin{theorem} \label{nnk+x_p+y_q1}
$p \ne j$ for $1\le j \le k$,
\begin{align*}
\sum_{z}{\max \left \{
\begin{array}{cc}
0, \\
\\
\displaystyle \sum_{j=1}^{k}\Big[P({y_{j}}_{x_{j}}\mid z)+\\
+P({x_{j}}\mid z)-P({x_{j}, y_{j}}\mid z)\Big] +\\
+P({x_{p}, y_{q}}\mid z) - k\\
\end{array}
\right \}\nonumber
} \times P(z) \\ \le \mathrm{PN}(k,p,q)
\end{align*}
\begin{align*}
\sum_{z}{\min \left \{
\begin{array}{cc}
P({x_{p}},{y_{q}}\mid z), \\
\\
\text{For }j \in \{1,...,k\}:\\
P({y_{j}}_{x_{j}}\mid z) - P({x_{j}, y_{j}}\mid z) \\
\end{array} 
\right \}\nonumber
} \times P(z)\\ \ge \mathrm{PN}(k,p,q)
\end{align*}
\end{theorem}

When $Z$ not only contains no descendants of $X$ in $G$ but also satisfies the back-door criterion \cite{pearl1993aspects, pearl1995causal} (Figures~\ref{fig:z-not-descendant}(b) and \ref{fig:z-not-descendant}(c) illustrate two examples of this setting), the back-door adjustment formula gives 
\begin{align*}
P({y_j}_{x_j}\mid z)=P(y_j\mid x_j,z).
\end{align*}
Substituting this equality into Theorems~\ref{nnk1}--\ref{nnk+x_p+y_q1} yields four new theorems that require only observational data, without the need for experimental data. For instance, the upper bound of PN(k,p,q) becomes
\begin{align*}
\sum_{z}{\min \left \{
\begin{array}{cc}
P({x_{p}},{y_{q}}\mid z), \\
\\
\text{For }j \in \{1,...,k\}:\\
P({y_{j}}\mid{x_{j}}, z) - P({x_{j}, y_{j}}\mid z) \\
\end{array} 
\right \}\nonumber
} \times P(z)\\ \ge \mathrm{PN}(k,p,q)
\end{align*}

\subsection{Mediation}
When $Z$ is a descendant of $X$, the previous results no longer apply. However, under certain mediation structures, information about $Z$ can still tighten the upper bound.
\begin{figure}[t]
\centering
\begin{tikzpicture}[>=Stealth, thick, scale=0.55]

\begin{scope}[shift={(-3.2,0)}]
    \node[
        circle,
        fill=black,
        inner sep=2pt,
        label=above:$U$
    ] (Ua) at (-1.2,1.8) {};

    \node[
        circle,
        fill=black,
        inner sep=2pt,
        label=above:$Z$
    ] (Za) at (0.8,1.8) {};

    \node[
        circle,
        fill=black,
        inner sep=2pt,
        label=left:$X$
    ] (Xa) at (-0.5,0) {};

    \node[
        circle,
        fill=black,
        inner sep=2pt,
        label=right:$Y$
    ] (Ya) at (1.8,0) {};

    \draw[dashed,->] (Ua) -- (Xa);
    \draw[dashed,->] (Ua) -- (Za);
    \draw[->] (Xa) -- (Za);
    \draw[->] (Za) -- (Ya);
    \draw[->] (Xa) -- (Ya);

    \node[below] at (0.3,-0.6)
        {(a) With direct effect};
\end{scope}

\begin{scope}[shift={(3.2,0)}]
    \node[
        circle,
        fill=black,
        inner sep=2pt,
        label=above:$Z$
    ] (Zb) at (0,1.5) {};

    \node[
        circle,
        fill=black,
        inner sep=2pt,
        label=left:$X$
    ] (Xb) at (-1,0) {};

    \node[
        circle,
        fill=black,
        inner sep=2pt,
        label=right:$Y$
    ] (Yb) at (1,0) {};

    \draw[->] (Xb) -- (Zb);
    \draw[->] (Zb) -- (Yb);

    \node[below] at (0,-0.6)
        {(b) With no direct effect};
\end{scope}

\end{tikzpicture}

\caption{In both graphs, \(Z\) is a mediator: (a) with a direct
effect and (b) with no direct effect.}
\label{med}
\end{figure}
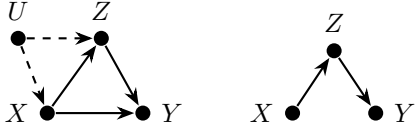

\subsubsection{Partial Mediator}
In Fig.~\ref{med}(a), when $Z$ is a mediator and $X$ also has a direct effect on $Y$, let $Z$ be a set of variables such that $\forall x_j, x_i \in X:\ x_j \neq x_i,
\quad
(Y_{x_j} \perp\!\!\!\perp X \cup Z_{x_i} \mid Z_{x_j})$,
the lower bound remains unchanged from that of \citet{shu2026identificationprobabilitiescausationrecursive}, while the upper bound includes an additional term. The proof of the additional term can be found in the appendix.




\begin{theorem} \label{nnk3}
Let the bounds on $\mathrm{PNS}(k)$ given in \citet{shu2026identificationprobabilitiescausationrecursive} be denoted by
\begin{align*}
    {PNS}_{{LB}}
    \leq \mathrm{PNS}(k)
    \leq
    {PNS}_{{UB}}.
\end{align*}
By treating $Z$ as a partial mediator, we can derive a new upper bound on $\mathrm{PNS}(k)$ as follows:

\begin{align*}
\min \left \{
\begin{array}{cc}
PNS_{UB},\\
\\
\text{For }j \in \{1,\dots,k\}:\\
\displaystyle \sum_{\substack{{z_1},...,{z_k} \\ \in \{1,\dots,n\}^{k}}}
\Big[{\min_j\{P({y_j}\mid {x_j},{z_j})\}} \times \\
\quad\times \displaystyle{\min_j \{P({z_j}_{x_j})\} }\Big]
\end{array} 
\right \}\nonumber
\ge \mathrm{PNS}(k)
\end{align*}
\end{theorem}

\begin{theorem} \label{nnk+x_p3}
Given $p \ne j$ for $1\le j \le k$. Let the bounds on $\mathrm{PSub}(k,p)$ given in \citet{shu2026identificationprobabilitiescausationrecursive} be denoted by
\begin{align*}
    {PSub}_{{LB}}
    \leq \mathrm{PSub}(k,p)
    \leq
    {PSub}_{{UB}}.
\end{align*}
By treating $Z$ as a partial mediator, we can derive a new upper bound on $\mathrm{PSub}(k,p)$ as follows:
\begin{align*}
\min \left \{
\begin{array}{cc}
PSub_{UB},\\
\\
\text{For }j \in \{1,...,k\}:\\
\displaystyle \sum_{\substack{{z_1},\dots,{z_k} \\ \in \{1,\dots,n\}^{k}}}
\Big[{\min_j\{P({y_j}\mid {x_j},{z_j})\}} \times \\
\times {\displaystyle \min_j \{P({z_j}_{x_j}), P(x_p)\} }\Big]
\end{array} 
\right \}\nonumber
\\ \ge \mathrm{PSub}(k,p)
\end{align*}
\end{theorem}

\begin{theorem} \label{nnk+y_q3}
Let the bounds on $\mathrm{PRep}(k,q)$ given in \citet{shu2026identificationprobabilitiescausationrecursive} be denoted by
\begin{align*}
    {PRep}_{{LB}}
    \leq \mathrm{PRep}(k,q)
    \leq
    {PRep}_{{UB}}.
\end{align*}
By treating $Z$ as a partial mediator, we can derive a new upper bound on $\mathrm{PRep}(k,q)$ as follows:
\begin{align*}
\min \left \{
\begin{array}{cc}
{PRep}_{UB},\\
\\
\text{For }j \in \{1,...,k\}:\\
\displaystyle \sum_{\substack{{z_1},...,{z_k} \\ \in \{1,...,n\}^{k}}}
\Big[{\min_j\{P({y_j}\mid {x_j},{z_j})\}}\times \\
\times \displaystyle{\min_j \{P({z_j}_{x_j})\} }\Big]
\end{array} 
\right \}\nonumber
\\ \ge \mathrm{PRep}(k,q)
\end{align*}
\end{theorem}

\begin{theorem} \label{nnk+x_p+y_q3}
Given $p \ne j$ for $1\le j \le k$. Let the bounds on $\mathrm{PN}(k,p,q)$ given in \citet{shu2026identificationprobabilitiescausationrecursive} be denoted by
\begin{align*}
    {PN}_{{LB}}
    \leq \mathrm{PN}(k,p,q)
    \leq
    {PN}_{{UB}}.
\end{align*}
By treating $Z$ as a partial mediator, we can derive a new upper bound on $\mathrm{PN}(k,p,q)$ as follows:
\begin{align*}
\min \left \{
\begin{array}{cc}
PN_{UB},\\
\\
\text{For }j \in \{1,...,k\}:\\
\displaystyle \sum_{\substack{{z_1},...,{z_k},{z_p} \\ \in \{1,...,n\}^{k+1}}}
\Bigg[\min_j\left\{
\begin{array}{c}
P(y_j\mid x_j,z_j),
\\
P(y_q\mid x_p,z_p)
\end{array}
\right\}\times \\
\times {\displaystyle \min_j \{P({z_j}_{x_j}),P({z_p}_{x_p}), P(x_p)\} }\Bigg]
\end{array} 
\right \}\nonumber
\\ \ge \mathrm{PN}(k,p,q)
\end{align*}
\end{theorem}

\subsubsection{Pure Mediator}

In Figure~\ref{med}(b), $X$ has no direct effect on $Y$ and affects $Y$ only through the mediator $Z$. The lower bound remains unchanged from that of \citet{shu2026identificationprobabilitiescausationrecursive}, while the upper bound includes an additional term using observational quantities involving $Z$.


\begin{theorem} \label{nnk4}
Let the bounds on $\mathrm{PNS}(k)$ given in \citet{shu2026identificationprobabilitiescausationrecursive} be denoted by
\begin{align*}
    {PNS}_{{LB}}
    \leq \mathrm{PNS}(k)
    \leq
    {PNS}_{{UB}}.
\end{align*}
By treating $Z$ as a pure mediator, we can derive a new upper bound on $\mathrm{PNS}(k)$ as follows:
\begin{align*}
\min \left \{
\begin{array}{cc}
{PNS}_{{UB}},\\
\\
\text{For }j \in \{1,...,k\}:\\
\displaystyle \sum_{\substack{{z_1}\neq...\neq{z_k} \\ \in \{1,...,n\}^{k}}}
\Big[{\min_j\{P({y_j}\mid {z_j})\}} \times \\ 
\qquad\quad\times \displaystyle{\min_j \{P({z_j}\mid {x_j})\} }\Big]
\end{array} 
\right \}\nonumber
\ge \mathrm{PNS}(k)
\end{align*}
\end{theorem}

\begin{theorem} \label{nnk+x_p4}
Given $p \ne j$ for $1\le j \le k$. Let the bounds on $\mathrm{PSub}(k,p)$ given in \citet{shu2026identificationprobabilitiescausationrecursive} be denoted by
\begin{align*}
    {PSub}_{{LB}}
    \leq \mathrm{PSub}(k,p)
    \leq
    {PSub}_{{UB}}.
\end{align*}
By treating $Z$ as a pure mediator, we can derive a new upper bound on $\mathrm{PSub}(k,p)$ as follows:
\begin{align*}
\min \left \{
\begin{array}{cc}
PSub_{UB},\\
\\
\text{For }j \in \{1,...,k\}:\\
\displaystyle \sum_{\substack{{z_1}\neq...\neq{z_k} \\ \in \{1,...,n\}^{k}}}
\Big[{\min_j\{P({y_j}\mid {z_j})\}} \times \\ 
\times \displaystyle{\min_j \{P({z_j}\mid {x_j})\} }\Big]\times P(x_p)
\end{array} 
\right \}\nonumber
\ge \mathrm{PSub}(k,p)
\end{align*}
\end{theorem}

\begin{theorem} \label{nnk+y_q4}
Let the bounds on $\mathrm{PRep}(k,q)$ given in \citet{shu2026identificationprobabilitiescausationrecursive} be denoted by
\begin{align*}
    {PRep}_{{LB}}
    \leq \mathrm{PRep}(k,q)
    \leq
    {PRep}_{{UB}}.
\end{align*}
By treating $Z$ as a pure mediator, we can derive a new upper bound on $\mathrm{PRep}(k,q)$ as follows:
\begin{align*}
\min \left \{
\begin{array}{cc}
{PRep}_{{UB}},\\
\\
\text{For }j \in \{1,...,k\}:\\
\displaystyle \sum_{\substack{{z_1}\neq...\neq{z_k} \\ \in \{1,...,n\}^{k}}}
\Big[{\min_j\{P({y_j}\mid {z_j})\}} \times \\ 
\qquad\quad\times {\min_j \{P({z_j}\mid {x_j})\} }\Big]
\end{array} 
\right \}\nonumber
\ge \mathrm{PRep}(k,q)
\end{align*}
\end{theorem}

\begin{theorem} \label{nnk+x_p+y_q4}
Given $p \ne j$ for $1\le j \le k$. Let the bounds on $\mathrm{PN}(k,p,q)$ given in \citet{shu2026identificationprobabilitiescausationrecursive} be denoted by
\begin{align*}
    {PN}_{{LB}}
    \leq \mathrm{PN}(k,p,q)
    \leq
    {PN}_{{UB}}.
\end{align*}
By treating $Z$ as a pure mediator, we can derive a new upper bound on $\mathrm{PN}(k,p,q)$ as follows:
\begin{align*}
\min \left \{
\begin{array}{cc}
PN_{UB},\\
\\
\text{For }j \in \{1,...,k\}:\\
\displaystyle \sum_{\substack{{z_1}\neq...\neq{z_k}\neq{z_p} \\ \in \{1,...,n\}^{k+1}}}
\Big[{\min_j\{P({y_j}\mid {z_j}),P({y_q}\mid {z_p})\}} \times \\
\times{\displaystyle\min_j \{P({z_j}\mid {x_j}),P({z_p}\mid {x_p})\} } \Big]\times P(x_p)
\end{array} 
\right \}\nonumber
\\ \ge \mathrm{PN}(k,p,q)
\end{align*}
\end{theorem}

\section{Examples}
In this section, for simplicity, we illustrate only the proposed $\mathrm{PNS}$ bounds under different causal structures and focus on three-dimensional cases.

Returning to the motivating example introduced in the Introduction, consider the pharmaceutical company that has developed a new drug claimed to help prevent diabetes. To illustrate the proposed bounds under different causal structures, we construct synthetic studies based on this setting.

Let $X$ denote how frequently an individual takes the drug, where $x_1$ indicates that the person takes the drug regularly and on time, $x_2$ indicates the person takes the drug occasionally, and $x_3$ indicates the person never takes the drug. Let $Y$ denote the individual's diabetes status,where $y_1$, $y_2$, and $y_3$ correspond to the person not developing diabetes, developing prediabetes, and developing diabetes, respectively. The target quantity is $\mathrm{PNS}(3) = P\left(y_{1x_1},y_{2x_2},y_{3x_3}\right)$,
which represents the proportion of individuals who would not develop diabetes if they took the drug regularly and on time, would develop prediabetes if they took it occasionally, and would develop diabetes if they never took it. This target quantity is chosen solely for ease of presentation; any joint counterfactual probability can be analyzed in the same manner using the proposed framework.

\subsection{Family History as a Covariate}\label{ex1}
We now incorporate family history as an observed covariate in the diabetes example. 
Let $Z$ denote an individual's family history of diabetes, where $z_1$, $z_2$, and $z_3$ correspond to no known family history of diabetes, diabetes only in a second-degree relative, and diabetes in a first-degree relative, respectively.
Individuals with a family history of diabetes may be more likely to develop the disease and may also be more likely to take preventive medication regularly, giving $Z\to Y$ and $Z\to X$. The corresponding causal structure is shown in Figure~\ref{fig:z-not-descendant}(b). In this setting, $Z$ also satisfies the back-door criterion. Therefore, in Theorem~\ref{nnk1}, all terms requiring experimental data can be replaced by their corresponding expressions based on observational data, yielding a modified version of the theorem that requires only observational data. The modified theorem is provided in the appendix.

To illustrate the calculation, we construct a synthetic study of 900 individuals. The sample was divided into three strata according to the levels of $Z$. The results are shown in Table~\ref{tb1}.

\begin{table}[t]
\centering
\begin{tabular}{|c|c|c|c|}
\hline
\multicolumn{4}{|c|}
{$Z=z_1$: No known family history}\\
\hline
$Y \backslash X$
& Regular use
& Occasional use
& No use\\
\hline
Diabetes-free & $46$ & $4$ & $54$\\
\hline
Prediabetes   & $12$ & $112$ & $120$\\
\hline
Diabetes      & $2$ & $4$ & $6$\\
\hline

\multicolumn{4}{c}{}\\
\hline
\multicolumn{4}{|c|}
{$Z=z_2$: Second-degree family history only}\\
\hline
$Y \backslash X$
& Regular use
& Occasional use
& No use\\
\hline
Diabetes-free & $3$ & $3$ & $2$\\
\hline
Prediabetes   & $75$ & $25$ & $2$\\
\hline
Diabetes      & $12$ & $2$ & $56$\\
\hline

\multicolumn{4}{c}{}\\
\hline
\multicolumn{4}{|c|}
{$Z=z_3$: First-degree family history}\\
\hline
$Y \backslash X$
& Regular use
& Occasional use
& No use\\
\hline
Diabetes-free & $192$ & $44$ & $4$\\
\hline
Prediabetes   & $24$ & $2$ & $2$\\
\hline
Diabetes      & $24$ & $14$ & $54$\\
\hline
\end{tabular}
\caption{Results of a drug study stratified by family history.}
\label{tb1}
\end{table}

Since $Z$ also satisfies back-door criterion, for $j\in \{1,2,3\}$, we can compute 
$P({y_j}_{x_j}) = \sum_z{P({y_j}\mid{x_j}, z)\times P(z)}$. 
So we can derive $P({y_1}_{x_1})=190/300, P({y_2}_{x_2})=166/300, P({y_3}_{x_3})=168/300$. 

The company emphasizes that the three interventional probabilities appearing in the target probability of causation,
$P({y_1}_{x_1})$, $P({y_2}_{x_2})$, and $P({y_3}_{x_3})$,
are all greater than $0.5$. Based on these marginal probabilities, the company claims that the drug is effective.

Applying the bounds on $\mathrm{PNS}(3)$ derived by \citet{shu2026identificationprobabilitiescausationrecursive} gives
\[
0\leq \mathrm{PNS}(3)\leq 0.551.
\]
This interval is too wide to determine whether the drug produces the target response pattern. 

By additionally incorporating information about family history, we obtain the tighter bounds (the detailed derivation is provided in the appendix):
\begin{align*}
0
\leq
\mathrm{PNS}(3)
\leq
0.033.
\end{align*}
The substantially smaller upper bound indicates that at most $3.3\%$ of the population would exhibit the target response pattern. Thus, the company's marginal comparisons provide limited evidence that the drug has the claimed effect.

\subsection{Blood glucose as a pure mediator}\label{ex2}
As in the previous example, let $X$ and $Y$ denote how frequently an individual takes the drug and the individual's diabetes status, respectively. A new investigation finds that the drug may affect an individual's blood glucose level and thereby help prevent the individual from developing diabetes. Accordingly, 
Let $Z$ denote the individual's intermediate blood glucose status measured after treatment but before the final diabetes assessment, where $z_1$, $z_2$, and $z_3$ correspond to normal, prediabetic, and diabetic blood glucose levels, respectively.
This setting corresponds to the causal structure shown in Figure~\ref{med}(b).

To illustrate the calculation under the pure-mediator structure, we construct a synthetic study of $2{,}700$ individuals. The synthetic observations are grouped according to the three levels of $Z$, as shown in Table~\ref{tb2}.

\begin{table}[t]
\centering
\begin{tabular}{|c|c|c|c|}
\hline
\multicolumn{4}{|c|}
{$Z=z_1$: normal blood glucose}\\
\hline
$Y \backslash X$
& Regular use
& Occasional use
& No use\\
\hline
Diabetes-free & $448$ & $304$ & $16$\\
\hline
Prediabetes   & $364$ & $247$ & $13$\\
\hline
Diabetes      & $28$  & $19$  & $1$\\
\hline

\multicolumn{4}{c}{}\\
\hline
\multicolumn{4}{|c|}
{$Z=z_2$: prediabetic blood glucose}\\
\hline
$Y \backslash X$
& Regular use
& Occasional use
& No use\\
\hline
Diabetes-free & $2$  & $16$  & $14$\\
\hline
Prediabetes   & $27$ & $216$ & $189$\\
\hline
Diabetes      & $1$  & $8$   & $7$\\
\hline

\multicolumn{4}{c}{}\\
\hline
\multicolumn{4}{|c|}
{$Z=z_3$: diabetic blood glucose}\\
\hline
$Y \backslash X$
& Regular use
& Occasional use
& No use\\
\hline
Diabetes-free & $6$  & $18$  & $132$\\
\hline
Prediabetes   & $3$  & $9$   & $66$\\
\hline
Diabetes      & $21$ & $63$  & $462$\\
\hline
\end{tabular}
\caption{Results of a drug study with blood glucose status.}
\label{tb2}
\end{table}

Applying the bounds on $\mathrm{PNS}(3)$ given by \citet{shu2026identificationprobabilitiescausationrecursive}, we obtain 
\[
0\leq \mathrm{PNS}(3)\leq 0.507.
\]
This interval is too wide to provide informative conclusions about the target response pattern.

By additionally incorporating information on blood glucose, we obtain the following tighter upper bound using Thm.~\ref{nnk4} (details are provided in the appendix):
\begin{align*}
0 \leq \mathrm{PNS}(3)\ \leq 0.151.
\end{align*}
The substantially smaller mediator-improved upper bound provides a substantially more informative characterization of the target response pattern.

\section{Simulated Results}
In this section, we empirically demonstrate that the proposed bounds, which incorporate information from the corresponding causal graphs, are tighter than the bounds derived by \citet{shu2026identificationprobabilitiescausationrecursive}. In the figures and tables that follow, ``Shu et al.'' refers to the bounds derived in that work. For simplicity, we focus on $\mathrm{PNS}$ to provide a clear visual illustration of the improvement achieved in practice.


\begin{figure}[b]
  \centering
  \includegraphics[width=\linewidth]{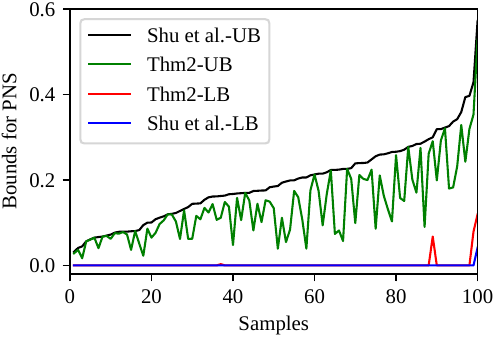}
  \caption{$\mathrm{PNS}(3)$ bounds under the non-descendant structure shown in Fig.~\ref{fig:z-not-descendant}(a).}
  \label{sim_pic1}
\end{figure}

\begin{figure}[t]
  \centering
  \includegraphics[width=\linewidth]{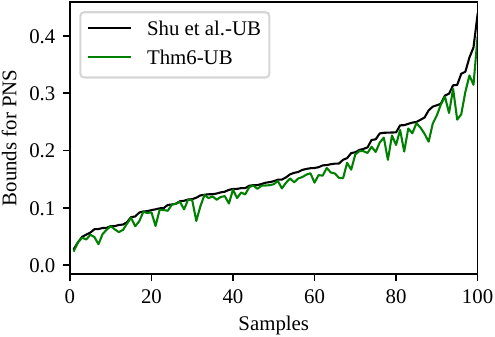}
  \caption{$\mathrm{PNS}(3)$ bounds under the partial mediator structure shown in Fig.~\ref{med}(a).}
  \label{sim_pic2}
\end{figure}

\begin{figure}[b]
  \centering
  \includegraphics[width=\linewidth]{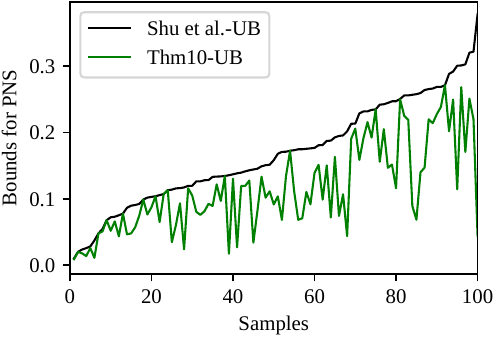}
  \caption{$\mathrm{PNS}(3)$ bounds under the pure mediator structure shown in Fig.~\ref{med}(b).}
  \label{sim_pic3}
\end{figure}

To illustrate the improvement for individual instances, we consider representative settings with $n,k\in \{3,4,5\}$, chosen for ease of presentation. For each setting, we generate $100{,}000$ sample distributions compatible with the causal diagrams shown in Figures~\ref{fig:z-not-descendant}(a), \ref{med}(a), and \ref{med}(b), corresponding to Theorems~\ref{nnk1}, \ref{nnk3}, and \ref{nnk4}, respectively, and compute the bounds using both methods. 
For sample $i$, let $[a_i,b_i]$ denote the bounds obtained from our theorem and $[c_i,d_i]$ those obtained from \citet{shu2026identificationprobabilitiescausationrecursive}. For each causal diagram, we summarize the following quantities:

\begin{itemize}
    \item Average improvement in the lower bound:
    $\frac{\sum (a_i-c_i)}{100{,}000}$.
    \item Average improvement in the upper bound:
    $\frac{\sum (d_i-b_i)}{100{,}000}$.
    \item Average width of Shu et al. PNS bounds:
    $\frac{\sum (d_i-c_i)}{100{,}000}$.
    \item Average width of the Theorems~\ref{nnk1}, \ref{nnk3}, and \ref{nnk4}:
    $\frac{\sum (b_i-a_i)}{100{,}000}$.
    \item Count of sample distributions benefiting from tighter bounds:
    $\sum f_i$, where $f_i=1$ if $a_i>c_i$ or $b_i<d_i$, and $f_i=0$ otherwise.
\end{itemize}

\begin{table}[t]
\centering
\begin{tabular}{
@{}l
@{\hspace{3pt}}c
@{\hspace{3pt}}c
@{\hspace{3pt}}c
@{\hspace{3pt}}c
@{\hspace{3pt}}c@{}
}
\toprule
&
\shortstack{Avg. LB\\gain}
&
\shortstack{Avg. UB\\gain}
&
\shortstack{Shu et al.\\width}
&
\shortstack{Thms\\width}
&
\shortstack{Improved\\(\%)}
\\
\midrule

\multicolumn{6}{c}{$n=k=3$}\\
\midrule
Non-desc.
  & 0.0014 & 0.0532 & 0.1953 & 0.1407 & 89.08\%\\
Part. med.
  & 0 & 6.52e-06 & 0.2038 & 0.2038 & 0.04\%\\
Pure med.
  & 0 & 0.0583 & 0.1776 & 0.1193 & 83.66\%\\

\midrule
\multicolumn{6}{c}{$n=k=4$}\\
\midrule
Non-desc.
  & 3.78e-07 & 0.0577 & 0.1346 & 0.0769 & 98.44\%\\
Part. med.
  & 0 & 0 & 0.1406 & 0.1406 & 0\%\\
Pure med.
  & 0 & 0.0137 & 0.1245 & 0.1108 & 36.58\%\\

\midrule
\multicolumn{6}{c}{$n=k=5$}\\
\midrule
Non-desc.
  & 0 & 0.0557 & 0.1033 & 0.0477 & 99.85\%\\
Part. med.
  & 0 & 0 & 0.1082 & 0.1082 & 0\%\\
Pure med.
  & 0 & 0.0003 & 0.0973 & 0.0970 & 1.26\%\\

\bottomrule
\end{tabular}
\caption{Performance metrics for Theorems~\ref{nnk1} (Non-desc.), \ref{nnk3} (Part. med.), and \ref{nnk4} (Pure med.). Here, LB and UB denote the lower and upper bounds, respectively.}
\label{sim_tab_1}
\end{table}

Table~\ref{sim_tab_1} reports these summary statistics across different theorems and dimensions. The partial-mediator structure yields only a small improvement because, as the dimension increases, the proposed bound is tighter than the existing bound in fewer and fewer simulation draws:
\begin{equation*}
\begin{aligned}
n=2,\ k=2 &: \quad 8.2\% \text{ of the bounds are tightened},\\
n=3,\ k=2 &: \quad 1.97\% \text{ of the bounds are tightened},\\
n=3,\ k=3 &: \quad 0.035\% \text{ of the bounds are tightened},\\
n=4,\ k=4 &: \quad 0 \text{ bounds are tightened in 20M draws}.
\end{aligned}
\label{sim_tab_2}
\end{equation*}

Thus, although the partial-mediator bound remains theoretically valid in higher-dimensional settings, it becomes increasingly unlikely to improve the existing bound in our simulations. Among the causal structures considered, the non-descendant covariate structure provides the most substantial tightening as the dimension increases.

In Figures~\ref{sim_pic1} and \ref{sim_pic3}, we randomly select $100$ of the $100{,}000$ sample distributions, plot the bounds obtained with and without incorporating causal graph information, and order the samples according to the upper bounds on $\mathrm{PNS}(3)$ given by \citet{shu2026identificationprobabilitiescausationrecursive}. For Figure~\ref{sim_pic2}, because Table~\ref{sim_tab_1} shows that Theorem~\ref{nnk3} yields narrower bounds less frequently, we continue generating samples until obtaining $100{,}000$ instances for which the bounds are narrowed and then randomly select $100$ of them for visualization.

\section{Conclusion}
In this paper, we extend bounds on probabilities of causation (PoCs) that incorporate causal graph information from binary to multivalued settings. Specifically, we derive new bounds under additional causal graph information for the four representative PoCs, $\mathrm{PNS}$, $\mathrm{PSub}$, $\mathrm{PRep}$, and $\mathrm{PN}$, which together could represent any discrete PoC. We consider several causal graph structures, including non-descendant covariates, backdoor adjustment sets, partial mediators, and pure mediators. Toy examples and simulation studies validate the proposed theorems and demonstrate that incorporating causal graph information consistently yields tighter bounds than approaches based solely on experimental and observational distributions.



\bibliography{ref} 
\clearpage 
\newpage 
\appendix
\section{Technical appendices and supplementary material}
\subsection{Theorem proofs}
\subsubsection{Proof of Theorem~\ref{nnk1}}
\begin{proof}
\begin{eqnarray}
    PNS(k) &=& P({y_{1}}_{x_{1}},...,{y_{k}}_{x_{k}})\nonumber\\
    &=& \sum_{z}{P({y_{1}}_{x_{1}},...,{y_{k}}_{x_{k}}\mid z)}\times P(z) \label{pnsk1*z}
\end{eqnarray}

First, we need to prove:
\begin{align*}
{\max \left \{
\begin{array}{cc}
0, \\
\\
\displaystyle \sum_{j = 1}^{k}P({y_{j}}_{x_{j}}\mid z) - k + 1, \\
\\
\text{For }i \in  \{1, ..., k\}:\\
\quad \displaystyle \sum_{\substack{1 \le j \le k \\ j \ne i}}
\Big[P({y_{j}}_{x_{j}}\mid z)+P({x_{j}\mid z})-\\
-P({x_{j}, y_{j}}\mid z)\Big]+ P({x_{i}, y_{i}}\mid z) - k + 1
\end{array}
\right \}\nonumber
} \\ \le  P({y_{1}}_{x_{1}},...,{y_{k}}_{x_{k}}\mid z)
\end{align*}
\begin{align*}
{\min \left \{
\begin{array}{cc}
\displaystyle \sum_{j = 1}^{k}P({x_{j}, y_{j}}\mid z) +\displaystyle\sum_{j = k+1}^{n}P({x_{j}}\mid z), \\
\\
\text{For }j \in \{1,...,k\}:\\
P({y_{j}}_{x_{j}}\mid z),\\
\\
\text{For }m\in\{1,...,k-1\},\\
t_j\in\{1,...,k\}:\\
\displaystyle \frac{1}{m}\Big[\sum_{j = 0}^{m}P({y_{t_j}}_{x_{t_j}}\mid z) -P({x_{t_j}, y_{t_j}\mid z})\Big]
\end{array} 
\right \}\nonumber
} \\ \ge  P({y_{1}}_{x_{1}},...,{y_{k}}_{x_{k}}\mid z)
\end{align*}

By the Fréchet Inequalities, for any $z$ with $P(z)>0$, we have
\begin{eqnarray*}
&&P({y_{1}}_{x_{1}},...,{y_{k}}_{x_{k}}\mid z)\\
&&\ge
\max\left\{0, \displaystyle \sum_{j = 1}^{k}P({y_{j}}_{x_{j}}\mid z) - k + 1\right\}\\
&&P({y_{1}}_{x_{1}},...,{y_{k}}_{x_{k}}\mid z)
\le
\min_{1\le j \le k} \left\{P({y_{j}}_{x_{j}}\mid z)\right\}
\end{eqnarray*}

The equality of the first lower bound holds when $\exists j\in [1,k], \text{that } P({y_j}_{x_j}\mid z)=0$.

The equality of the second lower bound holds when $\exists i \in [1,k]$, $P({y_j}_{x_j}\mid z)=1$ for $\forall j \in [1,k], j\ne i$.

The equality of the second upper bound holds when $\exists j \in [1,k]$, $P({y_{i}}_{x_{i}}\mid z) = 1$ for $\forall i \in [1,k]$, $i\ne j$.

Also,
\begin{eqnarray*}
    &&P({y_{1}}_{x_{1}},...,{y_{k}}_{x_{k}}\mid z) \\
    &=& \sum_{j=1}^{n} P({y_{1}}_{x_{1}},...,{y_{k}}_{x_{k}}, x_{j}\mid z)\\
    &\ge& \sum_{j=1}^{k} P({y_{1}}_{x_{1}},...,{y_{k}}_{x_{k}}, x_{j}\mid z)\\
    &&\text{here, the equal sign holds when }\forall {x_j} = 0 \text{ for } \\
    &&j \in [k+1, n].
\end{eqnarray*}

For $\forall i, \text{s.t.}, i \in \{1, ..., k\}$, 
\begin{eqnarray*}
&&\sum_{j=1}^{k} P({y_{1}}_{x_{1}},...,{y_{k}}_{x_{k}}, x_{j}\mid z)\\
&\ge& P({y_{1}}_{x_{1}},...,{y_{k}}_{x_{k}}, x_{i}\mid z)\\
&&\text{here, the equal sign holds when }\forall {x_j} = 0 \text{ for }\\
&& j \in [1, k] \text{ and } j\ne i.
\end{eqnarray*}
Thus, we have 
\begin{eqnarray*}
&& P({y_{1}}_{x_{1}},...,{y_{k}}_{x_{k}}\mid z) \\
&\ge& P({y_{1}}_{x_{1}},...,{y_{k}}_{x_{k}}, x_{i}\mid z) \\
&=& P({y_{1}}_{x_{1}},...,{y_{i-1}}_{x_{i-1}},{y_{i+1}}_{x_{i+1}},...,{y_{k}}_{x_{k}}, x_{i}, y_{i}\mid z) \\
&=& P({y_{1}}_{x_{1}},...,{y_{i-1}}_{x_{i-1}},{y_{i+1}}_{x_{i+1}},...,{y_{k}}_{x_{k}}, x_{i}, y_{i}\mid z) \\
&&+ \sum_{j=1}^{n}{P(x_{j}\mid z)} - 1\\ 
&&+ \sum_{j=1}^{n} \sum_{q=1}^{n} P({y_{1}}_{x_{1}},...,{y_{i-1}}_{x_{i-1}},{y_{i+1}}_{x_{i+1}},...,{y_{k}}_{x_{k}},\\
&&x_{j}, y_{q}\mid z) \\
&&- \sum_{j=1}^{n} \sum_{q=1}^{n} P({y_{1}}_{x_{1}},...,{y_{i-1}}_{x_{i-1}},{y_{i+1}}_{x_{i+1}},...,{y_{k}}_{x_{k}},\\ 
&&x_{j}, y_{q}\mid z) \\
&=& P({y_{1}}_{x_{1}},...,{y_{i-1}}_{x_{i-1}},{y_{i+1}}_{x_{i+1}},...,{y_{k}}_{x_{k}}, x_{i}, y_{i}\mid z) \\
&&+ \sum_{\substack{1 \le j \le k \\ j \ne i}}{P(x_{j}\mid z)} + P(x_{i}\mid z) + \sum_{j=k+1}^{n}{P(x_{j}\mid z)} \\
&&- 1 + P({y_{1}}_{x_{1}},...,{y_{i-1}}_{x_{i-1}},{y_{i+1}}_{x_{i+1}},...,{y_{k}}_{x_{k}}\mid z) \\
&&- \sum_{\substack{1 \le j \le k \\ j \ne i}}  P({y_{1}}_{x_{1}},...,{y_{i-1}}_{x_{i-1}},{y_{i+1}}_{x_{i+1}},...,{y_{k}}_{x_{k}},\\ 
&&x_{j}, y_{j}\mid z) \\
&&- \sum_{q=1}^{n} P({y_{1}}_{x_{1}},...,{y_{i-1}}_{x_{i-1}},{y_{i+1}}_{x_{i+1}},...,{y_{k}}_{x_{k}}, x_{i}, \\
&&y_{q}\mid z) \\
\end{eqnarray*}
\begin{eqnarray*}
&&- \sum_{j=k+1}^{n} \sum_{q=1}^{n} P({y_{1}}_{x_{1}},...,{y_{i-1}}_{x_{i-1}},{y_{i+1}}_{x_{i+1}},...,\\
&&{y_{k}}_{x_{k}}, x_{j}, y_{q}\mid z) \\
&=& P({y_{1}}_{x_{1}},...,{y_{i-1}}_{x_{i-1}},{y_{i+1}}_{x_{i+1}},...,{y_{k}}_{x_{k}}\mid z) - 1 \\
&&+ \sum_{\substack{1 \le j \le k \\ j \ne i}}{P(x_{j}\mid z)} \\
&&- \sum_{\substack{1 \le j \le k \\ j \ne i}}  P({y_{1}}_{x_{1}},...,{y_{i-1}}_{x_{i-1}},{y_{i+1}}_{x_{i+1}},...,{y_{k}}_{x_{k}},\\
&&x_{j}, y_{j}\mid z) \\
&&+ P(x_{i}\mid z) \\
&&- \sum_{q=1}^{n} P({y_{1}}_{x_{1}},...,{y_{i-1}}_{x_{i-1}},{y_{i+1}}_{x_{i+1}},...,{y_{k}}_{x_{k}},\\
&&x_{i}, y_{q}\mid z) \\
&&+ P({y_{1}}_{x_{1}},...,{y_{i-1}}_{x_{i-1}},{y_{i+1}}_{x_{i+1}},...,{y_{k}}_{x_{k}}, x_{i}, \\
&&y_{i}\mid z) \\
&&+ \sum_{j=k+1}^{n}{P(x_{j}\mid z)} \\
&&- \sum_{j=k+1}^{n} \sum_{q=1}^{n} P({y_{1}}_{x_{1}},...,{y_{i-1}}_{x_{i-1}},{y_{i+1}}_{x_{i+1}},...,\\
&&{y_{k}}_{x_{k}}, x_{j}, y_{q}\mid z)\\
&=& P({y_{1}}_{x_{1}},...,{y_{i-1}}_{x_{i-1}},{y_{i+1}}_{x_{i+1}},...,{y_{k}}_{x_{k}}\mid z) - 1 \\
&&+ \sum_{\substack{1 \le j \le k \\ j \ne i}}{P(x_{j}\mid z)} \\
&&- \sum_{\substack{1 \le j \le k \\ j \ne i}}  P({y_{1}}_{x_{1}},...,{y_{i-1}}_{x_{i-1}},{y_{i+1}}_{x_{i+1}},...,{y_{k}}_{x_{k}},\\ 
&&x_{j}, y_{j}\mid z) \\
&&+ P(x_{i}\mid z) \\
&&- \sum_{\substack{1 \le q \le n \\ q \ne i}} P({y_{1}}_{x_{1}},...,{y_{i-1}}_{x_{i-1}},{y_{i+1}}_{x_{i+1}},...,{y_{k}}_{x_{k}}, \\
&&x_{i}, y_{q}\mid z) \\
&&+ \sum_{j=k+1}^{n}{P(x_{j}\mid z)} \\
&&- \sum_{j=k+1}^{n} P({y_{1}}_{x_{1}},...,{y_{i-1}}_{x_{i-1}},{y_{i+1}}_{x_{i+1}},...,{y_{k}}_{x_{k}},\\
&&x_{j}\mid z)
\end{eqnarray*}
By applying the Frechet inequalities, for any $z$ with $P(z)>0$, we have \\
\begin{eqnarray*}
&&P({y_{1}}_{x_{1}},...,{y_{k}}_{x_{k}}\mid z) \\
&\ge& \sum_{\substack{1 \le j \le k \\ i \ne j}}{P({y_{j}}_{x_{j}}\mid z)} - (k-2) - 1 \\
&&+ \sum_{\substack{1 \le j \le k \\ j \ne i}}{P(x_{j}\mid z)} - \sum_{\substack{1 \le j \le k \\ j \ne i}}{P(x_{j}, y_{j}\mid z)} \\
&&+ P(x_{i}\mid z) - \sum_{\substack{1 \le q \le n \\ q \ne i}}{P(x_{i}, y_{q}\mid z)} \\
&&+ \sum_{j=k+1}^{n}{P(x_{j}\mid z)} - \sum_{j=k+1}^{n}{P(x_{j}\mid z)}\\
&&\text{here, the equal sign holds when }
\exists i \in [1,k], \\
&&\text{ and } P({y_j}_{x_j}\mid z)=1 \text{ for } \forall j \in [1,k], j\ne i.\\
\end{eqnarray*}
\begin{eqnarray*}
&=& \sum_{\substack{1 \le j \le k \\ i \ne j}}\left[P({y_{j}}_{x_{j}}\mid z)+P({x_{j}}\mid z)-P({x_{j}, y_{j}}\mid z)\right] \\
&&+ P({x_{i}, y_{i}}\mid z) - k + 1 \\
\end{eqnarray*}

To proof the first upper bound, 
\begin{eqnarray*}
    &&P({y_{1}}_{x_{1}},...,{y_{k}}_{x_{k}}\mid z) \\
    &=& \sum_{j=1}^{n} P({y_{1}}_{x_{1}},...,{y_{k}}_{x_{k}}, x_{j}\mid z)\\
    &=& \sum_{j=1}^{k} P({y_{1}}_{x_{1}},...,{y_{k}}_{x_{k}}, x_{j}\mid z) \\
    &&+ \sum_{j=k+1}^{n} P({y_{1}}_{x_{1}},...,{y_{k}}_{x_{k}}, x_{j}\mid z)\\
    &\le& \sum_{j=1}^{k} P(x_{j}, y_{j}\mid z) + \sum_{j = k+1}^{n}P({x_{j}}\mid z)
\end{eqnarray*}

The equality of the first upper bound holds when $P({y_j}_{x_j}\mid z)=1 \text{ for } \forall j \in [1,k].$

To proof the remaining upper bound, for $m\in\{1,...,k-1\}$, $t_j\in\{1,...,k\}$, we can first write

\begin{eqnarray}
    &&P({y_{1}}_{x_{1}},...,{y_{k}}_{x_{k}}\mid z) \nonumber\\
    &=& \sum_{j=1}^{n} P({y_{1}}_{x_{1}},...,{y_{k}}_{x_{k}}, x_{j}\mid z) \nonumber\\
    &=& \frac{1}{m} \times \left[(m) \sum_{j=1}^{n} P({y_{1}}_{x_{1}},...,{y_{k}}_{x_{k}}, x_{j}\mid z)\right]  \label{pns(k)1/m*m}
\end{eqnarray}
For $m = 1$, 
\begin{eqnarray*}
    &&\frac{1}{m}\left[\sum_{j = 0}^{m}P({y_{t_j}}_{x_{t_j}}\mid z) - P({x_{t_j}, y_{t_j}}\mid z)\right] \\
    &=& \sum_{j = 0}^{1}P({y_{t_j}}_{x_{t_j}}\mid z) - P({x_{t_j}, y_{t_j}}\mid z)\\
    &=& P({y_{t_0}}_{x_{t_0}}\mid z) + P({y_{t_1}}_{x_{t_1}}\mid z) - P({x_{t_0}, y_{t_0}}\mid z) \\
    &&- P({x_{t_1}, y_{t_1}}\mid z)
\end{eqnarray*}
W.L.O.G., let $1 \le t_0 < t_1 \le k$ :
\begin{eqnarray}
    && P({y_{1}}_{x_{1}},...,{y_{k}}_{x_{k}}\mid z) \nonumber\\
    &=& \sum_{j=1}^{n} P({y_{1}}_{x_{1}},...,{y_{k}}_{x_{k}}, x_{j}\mid z) \nonumber\\
    &=& \sum_{j=1}^{k} P({y_{1}}_{x_{1}},...,{y_{k}}_{x_{k}}, x_{j}\mid z) \nonumber\\
    &&+ \sum_{j=k+1}^{n} P({y_{1}}_{x_{1}},...,{y_{k}}_{x_{k}}, x_{j}\mid z) \nonumber\\
    &\le& \sum_{j=1}^{t_0}{P({y_{t_0}}_{x_{t_0}}, x_j, {y_{j+t_1-t_0}}_{x_{j+t_1-t_0}}\mid z)}
    \nonumber\\
    &&+ \sum_{j={t_0}+1}^{k-t_1+t_0} P({y_{t_0}}_{x_{t_0}}, x_{j}, {y_{j+t_1-t_0}}_{x_{j+t_1-t_0}}\mid z) \nonumber\\
    &&+ \sum_{j=1}^{t_1-t_0}{P({y_{t_0}}_{x_{t_0}}, x_{k-t_1+t_0+j}, {y_{j}}_{x_{j}}\mid z)} \nonumber\\
    &&+ \sum_{j=k+1}^{n} P({y_{1}}_{x_{1}},...,{y_{k}}_{x_{k}}, x_{j}\mid z) \nonumber\\
    &&\text{here, the equal sign holds when }
    \exists 1\le t_0 < t_1 \le k, \nonumber\\
    &&\text{ and } P({y_i}_{x_i}\mid z)=1 \text{ for } \forall i \in [1,k], i\ne t_0, \nonumber\\
    &&i\ne j+t_1-t_0 \text{ for } j\in [1,k-t_1+t_0].\nonumber\\
    &=& \sum_{j=1}^{t_0}{P({y_{t_0}}_{x_{t_0}}, x_j, {y_{j+t_1-t_0}}_{x_{j+t_1-t_0}}\mid z)} 
    \nonumber\\
    &&+ \sum_{j=t_0+1}^{k-t_1+t_0} P({y_{t_0}}_{x_{t_0}}, x_{j}, {y_{j+t_1-t_0}}_{x_{j+t_1-t_0}}\mid z) \nonumber\\
    &&+ \sum_{j=1}^{t_1-t_0}{P({y_{t_0}}_{x_{t_0}}, x_{k-t_1+t_0+j}, {y_{j}}_{x_{j}}\mid z)} 
    \nonumber\\
    &&+ \sum_{j=k+1}^{n} P({y_{1}}_{x_{1}},...,{y_{k}}_{x_{k}}, x_{j}\mid z) \nonumber\\
    &&+ P({y_{t_0}}_{x_{t_0}}\mid z) + P({y_{t_1}}_{x_{t_1}}\mid z) - P({x_{t_0}, y_{t_0}}\mid z) \nonumber\\
    &&- P({x_{t_1}, y_{t_1}}\mid z) \nonumber
\end{eqnarray}
\begin{eqnarray}
    &&- \sum_{\substack{\{i_1,...,i_{t_0-1},i_{t_0+1},...,i_{k+2}\} \\ \in \{1,...,n\}^{k+1}}}{P({y_{i_1}}_{x_1},...,{y_{i_{t_0-1}}}_{x_{t_0-1}},}\nonumber\\
    &&{{y_{t_0}}_{x_{t_0}},{y_{i_{t_0+1}}}_{x_{t_0+1}},...,{y_{i_{k}}}_{x_{k}},x_{i_{k+1}},y_{i_{k+2}}\mid z)} \nonumber\\
    &&- \sum_{\substack{\{i_1,...,i_{t_1-1},i_{t_1+1},...,i_{k+2}\} \\ \in \{1,...,n\}^{k+1}}}{P({y_{i_1}}_{x_1},...,{y_{i_{t_1-1}}}_{x_{t_1-1}},}\nonumber\\
    &&{{y_{t_1}}_{x_{t_1}},{y_{i_{t_1+1}}}_{x_{t_1+1}},...,{y_{i_{k}}}_{x_{k}},x_{i_{k+1}},y_{i_{k+2}}\mid z)} \nonumber\\ 
    &&+ \sum_{\substack{\{i_1,...,i_{t_0-1},i_{t_0+1},...,i_{k}\} \\ \in \{1,...,n\}^{k-1}}}{P({y_{i_1}}_{x_1},...,{y_{i_{t_0-1}}}_{x_{t_0-1}},}\nonumber\\
    &&{{y_{t_0}}_{x_{t_0}},{y_{i_{t_0+1}}}_{x_{t_0+1}},...,{y_{i_{k}}}_{x_{k}}, x_{t_0}, y_{t_0}\mid z)} \nonumber\\
    &&+ \sum_{\substack{\{i_1,...,i_{t_1-1},i_{t_1+1},...,i_{k}\} \\ \in \{1,...,n\}^{k-1}}}{P({y_{i_1}}_{x_1},...,{y_{i_{t_1-1}}}_{x_{t_1-1}},}\nonumber\\
    &&{{y_{t_1}}_{x_{t_1}},{y_{i_{t_1+1}}}_{x_{t_1+1}},...,{y_{i_{k}}}_{x_{k}}, x_{t_1}, y_{t_1}\mid z)} \nonumber\\
    &=& P({y_{t_0}}_{x_{t_0}}\mid z) + P({y_{t_1}}_{x_{t_1}}\mid z) - P({x_{t_0}, y_{t_0}}\mid z)  \nonumber\\
    &&- P({x_{t_1}, y_{t_1}}\mid z)\nonumber\\
    &&+ \sum_{j=1}^{t_0-1}{P({y_{t_0}}_{x_{t_0}}, x_j, {y_{j+t_1-t_0}}_{x_{j+t_1-t_0}}\mid z)}
    \nonumber\\
    &&+ {P({y_{t_0}}_{x_{t_0}}, x_{t_0}, {y_{t_1}}_{x_{t_1}}\mid z)}  \nonumber\\
    &&+ \sum_{j={t_0}+1}^{k-t_1+t_0} P({y_{t_0}}_{x_{t_0}}, x_{j}, {y_{j+t_1-t_0}}_{x_{j+t_1-t_0}}\mid z) 
    \nonumber\\
    &&+ \sum_{j=1}^{t_1-t_0}{P({y_{t_0}}_{x_{t_0}}, x_{k-t_1+t_0+j}, {y_{j}}_{x_{j}}\mid z)} \nonumber\\
    &&+ \sum_{j=k+1}^{n} P({y_{1}}_{x_{1}},...,{y_{k}}_{x_{k}}, x_{j}\mid z) \nonumber\\
    &&- \sum_{\substack{\{i_1,...,i_{t_0-1},i_{t_0+1},...,i_{k+2}\} \\ \in \{1,...,n\}^{k+1}}}{P({y_{i_1}}_{x_1},...,{y_{i_{t_0-1}}}_{x_{t_0-1}},}\nonumber\\
    &&{{y_{t_0}}_{x_{t_0}},{y_{i_{t_0+1}}}_{x_{t_0+1}},...,{y_{i_{k}}}_{x_{k}},x_{i_{k+1}},y_{i_{k+2}}\mid z)} \nonumber\\
    &&- \sum_{\substack{\{i_1,...,i_{t_1-1},i_{t_1+1},...,i_{k+2}\} \\ \in \{1,...,n\}^{k+1}}}{P({y_{i_1}}_{x_1},...,{y_{i_{t_1-1}}}_{x_{t_1-1}},}\nonumber\\
    &&{{y_{t_1}}_{x_{t_1}},{y_{i_{t_1+1}}}_{x_{t_1+1}},...,{y_{i_{k}}}_{x_{k}},x_{i_{k+1}},y_{i_{k+2}}\mid z)} \nonumber\\ 
    &&+ \sum_{\substack{\{i_1,...,i_{t_0-1},i_{t_0+1},...,i_{k}\} \\ \in \{1,...,n\}^{k-1}}}{P({y_{i_1}}_{x_1},...,{y_{i_{t_0-1}}}_{x_{t_0-1}},}\nonumber\\
    &&{{y_{t_0}}_{x_{t_0}},{y_{i_{t_0+1}}}_{x_{t_0+1}},...,{y_{i_{k}}}_{x_{k}}, x_{t_0}, y_{t_0}\mid z)} \nonumber\\
    &&+ \sum_{\substack{\{i_1,...,i_{t_1-1},i_{t_1+1},...,i_{k}\} \\ \in \{1,...,n\}^{k-1}}}{P({y_{i_1}}_{x_1},...,{y_{i_{t_1-1}}}_{x_{t_1-1}},}\nonumber\\
    &&{{y_{t_1}}_{x_{t_1}},{y_{i_{t_1+1}}}_{x_{t_1+1}},...,{y_{i_{k}}}_{x_{k}}, x_{t_1}, y_{t_1}\mid z)} \nonumber
\end{eqnarray}
\begin{eqnarray}
    &=& P({y_{t_0}}_{x_{t_0}}\mid z) + P({y_{t_1}}_{x_{t_1}}\mid z) - P({x_{t_0}, y_{t_0}}\mid z) \nonumber\\
    &&- P({x_{t_1}, y_{t_1}}\mid z) \nonumber\\
    &&- \sum_{\substack{\{i_1,...,i_{t_0-1},i_{t_0+1},...,i_{k+2}\} \\ \in \{1,...,n\}^{k+1}}}{P({y_{i_1}}_{x_1},...,{y_{i_{t_0-1}}}_{x_{t_0-1}},}\nonumber\\
    &&{{y_{t_0}}_{x_{t_0}},{y_{i_{t_0+1}}}_{x_{t_0+1}},...,{y_{i_{k}}}_{x_{k}},x_{i_{k+1}},y_{i_{k+2}}\mid z)} \nonumber\\
    &&+ \Bigg[\sum_{j=1}^{t_0-1}{P({y_{t_0}}_{x_{t_0}}, x_j, {y_{j+t_1-t_0}}_{x_{j+t_1-t_0}}\mid z)} 
    \nonumber\\
    &&+ \sum_{j=t_0+1}^{k-t_1+t_0} P({y_{t_0}}_{x_{t_0}}, x_{j}, {y_{j+t_1-t_0}}_{x_{j+t_1-t_0}}\mid z) \nonumber\\
    &&+ \sum_{j=1}^{t_1-t_0}{P({y_{t_0}}_{x_{t_0}}, x_{k-t_1+t_0+j}, {y_{j}}_{x_{j}}\mid z)} 
    \nonumber\\
    &&+ \sum_{j=k+1}^{n} P({y_{1}}_{x_{1}},...,{y_{k}}_{x_{k}}, x_{j}\mid z) \nonumber\\
    &&+ \sum_{\substack{\{i_1,...,i_{t_0-1},i_{t_0+1},...,i_{k}\} \\ \in \{1,...,n\}^{k-1}}}{P({y_{i_1}}_{x_1},...,{y_{i_{t_0-1}}}_{x_{t_0-1}},}\nonumber\\
    &&{{y_{t_0}}_{x_{t_0}},{y_{i_{t_0+1}}}_{x_{t_0+1}},...,{y_{i_{k}}}_{x_{k}}, x_{t_0}, y_{t_0}\mid z)} \Bigg] \nonumber\\
    &&- \sum_{\substack{\{i_1,...,i_{t_1-1},i_{t_1+1},...,i_{k+2}\} \\ \in \{1,...,n\}^{k+1}}}{P({y_{i_1}}_{x_1},...,{y_{i_{t_1-1}}}_{x_{t_1-1}},}\nonumber\\
    &&{{y_{t_1}}_{x_{t_1}},{y_{i_{t_1+1}}}_{x_{t_1+1}},...,{y_{i_{k}}}_{x_{k}},x_{i_{k+1}},y_{i_{k+2}}\mid z)} \nonumber\\
    &&+ \Bigg[{P({y_{t_0}}_{x_{t_0}}, x_{t_0}, {y_{t_1}}_{x_{t_1}}\mid z)} \nonumber\\
    &&+ \sum_{\substack{\{i_1,...,i_{t_1-1},i_{t_1+1},...,i_{k}\} \\ \in \{1,...,n\}^{k-1}}}{P({y_{i_1}}_{x_1},...,{y_{i_{t_1-1}}}_{x_{t_1-1}},}\nonumber\\
    &&{{y_{t_1}}_{x_{t_1}},{y_{i_{t_1+1}}}_{x_{t_1+1}},...,{y_{i_{k}}}_{x_{k}}, x_{t_1}, y_{t_1}\mid z)}\Bigg] \label{m=1}\\
    &\le& P({y_{t_0}}_{x_{t_0}}\mid z) + P({y_{t_1}}_{x_{t_1}}\mid z) - P({x_{t_0}, y_{t_0}}\mid z) \nonumber\\
    &&- P({x_{t_1}, y_{t_1}}\mid z) \label{t0t1}
\end{eqnarray}
Here, the equal sign holds when $\exists 1\le t_0 < t_1 \le k$, and $P({y_{t_0}}_{x_{t_0}}\mid z)=0$.\\
For $m = 2$, by applying equation (\ref{pns(k)1/m*m}), we can get:\\
\begin{eqnarray*}
    &&P({y_{1}}_{x_{1}},...,{y_{k}}_{x_{k}}\mid z)\\
    &=& \frac{1}{2} \times \left[2 \sum_{j=1}^{n} P({y_{1}}_{x_{1}},...,{y_{k}}_{x_{k}}, x_{j}\mid z)\right]\\
\end{eqnarray*}
\begin{eqnarray*}
    &=& \frac{1}{2} \times \Big[\sum_{j=1}^{n} P({y_{1}}_{x_{1}},...,{y_{k}}_{x_{k}}, x_{j}\mid z) \\
    && +\sum_{j=1}^{n} P({y_{1}}_{x_{1}},...,{y_{k}}_{x_{k}}, x_{j}\mid z)\Big]\\
\end{eqnarray*}
Then W.L.O.G., let $1 \le t_0 < t_1 < t_2 \le k$. By applying equation (\ref{m=1}), we have
\begin{eqnarray*}
    &&P({y_{1}}_{x_{1}},...,{y_{k}}_{x_{k}}) \\
    &\le& \frac{1}{2} \times \Bigg\{\Bigg[P({y_{t_0}}_{x_{t_0}}\mid z) + P({y_{t_1}}_{x_{t_1}}\mid z) \\
    &&- P({x_{t_0}, y_{t_0}}\mid z) - P({x_{t_1}, y_{t_1}}\mid z)\\
    &&- \sum_{\substack{\{i_1,...,i_{t_0-1},i_{t_0+1},...,i_{k+2}\} \\ \in \{1,...,n\}^{k+1}}}{P({y_{i_1}}_{x_1},...,{y_{i_{t_0-1}}}_{x_{t_0-1}},}\\
    &&{{y_{t_0}}_{x_{t_0}},{y_{i_{t_0+1}}}_{x_{t_0+1}},...,{y_{i_{k}}}_{x_{k}},x_{i_{k+1}},y_{i_{k+2}}\mid z)} \\
    &&+ \Bigg[\sum_{j=1}^{t_0-1}{P({y_{t_0}}_{x_{t_0}}, x_j, {y_{j+t_1-t_0}}_{x_{j+t_1-t_0}}\mid z)} \\
    &&
    + \sum_{j=t_0+1}^{k-t_1+t_0} P({y_{t_0}}_{x_{t_0}}, x_{j}, {y_{j+t_1-t_0}}_{x_{j+t_1-t_0}}\mid z) \\
    &&+ \sum_{j=1}^{t_1-t_0}{P({y_{t_0}}_{x_{t_0}}, x_{k-t_1+t_0+j}, {y_{j}}_{x_{j}}\mid z)} \\
    &&
    + \sum_{j=k+1}^{n} P({y_{1}}_{x_{1}},...,{y_{k}}_{x_{k}}, x_{j}\mid z) \\
    &&+ \sum_{\substack{\{i_1,...,i_{t_0-1},i_{t_0+1},...,i_{k}\} \\ \in \{1,...,n\}^{k-1}}}{P({y_{i_1}}_{x_1},...,{y_{i_{t_0-1}}}_{x_{t_0-1}},}\\
    &&{{y_{t_0}}_{x_{t_0}},{y_{i_{t_0+1}}}_{x_{t_0+1}},...,{y_{i_{k}}}_{x_{k}}, x_{t_0}, y_{t_0}\mid z)} \Bigg] \\
    &&- \sum_{\substack{\{i_1,...,i_{t_1-1},i_{t_1+1},...,i_{k+2}\} \\ \in \{1,...,n\}^{k+1}}}{P({y_{i_1}}_{x_1},...,{y_{i_{t_1-1}}}_{x_{t_1-1}},}\\
    &&{{y_{t_1}}_{x_{t_1}},{y_{i_{t_1+1}}}_{x_{t_1+1}},...,{y_{i_{k}}}_{x_{k}},x_{i_{k+1}},y_{i_{k+2}}\mid z)} \\
    &&+ \bigg[{P({y_{t_0}}_{x_{t_0}}, x_{t_0}, {y_{t_1}}_{x_{t_1}}\mid z)} \\
    &&+ \sum_{\substack{\{i_1,...,i_{t_1-1},i_{t_1+1},...,i_{k}\} \\ \in \{1,...,n\}^{k-1}}}{P({y_{i_1}}_{x_1},...,{y_{i_{t_1-1}}}_{x_{t_1-1}},}\\
    &&{{y_{t_1}}_{x_{t_1}},{y_{i_{t_1+1}}}_{x_{t_1+1}},...,{y_{i_{k}}}_{x_{k}}, x_{t_1}, y_{t_1}\mid z)}\bigg]\Bigg] \\
    &&+ \Bigg[P({y_{t_2}}_{x_{t_2}}\mid z) - P({x_{t_2}, y_{t_2}}\mid z)\\
\end{eqnarray*}
\begin{eqnarray*}
    &&- \sum_{\substack{\{i_1,...,i_{t_2-1},i_{t_2+1},...,i_{k+2}\} \\ \in \{1,...,n\}^{k+1}}}{P({y_{i_1}}_{x_1},...,{y_{i_{t_2-1}}}_{x_{t_2-1}},}\\
    &&{{y_{t_2}}_{x_{t_2}},{y_{i_{t_2+1}}}_{x_{t_2+1}},...,{y_{i_{k}}}_{x_{k}},x_{i_{k+1}},y_{i_{k+2}}\mid z)} \\ 
    &&+ \sum_{\substack{\{i_1,...,i_{t_2-1},i_{t_2+1},...,i_{k}\} \\ \in \{1,...,n\}^{k-1}}}{P({y_{i_1}}_{x_1},...,{y_{i_{t_2-1}}}_{x_{t_2-1}},}\\
    &&{{y_{t_2}}_{x_{t_2}},{y_{i_{t_2+1}}}_{x_{t_2+1}},...,{y_{i_{k}}}_{x_{k}}, x_{t_2}, y_{t_2}\mid z)} \\
    &&+ \sum_{\substack{1 \le j \le k \\ j \ne t_2}}{P({y_{1}}_{x_{1}},...,{y_{k}}_{x_{k}}, x_{j}\mid z)} \\
    &&+ {P({y_{1}}_{x_{1}},...,{y_{t_2-1}}_{x_{t_2+1}},{y_{t_2}}_{x_{t_2}},{y_{t_2+1}}_{x_{+1}},...,{y_{k}}_{x_{k}},}\\
    &&{x_{t_2}, y_{t_2}\mid z)} \\
    &&+ \sum_{j=k+1}^{n} P({y_{1}}_{x_{1}},...,{y_{k}}_{x_{k}}, x_{j}\mid z)\Bigg]\Bigg\}\\
    &\le& \frac{1}{2} \times \Bigg\{\Bigg[P({y_{t_0}}_{x_{t_0}}\mid z) + P({y_{t_1}}_{x_{t_1}}\mid z) \\
    &&- P({x_{t_0}, y_{t_0}}\mid z)- P({x_{t_1}, y_{t_1}}\mid z) \\
    &&- \sum_{\substack{\{i_1,...,i_{t_0-1},i_{t_0+1},...,i_{k+2}\} \\ \in \{1,...,n\}^{k+1}}}{P({y_{i_1}}_{x_1},...,{y_{i_{t_0-1}}}_{x_{t_0-1}},}\\
    &&{{y_{t_0}}_{x_{t_0}},{y_{i_{t_0+1}}}_{x_{t_0+1}},...,{y_{i_{k}}}_{x_{k}},x_{i_{k+1}},y_{i_{k+2}}\mid z)} \\
    &&+ \bigg[\sum_{j=1}^{t_0-1}{P({y_{t_0}}_{x_{t_0}}, x_j, {y_{j+t_1-t_0}}_{x_{j+t_1-t_0}}\mid z)}
    \\
    &&+ \sum_{j=t_0+1}^{k-t_1+t_0} P({y_{t_0}}_{x_{t_0}}, x_{j}, {y_{j+t_1-t_0}}_{x_{j+t_1-t_0}}\mid z) \\
    &&+ \sum_{j=1}^{t_1-t_0}{P({y_{t_0}}_{x_{t_0}}, x_{k-t_1+t_0+j}, {y_{j}}_{x_{j}}\mid z)}
    \\
    &&+ \sum_{j=k+1}^{n} P({y_{1}}_{x_{1}},...,{y_{k}}_{x_{k}}, x_{j}\mid z) \\
    &&+ \sum_{\substack{\{i_1,...,i_{t_0-1},i_{t_0+1},...,i_{k}\} \\ \in \{1,...,n\}^{k-1}}}{P({y_{i_1}}_{x_1},...,{y_{i_{t_0-1}}}_{x_{t_0-1}},}\\
    &&{{y_{t_0}}_{x_{t_0}},{y_{i_{t_0+1}}}_{x_{t_0+1}},...,{y_{i_{k}}}_{x_{k}}, x_{t_0}, y_{t_0}\mid z)} \bigg]\\
    &&- \sum_{\substack{\{i_1,...,i_{t_1-1},i_{t_1+1},...,i_{k+2}\} \\ \in \{1,...,n\}^{k+1}}}{P({y_{i_1}}_{x_1},...,{y_{i_{t_1-1}}}_{x_{t_1-1}},}\\
    &&{{y_{t_1}}_{x_{t_1}},{y_{i_{t_1+1}}}_{x_{t_1+1}},...,{y_{i_{k}}}_{x_{k}},x_{i_{k+1}},y_{i_{k+2}}\mid z)} \\
    &&+ \bigg[{P({y_{t_0}}_{x_{t_0}}, x_{t_0}, {y_{t_1}}_{x_{t_1}}\mid z)} \\
    &&+ \sum_{\substack{\{i_1,...,i_{t_1-1},i_{t_1+1},...,i_{k}\} \\ \in \{1,...,n\}^{k-1}}}{P({y_{i_1}}_{x_1},...,{y_{i_{t_1-1}}}_{x_{t_1-1}},}\\
\end{eqnarray*}
\begin{eqnarray*} 
    &&{{y_{t_1}}_{x_{t_1}},{y_{i_{t_1+1}}}_{x_{t_1+1}},...,{y_{i_{k}}}_{x_{k}}, x_{t_1}, y_{t_1}\mid z)}\bigg]\Bigg] \\
    &&+ \Bigg[P({y_{t_2}}_{x_{t_2}}\mid z) - P({x_{t_2}, y_{t_2}}\mid z)\\
    &&- \sum_{\substack{\{i_1,...,i_{t_2-1},i_{t_2+1},...,i_{k+2}\} \\ \in \{1,...,n\}^{k+1}}}{P({y_{i_1}}_{x_1},...,{y_{i_{t_2-1}}}_{x_{t_2-1}},}\\
    &&{{y_{t_2}}_{x_{t_2}},{y_{i_{t_2+1}}}_{x_{t_2+1}},...,{y_{i_{k}}}_{x_{k}},x_{i_{k+1}},y_{i_{k+2}}\mid z)}\\
    &&+ \sum_{\substack{\{i_1,...,i_{t_2-1},i_{t_2+1},...,i_{k}\} \\ \in \{1,...,n\}^{k-1}}}{P({y_{i_1}}_{x_1},...,{y_{i_{t_2-1}}}_{x_{t_2-1}},}\\
    &&{{y_{t_2}}_{x_{t_2}},{y_{i_{t_2+1}}}_{x_{t_2+1}},...,{y_{i_{k}}}_{x_{k}}, x_{t_2}, y_{t_2}\mid z)} \\
    &&+ \sum_{\substack{1 \le j \le k \\ j \ne t_2}}{P({y_{1}}_{x_{1}},...,{y_{k}}_{x_{k}}, x_{j}\mid z)} \\
    &&+ {P({y_{t_1}}_{x_{t_1}}, x_{t_2}, {y_{t_2}}_{x_{t_2}}\mid z)} \\
    &&+ \sum_{j=k+1}^{n} P({y_{1}}_{x_{1}},...,{y_{k}}_{x_{k}}, x_{j}\mid z)\Bigg]\Bigg\}\\
    &&\text{here, the equal sign holds when }\exists 1 \le t_1 < t_2 \le k,\\
    &&\forall P({y_j}_{x_j}\mid z) = 0 \text{ for } j \in [1, k], j\ne t_1.\\
    &=& \frac{1}{2} \times \Bigg\{\Bigg[P({y_{t_0}}_{x_{t_0}}\mid z) + P({y_{t_1}}_{x_{t_1}}\mid z) \\
    &&+ P({y_{t_2}}_{x_{t_2}}\mid z) - P({x_{t_0}, y_{t_0}}\mid z) \\
    &&- P({x_{t_1}, y_{t_1}}\mid z) - P({x_{t_2}, y_{t_2}}\mid z) \\
    &&- \sum_{\substack{\{i_1,...,i_{t_0-1},i_{t_0+1},...,i_{k+2}\} \\ \in \{1,...,n\}^{k+1}}}{P({y_{i_1}}_{x_1},...,{y_{i_{t_0-1}}}_{x_{t_0-1}},}\\
    &&{{y_{t_0}}_{x_{t_0}},{y_{i_{t_0+1}}}_{x_{t_0+1}},...,{y_{i_{k}}}_{x_{k}},x_{i_{k+1}},y_{i_{k+2}}\mid z)} \\
    &&+ \bigg[\sum_{j=1}^{t_0-1}{P({y_{t_0}}_{x_{t_0}}, x_j, {y_{j+t_1-t_0}}_{x_{j+t_1-t_0}}\mid z)}
    \\
    &&+ \sum_{j=t_0+1}^{k-t_1+t_0} P({y_{t_0}}_{x_{t_0}}, x_{j}, {y_{j+t_1-t_0}}_{x_{j+t_1-t_0}}\mid z) \\
    &&+ \sum_{j=1}^{t_1-t_0}{P({y_{t_0}}_{x_{t_0}}, x_{k-t_1+t_0+j}, {y_{j}}_{x_{j}}\mid z)} \\
    &&+ \sum_{j=k+1}^{n} P({y_{1}}_{x_{1}},...,{y_{k}}_{x_{k}}, x_{j}\mid z) \\
    &&+ \sum_{\substack{\{i_1,...,i_{t_0-1},i_{t_0+1},...,i_{k}\} \\ \in \{1,...,n\}^{k-1}}}{P({y_{i_1}}_{x_1},...,{y_{i_{t_0-1}}}_{x_{t_0-1}},}\\
    &&{{y_{t_0}}_{x_{t_0}},{y_{i_{t_0+1}}}_{x_{t_0+1}},...,{y_{i_{k}}}_{x_{k}}, x_{t_0}, y_{t_0}\mid z)} \bigg]
\end{eqnarray*}
\begin{eqnarray*}   
    &&- \sum_{\substack{\{i_1,...,i_{t_1-1},i_{t_1+1},...,i_{k+2}\} \\ \in \{1,...,n\}^{k+1}}}{P({y_{i_1}}_{x_1},...,{y_{i_{t_1-1}}}_{x_{t_1-1}},}\\
    &&{{y_{t_1}}_{x_{t_1}},{y_{i_{t_1+1}}}_{x_{t_1+1}},...,{y_{i_{k}}}_{x_{k}},x_{i_{k+1}},y_{i_{k+2}}\mid z)} \\
    &&+ \bigg[{P({y_{t_0}}_{x_{t_0}}, x_{t_0}, {y_{t_1}}_{x_{t_1}}\mid z)} \\
    &&+ {P({y_{t_1}}_{x_{t_1}}, x_{t_2}, {y_{t_2}}_{x_{t_2}}\mid z)} \\
    &&+ \sum_{\substack{\{i_1,...,i_{t_1-1},i_{t_1+1},...,i_{k}\} \\ \in \{1,...,n\}^{k-1}}}{P({y_{i_1}}_{x_1},...,{y_{i_{t_1-1}}}_{x_{t_1-1}},}\\
    &&{{y_{t_1}}_{x_{t_1}},{y_{i_{t_1+1}}}_{x_{t_1+1}},...,{y_{i_{k}}}_{x_{k}}, x_{t_1}, y_{t_1}\mid z)}\bigg]\Bigg] \\
    &&- \sum_{\substack{\{i_1,...,i_{t_2-1},i_{t_2+1},...,i_{k+2}\} \\ \in \{1,...,n\}^{k+1}}}{P({y_{i_1}}_{x_1},...,{y_{i_{t_2-1}}}_{x_{t_2-1}},}\\
    &&{{y_{t_2}}_{x_{t_2}},{y_{i_{t_2+1}}}_{x_{t_2+1}},...,{y_{i_{k}}}_{x_{k}},x_{i_{k+1}},y_{i_{k+2}}\mid z)} \\ 
    &&+ \Bigg[\sum_{\substack{\{i_1,...,i_{t_2-1},i_{t_2+1},...,i_{k}\} \\ \in \{1,...,n\}^{k-1}}}{P({y_{i_1}}_{x_1},...,{y_{i_{t_2-1}}}_{x_{t_2-1}},}\\
    &&{{y_{t_2}}_{x_{t_2}},{y_{i_{t_2+1}}}_{x_{t_2+1}},...,{y_{i_{k}}}_{x_{k}}, x_{t_2}, y_{t_2}\mid z)} \\
    &&+ \sum_{\substack{1 \le j \le k \\ j \ne t_2}}{P({y_{1}}_{x_{1}},...,{y_{k}}_{x_{k}}, x_{j}\mid z)} \\
    &&+ \sum_{j=k+1}^{n} P({y_{1}}_{x_{1}},...,{y_{k}}_{x_{k}}, x_{j}\mid z)\Bigg]\Bigg\}\\
    &\le& \frac{1}{2} \times \Big[P({y_{t_0}}_{x_{t_0}}\mid z) + P({y_{t_1}}_{x_{t_1}}\mid z) \\
    &&+ P({y_{t_2}}_{x_{t_2}}\mid z) - P({x_{t_0}, y_{t_0}}\mid z) \\
    &&- P({x_{t_1}, y_{t_1}}\mid z) - P({x_{t_2}, y_{t_2}}\mid z)\Big]
\end{eqnarray*}
Here, the equal sign holds when $\exists 1\le t_0 < t_1 \le k$, and $P({y_{t_0}}_{x_{t_0}}\mid z)=P({y_{t_1}}_{x_{t_1}}\mid z)=0$. This is just a sufficient condition illustrating tightness, while the equality can also hold in other cases.

For $3 \le m \le k-1$, W.L.O.G., let $1 \le t_0 < ... < t_{m} \le k$. 
By applying equation (\ref{pns(k)1/m*m}), we can get
\begin{eqnarray*}
    &&P({y_{1}}_{x_{1}},...,{y_{k}}_{x_{k}}\mid z) \\
    &=& \frac{1}{m} \times \Big[\sum_{j=1}^{n} P({y_{1}}_{x_{1}},...,{y_{k}}_{x_{k}}, x_{j}\mid z) \\
    &&+ (m-1) \sum_{j=1}^{n} P({y_{1}}_{x_{1}},...,{y_{k}}_{x_{k}}, x_{j}\mid z)\Big]\\
\end{eqnarray*}
By applying equation (\ref{t0t1}), we have
\begin{eqnarray*}
    &&P({y_{1}}_{x_{1}},...,{y_{k}}_{x_{k}}\mid z)\\
    &\le& \frac{1}{m} \times \Bigg[P({y_{t_0}}_{x_{t_0}}\mid z) + P({y_{t_1}}_{x_{t_1}}\mid z) \\ 
    &&- P({x_{t_0}, y_{t_0}}\mid z) - P({x_{t_1}, y_{t_1}}\mid z) \\ 
    &&+ (m-1) \sum_{j=1}^{n} P({y_{1}}_{x_{1}},...,{y_{k}}_{x_{k}}, x_{j}\mid z)\Bigg]\\
\end{eqnarray*}
Similar to the proof when $m = 2$, we can derive results for $3\le m\le k-1$:
\begin{eqnarray*}
    &&P({y_{1}}_{x_{1}},...,{y_{k}}_{x_{k}}\mid z)\\
    &\le& \frac{1}{m} \times \Bigg[P({y_{t_0}}_{x_{t_0}}\mid z) + P({y_{t_1}}_{x_{t_1}}\mid z) \\ 
    &&- P({x_{t_0}, y_{t_0}}\mid z) - P({x_{t_1}, y_{t_1}}\mid z) \\ 
    &&+ \sum_{j=2}^{m} \left[P({y_{t_j}}_{x_{t_j}}\mid z) - P({x_{t_j}, y_{t_j}}\mid z)\right]\Bigg]\\
    &=& \frac{1}{m} \times \sum_{j=0}^{m} \left[P({y_{t_j}}_{x_{t_j}}\mid z) - P({x_{t_j}, y_{t_j}}\mid z)\right]
\end{eqnarray*}
Overall, we have proved the third bound for $m\in [1,k-1]$.\\ \\
By substituting 
\begin{align*}
{\max \left \{
\begin{array}{cc}
0, \\
\\
\displaystyle \sum_{j = 1}^{k}P({y_{j}}_{x_{j}}\mid z) - k + 1, \\
\\
\text{For }i \in  \{1, ..., k\}:\\
\quad \displaystyle \sum_{\substack{1 \le j \le k \\ j \ne i}}
\Big[P({y_{j}}_{x_{j}}\mid z)+P({x_{j}\mid z})-\\
-P({x_{j}, y_{j}}\mid z)\Big]+ P({x_{i}, y_{i}}\mid z) - k + 1
\end{array}
\right \}\nonumber
} \\ \le  P({y_{1}}_{x_{1}},...,{y_{k}}_{x_{k}}\mid z)
\end{align*}
\begin{align*}
{\min \left \{
\begin{array}{cc}
\displaystyle \sum_{j = 1}^{k}P({x_{j}, y_{j}}\mid z) +\displaystyle\sum_{j = k+1}^{n}P({x_{j}}\mid z), \\
\text{For }j \in \{1,...,k\}:\\
P({y_{j}}_{x_{j}}\mid z),\\
\\
\text{For }m\in\{1,...,k-1\},\\
t_j\in\{1,...,k\}:\\
\displaystyle \frac{1}{m}\Big[\sum_{j = 0}^{m}P({y_{t_j}}_{x_{t_j}}\mid z) -P({x_{t_j}, y_{t_j}\mid z})\Big]
\end{array} 
\right \}\nonumber
} \\ \ge  P({y_{1}}_{x_{1}},...,{y_{k}}_{x_{k}}\mid z)
\end{align*}
into equation~(\ref{pnsk1*z}), and applying law of total probability, we have
\begin{eqnarray*}
    &&\sum_{z}{P({y_{1}}_{x_{1}},...,{y_{k}}_{x_{k}}\mid z)}\times P(z) \\
    &\ge& \sum_{z}{0}\times P(z) \\
    &=& 0 \\
    &&\sum_{z}{P({y_{1}}_{x_{1}},...,{y_{k}}_{x_{k}}\mid z)}\times P(z) \\
    &\ge& \sum_{z}{\left[\displaystyle \sum_{j = 1}^{k}P({y_{j}}_{x_{j}}\mid z) - k + 1\right]}\times P(z) \\
    &=& \displaystyle \sum_{j = 1}^{k}P({y_{j}}_{x_{j}}) - k + 1\\
    &&\sum_{z}{P({y_{1}}_{x_{1}},...,{y_{k}}_{x_{k}}\mid z)}\times P(z) \\
    &\ge& \sum_{z}{\Bigg\{\displaystyle \sum_{\substack{1 \le j \le k \\ j \ne i}} \Big[P({y_{j}}_{x_{j}}\mid z)+ P({x_{j}\mid z})}\\
    &&{-P({x_{j}, y_{j}}\mid z)\Big]+ P({x_{i}, y_{i}}\mid z) - k + 1\Bigg\}}\times P(z) \\
    &=& \displaystyle \sum_{\substack{1 \le j \le k \\ j \ne i}} \left[P({y_{j}}_{x_{j}})+P({x_{j}})-P({x_{j}, y_{j}})\right] \\
    &&+ P({x_{i}, y_{i}}) - k + 1
\end{eqnarray*}

\begin{eqnarray*}
    &&\sum_{z}{P({y_{1}}_{x_{1}},...,{y_{k}}_{x_{k}}\mid z)}\times P(z)\\ &\le& \sum_{z}{\left[\displaystyle \sum_{j = 1}^{k}P({x_{j}, y_{j}}\mid z) + \sum_{j = k+1}^{n}P({x_{j}}\mid z)\right]}\times P(z) \\
    &=& \displaystyle \sum_{j = 1}^{k}P({x_{j}, y_{j}}) + \sum_{j = k+1}^{n}P({x_{j}}) \\
    &&\sum_{z}{P({y_{1}}_{x_{1}},...,{y_{k}}_{x_{k}}\mid z)}\times P(z) \\&\le& \sum_{z}{P({y_{j}}_{x_{j}}\mid z)}\times P(z) \\
    &=& P({y_{j}}_{x_{j}})\\
    &&\sum_{z}{P({y_{1}}_{x_{1}},...,{y_{k}}_{x_{k}}\mid z)}\times P(z) \\&\le& \sum_{z}{\left\{\displaystyle \frac{1}{m}\left[\sum_{j = 0}^{m}P({y_{t_j}}_{x_{t_j}}\mid z) - P({x_{t_j}, y_{t_j}\mid z})\right]\right\}}\\
    &&\times P(z) \\
    &=& \displaystyle \frac{1}{m}\left[\sum_{j = 0}^{m}P({y_{t_j}}_{x_{t_j}}) - P({x_{t_j}, y_{t_j}})\right] \\
\end{eqnarray*}
Thus, we can guarantee that the proposed lower and upper bounds are no looser than the lower and upper bounds in \citet{shu2026identificationprobabilitiescausationrecursive}. Following the notation in \citet{mueller2021causeseffectslearningindividual}, $P({y_j}_{x_j}\mid z)$ denotes experimental data within the subpopulation characterized by $z$.
\end{proof}

\subsubsection{Proof of Theorem~\ref{nnk+x_p1}}
\begin{proof}
\begin{eqnarray}
    PSub(k,p) &=& P({y_{1}}_{x_{1}},...,{y_{k}}_{x_{k}},x_p)\nonumber\\
    &=& \sum_{z}{P({y_{1}}_{x_{1}},...,{y_{k}}_{x_{k}},x_p\mid z)}\nonumber \\
    &&\times P(z) \label{nnkxp1}
\end{eqnarray}

First, we need to prove:
\begin{align*}
{\max \left \{
\begin{array}{cc}
0, \\
\displaystyle \sum_{j=1}^{k}\Big[P({y_{j}}_{x_{j}}\mid z)+P({x_{j}}\mid z)-\\-P({x_{j}, y_{j}}\mid z)\Big] + P({x_{p}}\mid z) - k
\end{array}
\right \}\nonumber
} \\ \le {P({y_{1}}_{x_{1}},...,{y_{k}}_{x_{k}},x_p\mid z)}
\end{align*}
\begin{align*}
{\min \left \{
\begin{array}{cc}
P({x_{p}}\mid z),\\
\\
\text{For } j \in \{1,...,k\}:\\
P({y_{j}}_{x_{j}}\mid z) - P({x_{j}, y_{j}}\mid z)
\end{array} 
\right \}\nonumber
} \\ \ge {P({y_{1}}_{x_{1}},...,{y_{k}}_{x_{k}},x_p\mid z)}
\end{align*}

By the Fréchet Inequalities, for any $z$ with $P(z)>0$, we can obtain the first lower bound and the first upper bound, 
\begin{eqnarray*}
P({y_{1}}_{x_{1}},...,{y_{k}}_{x_{k}},{x_{p}}\mid z)&\ge& 0\\
P({y_{1}}_{x_{1}},...,{y_{k}}_{x_{k}},{x_{p}}\mid z)&\le& {P({x_{p}}\mid z)}.
\end{eqnarray*}
The equality of the first lower bound holds when 
$\exists j\in [1,k], \text{that } P({y_j}_{x_j}\mid z)=0$ or $P(x_p\mid z) = 0$, $p\ne j$.

The equality of the first upper bound holds when $P({y_{j}}_{x_{j}}\mid z) = 1$ for $\forall j \in [1,k]$, $j\ne p$.

For the second lower bound
\begin{eqnarray*}
&&P({y_{1}}_{x_{1}},...,{y_{k}}_{x_{k}}, x_p\mid z)\\
&=& P({y_{1}}_{x_{1}},...,{y_{k}}_{x_{k}}, x_p\mid z)\\ 
&&+ \sum_{j=1}^{n}{P({y_{1}}_{x_{1}},...,{y_{k}}_{x_{k}}, x_j\mid z)} \\
&&- \sum_{j=1}^{n}{P({y_{1}}_{x_{1}},...,{y_{k}}_{x_{k}}, x_j\mid z)}\\
&=& {P({y_{1}}_{x_{1}},...,{y_{k}}_{x_{k}}\mid z)} + P({y_{1}}_{x_{1}},...,{y_{k}}_{x_{k}}, x_p\mid z) \\
&&- \sum_{j=1}^{n}{P({y_{1}}_{x_{1}},...,{y_{k}}_{x_{k}}, x_j\mid z)}
\end{eqnarray*}
\begin{eqnarray*}
&=& {P({y_{1}}_{x_{1}},...,{y_{k}}_{x_{k}}\mid z)} + P({y_{1}}_{x_{1}},...,{y_{k}}_{x_{k}}, x_p\mid z) \\
&&+ \sum_{j=1}^{n}{P(x_j\mid z)} - 1 - \sum_{j=1}^{n}{P({y_{1}}_{x_{1}},...,{y_{k}}_{x_{k}}, x_j\mid z)}\\
&=& {P({y_{1}}_{x_{1}},...,{y_{k}}_{x_{k}}\mid z)} - 1\\
&&+ P({y_{1}}_{x_{1}},...,{y_{k}}_{x_{k}}, x_p\mid z) \\
&&+ \sum_{j=1}^{k}{P(x_j\mid z)} - \sum_{j=1}^{k}{P({y_{1}}_{x_{1}},...,{y_{k}}_{x_{k}}, x_j\mid z)} \\
&&+ \sum_{j=k+1}^{n}{P(x_j\mid z)} \\
&&- \sum_{j=k+1}^{n}{P({y_{1}}_{x_{1}},...,{y_{k}}_{x_{k}}, x_j\mid z)}\\
&=& {P({y_{1}}_{x_{1}},...,{y_{k}}_{x_{k}}\mid z)} - 1 \\
&&+ \sum_{j=1}^{k}{P(x_j\mid z)} \\
&&- \sum_{j=1}^{k}{P({y_{1}}_{x_{1}},...,{y_{j-1}}_{x_{j-1}},{y_{j+1}}_{x_{j+1}},...,{y_{k}}_{x_{k}},}\\
&&{x_j, y_j\mid z)} \\
&&+ \sum_{j=k+1}^{n}{P(x_j\mid z)} \\
&&- \sum_{j=k+1}^{n}{P({y_{1}}_{x_{1}},...,{y_{k}}_{x_{k}}, x_j\mid z)} \\
&&+ P({y_{1}}_{x_{1}},...,{y_{k}}_{x_{k}}, x_p\mid z)\\
&=& {P({y_{1}}_{x_{1}},...,{y_{k}}_{x_{k}}\mid z)} - 1 \\
&&+ \sum_{j=1}^{k}{P(x_j\mid z)} \\
&&- \sum_{j=1}^{k}{P({y_{1}}_{x_{1}},...,{y_{j-1}}_{x_{j-1}},{y_{j+1}}_{x_{j+1}},...,{y_{k}}_{x_{k}},}\\
&&{x_j, y_j\mid z)} \\
&&+ \sum_{j=k+1}^{n}{P(x_j\mid z)} \\
&&- \sum_{\substack{k+1 \le j \le n, \\ j \ne p}}{P({y_{1}}_{x_{1}},...,{y_{k}}_{x_{k}}, x_j\mid z)} \\
&\ge& \sum_{j=1}^{k}{P({y_j}_{x_j}\mid z)} - (k-1) - 1 + \sum_{j=1}^{k}{P(x_j\mid z)} \\
&&- \sum_{j=1}^{k}{P(x_j, y_j\mid z)} + \sum_{j=k+1}^{n}{P(x_j\mid z)} \\
\end{eqnarray*}
\begin{eqnarray*}
&&- \sum_{\substack{k+1 \le j \le n, \\ j \ne p}}{P(x_j\mid z)}\\
&&\text{here, the equal sign holds when }
\exists j \in [1,n], \\
&&\text{ and } P({y_i}_{x_i}\mid z)=1 \text{ for } \forall i \in [1,k], i\ne j.\\
&=& \sum_{j=1}^{k}{P({y_j}_{x_j}\mid z)} - k + \sum_{j=1}^{k}{P(x_j\mid z)} \\
&&- \sum_{j=1}^{k}{P(x_j, y_j\mid z)} + P(x_p\mid z)\\
&=& \sum_{j=1}^{k}\left[P({y_{j}}_{x_{j}}\mid z)+P({x_{j}}\mid z)-P({x_{j}, y_{j}}\mid z)\right] \\
&&+ P({x_{p}}\mid z) - k
\end{eqnarray*}
The equality of the second lower bound holds when $\exists j \in [1,n], P({y_i}_{x_i}\mid z)=1 \text{ for } \forall i \in [1,k], i\ne j$.\\

For the remaining upper bounds, $\forall j \in [1,k]$:
\begin{eqnarray*}
&&P({y_{1}}_{x_{1}},...,{y_{k}}_{x_{k}}, x_p\mid z)\\
&=& P({y_{1}}_{x_{1}},...,{y_{k}}_{x_{k}}, x_p\mid z) + P({y_j}_{x_j}\mid z)\\
&&- P({y_j}_{x_j}\mid z)\\
&=& P({y_{1}}_{x_{1}},...,{y_{k}}_{x_{k}}, x_p\mid z) + P({y_j}_{x_j}\mid z) \\
&&- \sum_{\substack{\{i_1,...,i_{j-1},i_{j+1},...,i_{k+1}\} \\ \in \{1,...,n\}^{k}}}{P({y_{i_1}}_{x_1},...,{y_{i_{j-1}}}_{x_{j-1}},}\\
&&{{y_{j}}_{x_{j}},{y_{i_{j+1}}}_{x_{j+1}},...,{y_{i_{k}}}_{x_{k}},x_{i_{k+1}}\mid z)}
\end{eqnarray*}
Since $p \ne j$ for $1\le j \le k$, 
\begin{eqnarray*}
&&P({y_{1}}_{x_{1}},...,{y_{k}}_{x_{k}}, x_p\mid z)\\
&\le& P({y_j}_{x_j}\mid z) \\
&&- \sum_{\substack{\{i_1,...,i_{j-1},i_{j+1},...,i_{k}\} \\ \in \{1,...,n\}^{k-1}}}{P({y_{i_1}}_{x_1},...,{y_{i_{j-1}}}_{x_{j-1}},{y_{j}}_{x_{j}},}\\
&&{{y_{i_{j+1}}}_{x_{j+1}},...,{y_{i_{k}}}_{x_{k}},x_{j}\mid z)}\\
&&\text{here, the equal sign holds when } P({x_j}\mid z)=0 \\
&&\text{ for } \forall j \in [k+1, n].\\
&=& P({y_j}_{x_j}\mid z) - P({y_{j}}_{x_{j}}, x_{j}\mid z)\\
&=& P({y_{j}}_{x_{j}}\mid z) - P({x_{j}, y_{j}}\mid z)
\end{eqnarray*}
The equality of the second upper bound holds when $P({x_j}\mid z)=0 \text{ for } \forall j \in [k+1, n]$.\\ \\
By substituting 
\begin{align*}
{\max \left \{
\begin{array}{cc}
0, \\
\displaystyle \sum_{j=1}^{k}\Big[P({y_{j}}_{x_{j}}\mid z)+P({x_{j}}\mid z)-\\-P({x_{j}, y_{j}}\mid z)\Big] + P({x_{p}}\mid z) - k
\end{array}
\right \}\nonumber
} \\ \le {P({y_{1}}_{x_{1}},...,{y_{k}}_{x_{k}},x_p\mid z)}
\end{align*}
\begin{align*}
{\min \left \{
\begin{array}{cc}
P({x_{p}}\mid z),\\
\\
\text{For } j \in \{1,...,k\}:\\
P({y_{j}}_{x_{j}}\mid z) - P({x_{j}, y_{j}}\mid z)
\end{array} 
\right \}\nonumber
} \\ \ge {P({y_{1}}_{x_{1}},...,{y_{k}}_{x_{k}},x_p\mid z)}
\end{align*}
into equation~(\ref{nnkxp1}), and applying law of total probability, we have
\begin{eqnarray*}
    &&\sum_{z}{P({y_{1}}_{x_{1}},...,{y_{k}}_{x_{k}},x_p\mid z)} \times P(z)\\
    &\ge& \sum_{z}{0}\times P(z) \\
    &=& 0 \\
    &&\sum_{z}{P({y_{1}}_{x_{1}},...,{y_{k}}_{x_{k}},x_p\mid z)} \times P(z)\\
    &\ge& \sum_{z}{\Bigg\{\displaystyle \sum_{j=1}^{k}\Big[P({y_{j}}_{x_{j}}\mid z)+P({x_{j}}\mid z)}\\
    &&{-P({x_{j}, y_{j}}\mid z)\Big] + P({x_{p}}\mid z) - k\Bigg\}}\times P(z) \\
    &=& \displaystyle \sum_{j=1}^{k}\left[P({y_{j}}_{x_{j}})+P({x_{j}})-P({x_{j}, y_{j}})\right] \\
    &&+ P({x_{p}}) - k
\end{eqnarray*}

\begin{eqnarray*}
    &&\sum_{z}{P({y_{1}}_{x_{1}},...,{y_{k}}_{x_{k}},x_p\mid z)} \times P(z)\\ 
    &\le& P({x_{p}}\mid z) \times P(z) \\
    &=& P({x_{p}}) \\
    &&\sum_{z}{P({y_{1}}_{x_{1}},...,{y_{k}}_{x_{k}},x_p\mid z)} \times P(z) \\
    &\le& \left[P({y_{j}}_{x_{j}}\mid z) - P({x_{j}, y_{j}}\mid z)\right]\times P(z) \\
    &=& P({y_{j}}_{x_{j}}) - P({x_{j}, y_{j}})\\
\end{eqnarray*}
Thus, we can guarantee that the proposed lower and upper bounds are no looser than those of \citet{shu2026identificationprobabilitiescausationrecursive}. Following the notation in \citet{mueller2021causeseffectslearningindividual}, $P({y_j}_{x_j}\mid z)$ denotes experimental data within the subpopulation characterized by $z$.
\end{proof}

\subsubsection{Proof of Theorem~\ref{nnk+y_q1}}
\begin{proof}
\begin{eqnarray}
    PRep(k,q) &=& P({y_{1}}_{x_{1}},...,{y_{k}}_{x_{k}}, y_q)\nonumber\\
    &=& \sum_{z}{P({y_{1}}_{x_{1}},...,{y_{k}}_{x_{k}},y_q\mid z)}\nonumber \\
    &&\times P(z) \label{nnkyq1}
\end{eqnarray}

First, we need to prove:
\begin{align*}
{\max \left \{
\begin{array}{cc}
0, \\
\\
\displaystyle \sum_{j=1}^{k}\Big[P({y_{j}}_{x_{j}}\mid z) +P({x_{j}}\mid z)-\\
-P({x_{j}, y_{j}}\mid z)\Big]
+\displaystyle \sum_{\substack{k+1 \le j \le n \\ j \ne q}}{P({x_{j}, y_{q}}\mid z)} +\\
+ P({x_{q}, y_{q}}\mid z) - k,\\
\\
\text{If } q \in \{1,...,k\}:\\
\displaystyle \sum_{\substack{1 \le j \le k \\ j \ne q}}\Big[P({y_{j}}_{x_{j}}\mid z)+P({x_{j}}\mid z)-\\
-P({x_{j}, y_{j}}\mid z)\Big]+ P({x_{q}, y_{q}}\mid z) - (k-1) 
\end{array}
\right \}\nonumber
} \\ \le  P({y_{1}}_{x_{1}},...,{y_{k}}_{x_{k}},y_q\mid z)
\end{align*}
\begin{align*}
{\min \left \{
\begin{array}{cc}
\text{If } q \in \{1,...,k\}:\\
P({y_{q}}_{x_{q}}\mid z), \\
\\
P({x_{q}, y_{q}}\mid z) +\displaystyle \sum_{\substack{k+1 \le j \le n \\ j \ne q}}P({x_{j}, y_{q}}\mid z),\\
\\
\text{For }j \in \{1,...,k\},\text{ and } j \ne q\\
P({y_{j}}_{x_{j}}\mid z) - P({x_{j}, y_{j}}\mid z),\\
\end{array} 
\right \}\nonumber
} \\ \ge  P({y_{1}}_{x_{1}},...,{y_{k}}_{x_{k}},y_q\mid z)
\end{align*}

By the Fréchet Inequalities, for any $z$ with $P(z)>0$, we can obtain the first lower bound and the first upper bound, 
\begin{eqnarray*}
P({y_{1}}_{x_{1}},...,{y_{k}}_{x_{k}},{y_{q}}\mid z)&\ge& 0\\
P({y_{1}}_{x_{1}},...,{y_{k}}_{x_{k}},{y_{q}}\mid z)&\le& {P({y_{q}}_{x_{q}}\mid z)}, \forall 1 \le q \le k.
\end{eqnarray*}
The equality of the first lower bound holds when
$\exists j\in [1,k], \text{that } P({y_j}_{x_j}\mid z)=0$ or $P(y_q\mid z) = 0$.

The equality of the first upper bound holds when $P({y_{j}}_{x_{j}}\mid z) = 1, P(y_q\mid z) = 1$ for $\forall j \in [1,k], j\ne q$.

For the second lower bound
\begin{eqnarray*}
&&P({y_{1}}_{x_{1}},...,{y_{k}}_{x_{k}}, y_q\mid z)\\
&=& P({y_{1}}_{x_{1}},...,{y_{k}}_{x_{k}}, y_q\mid z) + P({y_{1}}_{x_{1}},...,{y_{k}}_{x_{k}}\mid z) \\
&&- \sum_{j=1}^{n}\sum_{l=1}^{n}{P({y_{1}}_{x_{1}},...,{y_{k}}_{x_{k}}, x_j, y_l\mid z)}\\
&=& P({y_{1}}_{x_{1}},...,{y_{k}}_{x_{k}}, y_q\mid z) + P({y_{1}}_{x_{1}},...,{y_{k}}_{x_{k}}\mid z) \\
&&- \sum_{j=1}^{k}{P({y_{1}}_{x_{1}},...,{y_{k}}_{x_{k}}, x_j, y_j\mid z)} \\
&&- \sum_{j=k+1}^{n}\sum_{l=1}^{n}{P({y_{1}}_{x_{1}},...,{y_{k}}_{x_{k}}, x_j, y_l\mid z)}
\end{eqnarray*}
\begin{eqnarray*}
&=& P({y_{1}}_{x_{1}},...,{y_{k}}_{x_{k}}, y_q\mid z) + P({y_{1}}_{x_{1}},...,{y_{k}}_{x_{k}}\mid z) \\
&&+ \sum_{j=1}^{n}{P(x_j\mid z)} - 1\\
&&- \sum_{j=1}^{k}{P({y_{1}}_{x_{1}},...,{y_{k}}_{x_{k}}, x_j, y_j\mid z)} \\
&&- \sum_{j=k+1}^{n}\sum_{l=1}^{n}{P({y_{1}}_{x_{1}},...,{y_{k}}_{x_{k}}, x_j, y_l\mid z)}\\
&=& \sum_{j=1}^{n}{P({y_{1}}_{x_{1}},...,{y_{k}}_{x_{k}}, x_j,y_q\mid z)} \\
&&+ P({y_{1}}_{x_{1}},...,{y_{k}}_{x_{k}}\mid z) - 1\\
&&+ \sum_{j=1}^{k}{P(x_j\mid z)} - \sum_{j=1}^{k}{P({y_{1}}_{x_{1}},...,{y_{k}}_{x_{k}}, x_j, y_j\mid z)} \\
&&+ \sum_{j=k+1}^{n}{P(x_j\mid z)} \\
&&- \sum_{j=k+1}^{n}\sum_{l=1}^{n}{P({y_{1}}_{x_{1}},...,{y_{k}}_{x_{k}}, x_j, y_l\mid z)}
\end{eqnarray*}
If $q\in [1,k]$,
\begin{eqnarray*}
&&P({y_{1}}_{x_{1}},...,{y_{k}}_{x_{k}}, y_q\mid z)\\
&=& P({y_{1}}_{x_{1}},...,{y_{k}}_{x_{k}}\mid z) - 1\\
&&+ \sum_{j=1}^{k}{P(x_j\mid z)} - \sum_{j=1}^{k}{P({y_{1}}_{x_{1}},...,{y_{k}}_{x_{k}}, x_j, y_j\mid z)} \\
&&+ {P({y_{1}}_{x_{1}},...,{y_{k}}_{x_{k}}, x_q,y_q\mid z)} + \sum_{j=k+1}^{n}{P(x_j\mid z)} \\
&&- \sum_{j=k+1}^{n}\sum_{l=1}^{n}{P({y_{1}}_{x_{1}},...,{y_{k}}_{x_{k}}, x_j, y_l\mid z)} \\
&&+ \sum_{j=k+1}^{n}{P({y_{1}}_{x_{1}},...,{y_{k}}_{x_{k}}, x_j,y_q\mid z)}\\
&=& P({y_{1}}_{x_{1}},...,{y_{k}}_{x_{k}}\mid z) - 1\\
&&+ \sum_{j=1}^{k}{P(x_j\mid z)} \\
&&- \sum_{\substack{1\le j \le k,\\ j\ne q}}{P({y_{1}}_{x_{1}},...,{y_{k}}_{x_{k}}, x_j, y_j\mid z)}\\
&&+ \sum_{j=k+1}^{n}{P(x_j\mid z)} \\
&&- \sum_{j=k+1}^{n}\sum_{\substack{1\le l \le n,\\l\ne q}}{P({y_{1}}_{x_{1}},...,{y_{k}}_{x_{k}}, x_j, y_l\mid z)} 
\end{eqnarray*}
\begin{eqnarray*}
&\ge& \sum_{j=1}^{k}{P({y_j}_{x_j}\mid z)} - (k-1) - 1\\
&&+ \sum_{j=1}^{k}{P(x_j\mid z)} - \sum_{\substack{1\le j
\le k,\\ j\ne q}}{P(x_j, y_j\mid z)} \\
&&+ \sum_{j=k+1}^{n}{P(x_j\mid z)} - \sum_{j=k+1}^{n}\sum_{\substack{1\le l\le n,\\l\ne q}}{P(x_j, y_l\mid z)} \\
&&\text{here, the equal sign holds when }
\exists j \in [1,n], \\
&&\text{ and } P({y_i}_{x_i}\mid z)=1 \text{ for } \forall i \in [1,k], i\ne j.\\
&=& \sum_{j=1}^{k}{P({y_j}_{x_j}\mid z)} - k + \sum_{j=1}^{k}{P(x_j\mid z)} \\
&&- \sum_{j=1}^{k}{P(x_j, y_j\mid z)} + P(x_q, y_q\mid z) \\
&&+ \sum_{j=k+1}^{n}{P(x_j,y_q\mid z)}\\
&=& \sum_{j=1}^{k}\left[P({y_{j}}_{x_{j}}\mid z)+P({x_{j}}\mid z)-P({x_{j}, y_{j}}\mid z)\right] \\
&&+ \sum_{\substack{k+1\le j \le n, \\ j \ne q}}{P({x_{j}, y_{q}}\mid z)} + P({x_{q}, y_{q}}\mid z) - k
\end{eqnarray*}
If $q\in [k+1,n]$,
\begin{eqnarray*}
&&P({y_{1}}_{x_{1}},...,{y_{k}}_{x_{k}}, y_q\mid z)\\
&=& P({y_{1}}_{x_{1}},...,{y_{k}}_{x_{k}}\mid z) - 1+ \sum_{j=1}^{k}{P(x_j\mid z)} \\
&&- \sum_{j=1}^{k}{P({y_{1}}_{x_{1}},...,{y_{k}}_{x_{k}}, x_j, y_j\mid z)} \\
&&+ \sum_{j=k+1}^{n}{P(x_j\mid z)} \\
&&- \sum_{j=k+1}^{n}\sum_{l=1}^{n}{P({y_{1}}_{x_{1}},...,{y_{k}}_{x_{k}}, x_j, y_l\mid z)} \\
&&+ \sum_{j=k+1}^{n}{P({y_{1}}_{x_{1}},...,{y_{k}}_{x_{k}}, x_j,y_q\mid z)}\\
&=& P({y_{1}}_{x_{1}},...,{y_{k}}_{x_{k}}\mid z) - 1 + \sum_{j=1}^{k}{P(x_j\mid z)} \\
&&- \sum_{j=1}^{k}{P({y_{1}}_{x_{1}},...,{y_{k}}_{x_{k}}, x_j, y_j\mid z)} \\
&&+ \sum_{j=k+1}^{n}{P(x_j\mid z)} \\
\end{eqnarray*}
\begin{eqnarray*}
&&- \sum_{j=k+1}^{n}\sum_{\substack{1\le l\le n,\\l\ne q}}{P({y_{1}}_{x_{1}},...,{y_{k}}_{x_{k}}, x_j, y_l\mid z)}\\
&\ge& \sum_{j=1}^{k}{P({y_j}_{x_j}\mid z)} - (k-1) - 1 + \sum_{j=1}^{k}{P(x_j\mid z)} \\
&&- \sum_{j=1}^{k}{P(x_j, y_j\mid z)} \\
&&+ \sum_{j=k+1}^{n}{P(x_j\mid z)} - \sum_{j=k+1}^{n}\sum_{\substack{1\le l\le n,\\l\ne q}}{P(x_j, y_l\mid z)} \\
&&\text{here, the equal sign holds when }
\exists j \in [1,n], \\
&&\text{ and } P({y_i}_{x_i}\mid z)=1 \text{ for } \forall i \in [1,k], i\ne j.\\
&=& \sum_{j=1}^{k}{P({y_j}_{x_j}\mid z)} - k + \sum_{j=1}^{k}{P(x_j\mid z)} \\
&&- \sum_{j=1}^{k}{P(x_j, y_j\mid z)} + \sum_{j=k+1}^{n}{P({x_{j}, y_{q}}\mid z)}\\
&=& \sum_{j=1}^{k}\left[P({y_{j}}_{x_{j}}\mid z)+P({x_{j}}\mid z)-P({x_{j}, y_{j}}\mid z)\right] \\
&&+ \sum_{\substack{k+1\le j \le n, \\ j \ne q}}{P({x_{j}, y_{q}}\mid z)} + P({x_{q}, y_{q}}\mid z) - k
\end{eqnarray*}
To summarize
\begin{eqnarray*}
&&P({y_{1}}_{x_{1}},...,{y_{k}}_{x_{k}}, y_q\mid z)\\
&\ge& \sum_{j=1}^{k}\left[P({y_{j}}_{x_{j}}\mid z)+P({x_{j}}\mid z)-P({x_{j}, y_{j}}\mid z)\right]\\
&&+ \sum_{\substack{k+1 \le j \le n \\ j \ne q}}{P({x_{j}, y_{q}}\mid z)} + P({x_{q}, y_{q}}\mid z) - k
\end{eqnarray*}
and the equality of the second lower bound holds when $\exists j \in [1,n], \text{ and } P({y_i}_{x_i}\mid z)=1 \text{ for } \forall i \in [1,k], i\ne j$.

For the third lower bound, if $1\le q\le k$
\begin{eqnarray*}
&&P({y_{1}}_{x_{1}},...,{y_{k}}_{x_{k}}, y_q\mid z)\\
&=& \sum_{j=1}^{n}{P({y_{1}}_{x_{1}},...,{y_{k}}_{x_{k}}, x_j,y_q\mid z)}\\
&=& P({y_{1}}_{x_{1}},...,{y_{q-1}}_{x_{q-1}},{y_{q+1}}_{x_{q+1}},...,{y_{k}}_{x_{k}}, x_q,y_q\mid z) \\
&&+ \sum_{j=k+1}^{n}{P({y_{1}}_{x_{1}},...,{y_{k}}_{x_{k}}, x_j,y_q\mid z)}\\
&=& P({y_{1}}_{x_{1}},...,{y_{q-1}}_{x_{q-1}},{y_{q+1}}_{x_{q+1}},...,{y_{k}}_{x_{k}}, x_q,y_q\mid z) \\
&&+ \sum_{j=k+1}^{n}{P({y_{1}}_{x_{1}},...,{y_{k}}_{x_{k}}, x_j,y_q\mid z)} \\
\end{eqnarray*}
\begin{eqnarray*}
&&+ P({y_{1}}_{x_{1}},...,{y_{q-1}}_{x_{q-1}},{y_{q+1}}_{x_{q+1}},...,{y_{k}}_{x_{k}}\mid z) \\
&&- \sum_{j=1}^{n}\sum_{l=1}^{n}{P({y_{1}}_{x_{1}},...,{y_{q-1}}_{x_{q-1}},{y_{q+1}}_{x_{q+1}},...,{y_{k}}_{x_{k}},}\\ 
&&{x_j, y_l\mid z)}\\
&&+ \sum_{j=1}^{n}{P(x_j\mid z)} - 1\\
&=& P({y_{1}}_{x_{1}},...,{y_{q-1}}_{x_{q-1}},{y_{q+1}}_{x_{q+1}},...,{y_{k}}_{x_{k}}\mid z) -1\\
&&+ \sum_{j=1}^{k}{P(x_j\mid z)} \\
&&- \sum_{\substack{1\le j\le k,\\j\ne q}}{P({y_{1}}_{x_{1}},...,{y_{q-1}}_{x_{q-1}},{y_{q+1}}_{x_{q+1}},...,{y_{k}}_{x_{k}},}\\ 
&&{x_j, y_j\mid z)} \\
&&- \sum_{l=1}^{n}{P({y_{1}}_{x_{1}},...,{y_{q-1}}_{x_{q-1}},{y_{q+1}}_{x_{q+1}},...,{y_{k}}_{x_{k}}, }\\
&&{x_q, y_l\mid z)} \\
&&+ P({y_{1}}_{x_{1}},...,{y_{q-1}}_{x_{q-1}},{y_{q+1}}_{x_{q+1}},...,{y_{k}}_{x_{k}},x_q, \\
&&y_q\mid z)\\
&&+ \sum_{j=k+1}^{n}{P(x_j\mid z)} \\
&&- \sum_{j=k+1}^{n}\sum_{l=1}^{n}{P({y_{1}}_{x_{1}},...,{y_{q-1}}_{x_{q-1}},{y_{q+1}}_{x_{q+1}},...,}\\
&&{{y_{k}}_{x_{k}}, x_j, y_l\mid z)} \\
&&+ \sum_{j=k+1}^{n}{P({y_{1}}_{x_{1}},...,{y_{k}}_{x_{k}}, x_j,y_q\mid z)} \\
&=& P({y_{1}}_{x_{1}},...,{y_{q-1}}_{x_{q-1}},{y_{q+1}}_{x_{q+1}},...,{y_{k}}_{x_{k}}\mid z) -1\\
&&+ \sum_{j=1}^{k}{P(x_j\mid z)} \\
&&- \sum_{\substack{1\le j\le k,\\j\ne q}}{P({y_{1}}_{x_{1}},...,{y_{q-1}}_{x_{q-1}},{y_{q+1}}_{x_{q+1}},...,{y_{k}}_{x_{k}},}\\
&&{x_j, y_j\mid z)} \\
&&- \sum_{\substack{1\le l\le n,\\l\ne q}}{P({y_{1}}_{x_{1}},...,{y_{q-1}}_{x_{q-1}},{y_{q+1}}_{x_{q+1}},...,{y_{k}}_{x_{k}},}\\
&&{x_q, y_l\mid z)} \\
&&+ \sum_{j=k+1}^{n}{P(x_j\mid z)} \\
&&- \sum_{j=k+1}^{n}\sum_{\substack{1\le l \le n,\\l\ne q}}{P({y_{1}}_{x_{1}},...,{y_{q-1}}_{x_{q-1}},{y_{q+1}}_{x_{q+1}},...,}\\
&&{{y_{k}}_{x_{k}}, x_j, y_l\mid z)}
\end{eqnarray*}
\begin{eqnarray*}
&\ge& \sum_{\substack{1\le j\le k,\\ j \ne q}}{P({y_{j}}_{x_{j}}\mid z)} - (k-2) -1\\
&&+ \sum_{j=1}^{k}{P(x_j\mid z)} - \sum_{\substack{1\le j\le k,\\j\ne q}}{P(x_j, y_j\mid z)} \\
&&- \sum_{\substack{1\le l \le n,\\l\ne q}}{P(x_q, y_l\mid z)} + \sum_{j=k+1}^{n}{P(x_j\mid z)} \\
&&- \sum_{j=k+1}^{n}\sum_{\substack{1
\le l\le n,\\l\ne q}}{P(x_j, y_l\mid z)} \\
&&\text{here, the equal sign holds when }
\exists q \in [1,k], \\
&&\text{ and } P({y_i}_{x_i}\mid z)=1 \text{ for } \forall i \in [1,k], i\ne q.\\
&=& \sum_{\substack{1\le j \le k,\\ j \ne q}}{P({y_{j}}_{x_{j}}\mid z)} - (k-1) \\
&&+ \sum_{\substack{1\le j\le k,\\j\ne q}}{P(x_j\mid z)} + P(x_q\mid z) \\
&&- \sum_{\substack{1\le j\le k,\\j\ne q}}{P(x_j, y_j\mid z)} - \sum_{\substack{1\le l\le n,\\l\ne q}}{P(x_q, y_l\mid z)} \\
&&+ \sum_{j=k+1}^{n}{P(x_j, y_q\mid z)} \\
&=& \sum_{\substack{1\le j\le k, \\ j \ne q}}\left[P({y_{j}}_{x_{j}}\mid z)+P({x_{j}}\mid z)-P({x_{j}, y_{j}}\mid z)\right] \\
&&- (k-1)+ P(x_q\mid z) - \sum_{\substack{1\le l\le n,\\l\ne q}}{P(x_q, y_l\mid z)} \\
&&+ \sum_{j=k+1}^{n}{P(x_j, y_q\mid z)} \\
&=& \sum_{\substack{1\le j\le k, \\ j \ne q}}\left[P({y_{j}}_{x_{j}}\mid z)+P({x_{j}}\mid z)-P({x_{j}, y_{j}}\mid z)\right] \\
&&- (k-1) + P(x_q, y_q\mid z) + \sum_{j=k+1}^{n}{P(x_j, y_q\mid z)} \\
&\ge& \sum_{\substack{1\le j\le k, \\ j \ne q}}\left[P({y_{j}}_{x_{j}}\mid z)+P({x_{j}\mid z})-P({x_{j}, y_{j}}\mid z)\right] \\
&&+ P({x_{q}, y_{q}}\mid z) - (k-1)\\
&&\text{here, the equal sign holds when }P({x_j}\mid z)=0 \\
&&\text{ for } \forall j \in [k+1,n].
\end{eqnarray*}
The equality of the third upper bound holds when $\exists q \in [1,k], \text{ and } P({y_i}_{x_i}\mid z)=1, P({x_j}\mid z)=0 \text{ for } \forall i \in [1,k], j\in [k+1, n], i\ne q$.

For the second upper bound
\begin{eqnarray*}
&&P({y_{1}}_{x_{1}},...,{y_{k}}_{x_{k}}, y_q\mid z) \\
&=& \sum_{j=1}^{n}{P({y_{1}}_{x_{1}},...,{y_{k}}_{x_{k}}, x_j,y_q\mid z)}\\
&=& \sum_{j=1}^{k}{P({y_{1}}_{x_{1}},...,{y_{k}}_{x_{k}}, x_j,y_q\mid z)} \\
&&+ \sum_{j=k+1}^{n}{P({y_{1}}_{x_{1}},...,{y_{k}}_{x_{k}}, x_j,y_q\mid z)}
\end{eqnarray*}
If $q\in [1,k]$,
\begin{eqnarray*}
&&P({y_{1}}_{x_{1}},...,{y_{k}}_{x_{k}}, y_q\mid z) \\
&=& P({y_{1}}_{x_{1}},...,{y_{k}}_{x_{k}}, x_q,y_q\mid z) \\
&&+ \sum_{j=k+1}^{n}{P({y_{1}}_{x_{1}},...,{y_{k}}_{x_{k}}, x_j,y_q\mid z)}\\
&\le& P(x_q,y_q\mid z) + \sum_{j=k+1}^{n}{P(x_j,y_q\mid z)}\\
&&\text{here, the equal sign holds when } P({y_i}_{x_i}\mid z)=1 \\
&&\text{ for } \forall i \in [1,k].\\
&=& P({x_{q}, y_{q}}\mid z) + \sum_{\substack{k+1\le j \le n, \\ j \ne q}}P({x_{j}, y_{q}}\mid z)
\end{eqnarray*}
If $q\in [k+1,n]$,
\begin{eqnarray*}
&&P({y_{1}}_{x_{1}},...,{y_{k}}_{x_{k}}, y_q\mid z) \\
&=& \sum_{j=k+1}^{n}{P({y_{1}}_{x_{1}},...,{y_{k}}_{x_{k}}, x_j,y_q\mid z)}\\
&\le& \sum_{j=k+1}^{n}{P(x_j,y_q\mid z)}\\
&&\text{here, the equal sign holds when } P({y_i}_{x_i}\mid z)=1 \\
&&\text{ for } \forall i \in [1,k].\\
&=& P({x_{q}, y_{q}}\mid z) + \sum_{\substack{k+1\le j \le n, \\ j \ne q}}P({x_{j}, y_{q}}\mid z)
\end{eqnarray*}
To summarize
\begin{eqnarray*}
&&P({y_{1}}_{x_{1}},...,{y_{k}}_{x_{k}}, y_q\mid z) \\
&\le& P({x_{q}, y_{q}}\mid z) + \sum_{\substack{k+1 \le j \le n \\ j \ne q}}P({x_{j}, y_{q}}\mid z)
\end{eqnarray*}
and the equality of the second upper bound holds when $P({y_i}_{x_i}\mid z)=1 \text{ for } \forall i \in [1,k]$.

For the remaining upper bounds, $\forall j \in [1,k]$:
\begin{eqnarray*}
&&P({y_{1}}_{x_{1}},...,{y_{k}}_{x_{k}}, y_q\mid z)\\
&=& \sum_{i=1}^{n}{P({y_{1}}_{x_{1}},...,{y_{k}}_{x_{k}}, x_i,y_q\mid z)} + P({y_j}_{x_j}\mid z) \\
&&- P({y_j}_{x_j}\mid z)\\
&=& \sum_{i=1}^{n}{P({y_{1}}_{x_{1}},...,{y_{k}}_{x_{k}}, x_i,y_q\mid z)} + P({y_j}_{x_j}\mid z) \\
&&- \sum_{\substack{\{i_1,...,i_{j-1},i_{j+1},...,i_{k+2}\} \\ \in \{1,...,n\}^{k+1}}}{P({y_{i_1}}_{x_1},...,{y_{i_{j-1}}}_{x_{j-1}},}\\
&&{{y_{j}}_{x_{j}},{y_{i_{j+1}}}_{x_{j+1}},...,{y_{i_{k}}}_{x_{k}},x_{i_{k+1}},y_{i_{k+2}}\mid z)}
\end{eqnarray*}
Since $q \ne j$ for $1\le j \le k$, 
\begin{eqnarray*}
&&P({y_{1}}_{x_{1}},...,{y_{k}}_{x_{k}}, y_q\mid z)\\
&\le& P({y_j}_{x_j}\mid z) \\
&&- \sum_{\substack{\{i_1,...,i_{j-1},i_{j+1},...,i_{k}\} \\ \in \{1,...,n\}^{k-1}}}{P({y_{i_1}}_{x_1},...,{y_{i_{j-1}}}_{x_{j-1}},{y_{j}}_{x_{j}},}\\
&&{{y_{i_{j+1}}}_{x_{j+1}},...,{y_{i_{k}}}_{x_{k}},x_{j},y_{j}\mid z)}\\
&&\text{here, the equal sign holds when } P({x_j}\mid z)=0,\\
&&P({y_j}\mid z)=0 \text{ and } P({x_{t_1}, y_{t_2}}\mid z)=0\\ &&\text{ for } \forall j \in [k+1, n], \forall t_1, t_2 \in [1, k] \text{ and } t_1\ne t_2.\\
&=& P({y_j}_{x_j}\mid z) - P({y_{j}}_{x_{j}}, x_{j}, y_{j}\mid z)\\
&=& P({y_{j}}_{x_{j}}\mid z) - P({x_{j}, y_{j}}\mid z)
\end{eqnarray*}
The equality of the third upper bound holds when $P({x_j}\mid z)=0,  P({y_j}\mid z)=0 \text{ and } P({x_{t_1}, y_{t_2}}\mid z)=0 \text{ for } \forall j \in [k+1, n], t_1, t_2 \in [1, k] \text{ and } t_1\ne t_2$.

By substituting 
\begin{align*}
{\max \left \{
\begin{array}{cc}
0, \\
\\
\displaystyle \sum_{j=1}^{k}\Big[P({y_{j}}_{x_{j}}\mid z) +P({x_{j}}\mid z)-\\
-P({x_{j}, y_{j}}\mid z)\Big]
+\displaystyle \sum_{\substack{k+1 \le j \le n \\ j \ne q}}{P({x_{j}, y_{q}}\mid z)} +\\
+ P({x_{q}, y_{q}}\mid z) - k,\\
\\
\text{If } q \in \{1,...,k\}:\\
\displaystyle \sum_{\substack{1 \le j \le k \\ j \ne q}}\Big[P({y_{j}}_{x_{j}}\mid z)+P({x_{j}}\mid z)-\\
-P({x_{j}, y_{j}}\mid z)\Big]+ P({x_{q}, y_{q}}\mid z) - (k-1) 
\end{array}
\right \}\nonumber
} \\ \le  P({y_{1}}_{x_{1}},...,{y_{k}}_{x_{k}},y_q\mid z)
\end{align*}
\begin{align*}
{\min \left \{
\begin{array}{cc}
\text{If } q \in \{1,...,k\}:\\
P({y_{q}}_{x_{q}}\mid z), \\
\\
P({x_{q}, y_{q}}\mid z) +\displaystyle \sum_{\substack{k+1 \le j \le n \\ j \ne q}}P({x_{j}, y_{q}}\mid z),\\
\\
\text{For }j \in \{1,...,k\},\text{ and } j \ne q\\
P({y_{j}}_{x_{j}}\mid z) - P({x_{j}, y_{j}}\mid z),\\
\end{array} 
\right \}\nonumber
} \\ \ge  P({y_{1}}_{x_{1}},...,{y_{k}}_{x_{k}},y_q\mid z)
\end{align*}
into equation~(\ref{nnkyq1}), and applying law of total probability, we have
\begin{eqnarray*}
    &&\sum_{z}{P({y_{1}}_{x_{1}},...,{y_{k}}_{x_{k}},y_q\mid z)} \times P(z)\\
    &\ge& \sum_{z}{0} \times P(z) \\
    &=& 0 \\
    &&\sum_{z}{P({y_{1}}_{x_{1}},...,{y_{k}}_{x_{k}},x_p\mid z)} \times P(z)\\
    &\ge& \sum_{z}\Bigg\{{\displaystyle \sum_{j=1}^{k}\Big[P({y_{j}}_{x_{j}}\mid z)+P({x_{j}}\mid z)}\\
    &&{-P({x_{j}, y_{j}}\mid z)\Big]+\displaystyle \sum_{\substack{k+1 \le j \le n \\ j \ne q}}{P({x_{j}, y_{q}}\mid z)}}\\
    &&{+ P({x_{q}, y_{q}}\mid z) - k}\Bigg\}\times P(z)\\
    &=& \displaystyle \sum_{j=1}^{k}\Big[P({y_{j}}_{x_{j}})+P({x_{j}})-P({x_{j}, y_{j}})\Big]\\
    &&+\displaystyle \sum_{\substack{k+1 \le j \le n \\ j \ne q}}{P({x_{j}, y_{q}})} + P({x_{q}, y_{q}}) - k\\
    &&\sum_{z}{P({y_{1}}_{x_{1}},...,{y_{k}}_{x_{k}},x_p\mid z)} \times P(z)\\
    &\ge& \sum_{z}\Bigg\{{\displaystyle \sum_{\substack{1 \le j \le k \\ j \ne q}}\Big[P({y_{j}}_{x_{j}}\mid z)+P({x_{j}}\mid z)}\\
    &&{-P({x_{j}, y_{j}}\mid z)\Big]+P({x_{q}, y_{q}}\mid z) - (k-1)}\Bigg\}\\
    &&\times P(z)\\
    &=& \displaystyle \sum_{\substack{1 \le j \le k \\ j \ne q}}\Big[P({y_{j}}_{x_{j}})+P({x_{j}})-P({x_{j}, y_{j}})\Big]\\
    &&+ P({x_{q}, y_{q}}) - (k-1)
\end{eqnarray*}
\begin{eqnarray*}
    &&\sum_{z}{P({y_{1}}_{x_{1}},...,{y_{k}}_{x_{k}},y_q\mid z)} \times P(z)\\
    &\le& \sum_{z}{P({y_{q}}_{x_{q}}\mid z)} \times P(z) \\
    &=&  P({y_{q}}_{x_{q}})\\
    &&\sum_{z}{P({y_{1}}_{x_{1}},...,{y_{k}}_{x_{k}},x_p\mid z)} \times P(z)\\
    &\le& \sum_{z}\left[P({x_{q}, y_{q}}\mid z) + \displaystyle \sum_{\substack{k+1 \le j \le n \\ j \ne q}}P({x_{j}, y_{q}}\mid z)\right]\\
    &&\times P(z)\\
    &=& P({x_{q}, y_{q}}) + \displaystyle \sum_{\substack{k+1 \le j \le n \\ j \ne q}}P({x_{j}, y_{q}})\\
    &&\sum_{z}{P({y_{1}}_{x_{1}},...,{y_{k}}_{x_{k}},x_p\mid z)} \times P(z)\\
    &\le& \sum_{z}\left[P({y_{j}}_{x_{j}}\mid z) - P({x_{j}, y_{j}}\mid z)\right]\times P(z)\\
    &=& P({y_{j}}_{x_{j}}) - P({x_{j}, y_{j}})
\end{eqnarray*}
Thus, we can guarantee that the proposed lower and upper bounds are no looser than those of \citet{shu2026identificationprobabilitiescausationrecursive}. Following the notation in \citet{mueller2021causeseffectslearningindividual}, $P({y_j}_{x_j}\mid z)$ denotes experimental data within the subpopulation characterized by $z$.
\end{proof}

\subsubsection{Proof of Theorem~\ref{nnk+x_p+y_q1}}
\begin{proof}
\begin{eqnarray}
    PN(k,p,q) &=& P({y_{1}}_{x_{1}},...,{y_{k}}_{x_{k}},x_p, y_q)\nonumber\\
    &=& \sum_{z}{P({y_{1}}_{x_{1}},...,{y_{k}}_{x_{k}},x_p, y_q\mid z)}\nonumber \\
    &&\times P(z) \label{nnkxpyq1}
\end{eqnarray}

First, we need to prove:
\begin{align*}
{\max \left \{
\begin{array}{cc}
0, \\
\displaystyle \sum_{j=1}^{k}\Big[P({y_{j}}_{x_{j}}\mid z)+P({x_{j}}\mid z)-\\-P({x_{j}, y_{j}}\mid z)\Big] + P({x_{p}, y_{q}}\mid z) - k\\
\end{array}
\right \}\nonumber
} \\ \le P({y_{1}}_{x_{1}},...,{y_{k}}_{x_{k}},x_p, y_q\mid z)
\end{align*}
\begin{align*}
{\min \left \{
\begin{array}{cc}
P({x_{p}},{y_{q}}\mid z), \\
\\
\text{For }j \in \{1,...,k\}:\\
P({y_{j}}_{x_{j}}\mid z) - P({x_{j}, y_{j}}\mid z) \\
\end{array} 
\right \}\nonumber
} \\ \ge P({y_{1}}_{x_{1}},...,{y_{k}}_{x_{k}},x_p, y_q\mid z)
\end{align*}
By the Fréchet Inequalities, for any $z$ with $P(z)>0$, we can obtain the first lower bound and the first upper bound, 
\begin{eqnarray*}
P({y_{1}}_{x_{1}},...,{y_{k}}_{x_{k}},{x_{p}},{y_{q}}\mid z)&\ge& 0\\
P({y_{1}}_{x_{1}},...,{y_{k}}_{x_{k}},{x_{p}},{y_{q}}\mid z)&\le& {P({x_{p}},{y_{q}}\mid z)}.
\end{eqnarray*}
The equality of the first lower bound holds when 
$\exists j\in [1,k], \text{that } P({y_j}_{x_j}\mid z)=0$ or $P(x_p\mid z) = 0$ or $P(y_q\mid z) = 0$, $p\ne j$.

The equality of the first upper bound holds when $P({y_{j}}_{x_{j}}\mid z) = 1$ for $\forall j \in [1,k]$.

For the second lower bound
\begin{eqnarray*}
&&P({y_{1}}_{x_{1}},...,{y_{k}}_{x_{k}}, x_p, y_q\mid z)\\
&=& P({y_{1}}_{x_{1}},...,{y_{k}}_{x_{k}}, x_p, y_q\mid z) + P({y_{1}}_{x_{1}},...,{y_{k}}_{x_{k}}\mid z) \\
&&- P({y_{1}}_{x_{1}},...,{y_{k}}_{x_{k}}\mid z)\\
&=& {P({y_{1}}_{x_{1}},...,{y_{k}}_{x_{k}}\mid z)} + P({y_{1}}_{x_{1}},...,{y_{k}}_{x_{k}}, x_p, y_q\mid z) \\
&&-\sum_{j=1}^{n}\sum_{l=1}^{n}{P({y_{1}}_{x_{1}},...,{y_{k}}_{x_{k}}, x_j, y_l\mid z)}\\
&=& {P({y_{1}}_{x_{1}},...,{y_{k}}_{x_{k}}\mid z)} + P({y_{1}}_{x_{1}},...,{y_{k}}_{x_{k}}, x_p, y_q\mid z) \\
&&+ \sum_{j=1}^{n}{P(x_j\mid z)} - 1\\
&&- \sum_{j=1}^{k}{P({y_{1}}_{x_{1}},...,{y_{k}}_{x_{k}}, x_j, y_j\mid z)} \\
&&- \sum_{j=k+1}^{n}\sum_{l=1}^{n}{P({y_{1}}_{x_{1}},...,{y_{k}}_{x_{k}}, x_j, y_l\mid z)} \\
&=& {P({y_{1}}_{x_{1}},...,{y_{k}}_{x_{k}}\mid z)} - 1 \\
&&+ \sum_{j=1}^{k}{P(x_j\mid z)} - \sum_{j=1}^{k}{P({y_{1}}_{x_{1}},...,{y_{k}}_{x_{k}}, x_j, y_j\mid z)} \\ 
&&+ \sum_{j=k+1}^{n}{P(x_j\mid z)} \\
&&- \sum_{j=k+1}^{n}\sum_{l=1}^{n}{P({y_{1}}_{x_{1}},...,{y_{k}}_{x_{k}}, x_j, y_l\mid z)} \\
&&+ P({y_{1}}_{x_{1}},...,{y_{k}}_{x_{k}}, x_p, y_q\mid z) \\
&=& {P({y_{1}}_{x_{1}},...,{y_{k}}_{x_{k}}\mid z)} - 1\\
&&+ \sum_{j=1}^{k}{P(x_j\mid z)} - \sum_{j=1}^{k}{P({y_{1}}_{x_{1}},...,{y_{k}}_{x_{k}}, x_j, y_j\mid z)} \\ 
&&+ \sum_{j=k+1}^{n}{P(x_j\mid z)}\\ 
&&- \sum_{\substack{k+1\le j \le n, \\ j \ne p}}\sum_{l=1}^{n}{P({y_{1}}_{x_{1}},...,{y_{k}}_{x_{k}}, x_j, y_l\mid z)} \\
&&- \sum_{\substack{1\le l\le n, \\ l\ne p}}{P({y_{1}}_{x_{1}},...,{y_{k}}_{x_{k}}, x_p, y_l\mid z)}
\end{eqnarray*}
\begin{eqnarray*}
&\ge& \sum_{j=1}^{k}{P({y_{j}}_{x_{j}}\mid z)} - (k-1) - 1 + \sum_{j=1}^{k}{P(x_j\mid z)} \\
&&- \sum_{j=1}^{k}{P(x_j, y_j\mid z)} + \sum_{j=k+1}^{n}{P(x_j\mid z)} \\
&&- \sum_{\substack{k+1\le j \le n, \\ j \ne p}}{P(x_j\mid z)} - \sum_{\substack{1\le l \le n, \\ l\ne p}}{P(x_p, y_l\mid z)}\\
&&\text{here, the equal sign holds when }
\exists j \in [1,n], \\
&&\text{ and } P({y_i}_{x_i}\mid z)=1 \text{ for } \forall i \in [1,k], i\ne j.\\
&=& \sum_{j=1}^{k}{P({y_{j}}_{x_{j}}\mid z)} - k + \sum_{j=1}^{k}{P(x_j\mid z)} \\
&&- \sum_{j=1}^{k}{P(x_j, y_j\mid z)} + P(x_p\mid z) \\
&&- \sum_{\substack{1\le l\le n, \\ l\ne p}}{P(x_p, y_l\mid z)}\\
&=& \sum_{j=1}^{k}\left[P({y_{j}}_{x_{j}}\mid z)+P({x_{j}}\mid z)-P({x_{j}, y_{j}}\mid z)\right] \\
&&+ P({x_{p}, y_{q}}\mid z) - k
\end{eqnarray*}
The equality of the second lower bound holds when $\exists j \in [1,n], P({y_i}_{x_i}\mid z)=1 \text{ for } \forall i \in [1,k], i\ne j$.

For the remaining upper bounds, $\forall j \in [1,k]$:
\begin{eqnarray*}
&&P({y_{1}}_{x_{1}},...,{y_{k}}_{x_{k}}, x_p, y_q\mid z)\\
&=& P({y_{1}}_{x_{1}},...,{y_{k}}_{x_{k}}, x_p, y_q\mid z) + P({y_j}_{x_j}\mid z) \\
&&- P({y_j}_{x_j}\mid z)\\
&=& P({y_{1}}_{x_{1}},...,{y_{k}}_{x_{k}}, x_p, y_q\mid z) + P({y_j}_{x_j}\mid z) \\
&&- \sum_{\substack{\{i_1,...,i_{j-1},i_{j+1},...,i_{k+2}\} \\ \in \{1,...,n\}^{k+1}}}{P({y_{i_1}}_{x_1},...,{y_{i_{j-1}}}_{x_{j-1}},}\\
&&{{y_{j}}_{x_{j}},{y_{i_{j+1}}}_{x_{j+1}},...,{y_{i_{k}}}_{x_{k}},x_{i_{k+1}},y_{i_{k+2}}\mid z)}
\end{eqnarray*}
Since $p \ne j$ for $1\le j \le k$, 
\begin{eqnarray*}
&&P({y_{1}}_{x_{1}},...,{y_{k}}_{x_{k}}, x_p, y_q\mid z)\\
&\le& P({y_j}_{x_j}\mid z) - \sum_{\substack{\{i_1,...,i_{j-1},i_{j+1},...,i_{k}\} \\ \in \{1,...,n\}^{k-1}}}{P({y_{i_1}}_{x_1},...,}\\
&&{{y_{i_{j-1}}}_{x_{j-1}},{y_{j}}_{x_{j}},{y_{i_{j+1}}}_{x_{j+1}},...,{y_{i_{k}}}_{x_{k}},x_{j},y_{j}\mid z)}\\
&&\text{here, the equal sign holds when } P({x_j}\mid z)=0,\\
&&P({y_j\mid z})=0 \text{ and } P({x_{t_1}, y_{t_2}}\mid z)=0\\ &&\text{ for } \forall j \in [k+1, n], \forall t_1, t_2 \in [1, k] \text{ and } t_1\ne t_2.\\
&=& P({y_j}_{x_j}\mid z) - P({y_{j}}_{x_{j}}, x_{j}, y_{j}\mid z)\\
&=& P({y_{j}}_{x_{j}}\mid z) - P({x_{j}, y_{j}}\mid z)
\end{eqnarray*}
The equality of the second upper bound holds when $P({x_j}\mid z)=0,  P({y_j}\mid z)=0 \text{ and } P({x_{t_1}, y_{t_2}}\mid z)=0 \text{ for } \forall j \in [k+1, n], t_1, t_2 \in [1, k] \text{ and } t_1\ne t_2$.

By substituting 
\begin{align*}
{\max \left \{
\begin{array}{cc}
0, \\
\displaystyle \sum_{j=1}^{k}\Big[P({y_{j}}_{x_{j}}\mid z)+P({x_{j}}\mid z)-\\-P({x_{j}, y_{j}}\mid z)\Big] + P({x_{p}, y_{q}}\mid z) - k\\
\end{array}
\right \}\nonumber
} \\ \le P({y_{1}}_{x_{1}},...,{y_{k}}_{x_{k}},x_p, y_q\mid z)
\end{align*}
\begin{align*}
{\min \left \{
\begin{array}{cc}
P({x_{p}},{y_{q}}\mid z), \\
\\
\text{For }j \in \{1,...,k\}:\\
P({y_{j}}_{x_{j}}\mid z) - P({x_{j}, y_{j}}\mid z) \\
\end{array} 
\right \}\nonumber
} \\ \ge P({y_{1}}_{x_{1}},...,{y_{k}}_{x_{k}},x_p, y_q\mid z)
\end{align*}
into equation~(\ref{nnkxpyq1}), and applying law of total probability, we have
\begin{eqnarray*}
    &&\sum_{z}{P({y_{1}}_{x_{1}},...,{y_{k}}_{x_{k}},x_p, y_q\mid z)} \times P(z)\\
    &\ge& \sum_{z}{0}\times P(z) \\
    &=& 0 \\
    &&\sum_{z}{P({y_{1}}_{x_{1}},...,{y_{k}}_{x_{k}},x_p, y_q\mid z))} \times P(z)\\
    &\ge& \sum_{z}{\Bigg\{\displaystyle \sum_{j=1}^{k}\Big[P({y_{j}}_{x_{j}}\mid z)+P({x_{j}}\mid z)}\\
    &&{-P({x_{j}, y_{j}}\mid z)\Big] + P({x_{p}, y_{q}}\mid z) - k\Bigg\}}\times P(z) \\
    &=& \displaystyle \sum_{j=1}^{k}\Big[P({y_{j}}_{x_{j}})+P({x_{j}})-P({x_{j}, y_{j}})\Big] \\
    &&+ P({x_{p}, y_{q}}) - k
\end{eqnarray*}

\begin{eqnarray*}
    &&\sum_{z}{P({y_{1}}_{x_{1}},...,{y_{k}}_{x_{k}},x_p, y_q\mid z)} \times P(z)\\ 
    &\le& P({x_{p}},{y_{q}}\mid z) \times P(z) \\
    &=& P({x_{p}},{y_{q}}) \\
    &&\sum_{z}{P({y_{1}}_{x_{1}},...,{y_{k}}_{x_{k}},x_p, y_q\mid z)} \times P(z) \\
    &\le& \left[P({y_{j}}_{x_{j}}\mid z) - P({x_{j}, y_{j}}\mid z)\right]\times P(z) \\
    &=& P({y_{j}}_{x_{j}}) - P({x_{j}, y_{j}}))\\
\end{eqnarray*}
Thus, we can guarantee that the proposed lower and upper bounds are no looser than those of \citet{shu2026identificationprobabilitiescausationrecursive}. Following the notation in \citet{mueller2021causeseffectslearningindividual}, $P({y_j}_{x_j}\mid z)$ denotes experimental data within the subpopulation characterized by $z$.
\end{proof}


\subsubsection{Proof of Theorem~\ref{nnk3}}
\begin{proof}
\begin{eqnarray*}
&&PNS(k) \\
&=& P({y_{1}}_{x_{1}},...,{y_{k}}_{x_{k}})\\
&=& \sum_{\substack{{z_1},...,{z_k} \\ \in \{1,...,n\}^{k}}} P({y_{1}}_{x_{1}},...,{y_{k}}_{x_{k}},{z_1}_{x_1},...,{z_k}_{x_k})\\
&=& \sum_{\substack{{z_1},...,{z_k} \\ \in \{1,...,n\}^{k}}} P({y_{1}}_{x_{1}},...,{y_{k}}_{x_{k}}\mid {z_1}_{x_1},...,{z_k}_{x_k})\\&&\times P({z_1}_{x_1},...,{z_k}_{x_k})\\
&\le& \sum_{\substack{{z_1},...,{z_k} \\ \in \{1,...,n\}^{k}}} \min\Big\{{P({y_{1}}_{x_{1}}\mid {z_1}_{x_1},...,{z_k}_{x_k}),...,}\\&&{P({y_{k}}_{x_{k}}\mid {z_1}_{x_1},...,{z_k}_{x_k})}\Big\}\\&&\times \min\Big\{P({z_1}_{x_1}),...,P({z_k}_{x_k})\Big\}\\
&&\text{the equality sign holds when for $\forall j\in[1,k]$,}\\
&&P({y_j}_{x_j}\mid {z_1}_{x_1},\dots,{z_k}_{x_k}) = 0 \text{ or } P({z_j}_{x_j}) = 0\\
&=& \sum_{\substack{{z_1},...,{z_k} \\ \in \{1,...,n\}^{k}}} \min\Big\{{P({y_{1}}_{x_{1}}\mid {z_1}_{x_1}),...,}\\&&{P({y_{k}}_{x_{k}}\mid {z_k}_{x_k})}\Big\}\times \min\Big\{P({z_1}_{x_1}),...,P({z_k}_{x_k})\Big\}\\
&&\text{it follows from the cross-world conditional}\\ &&\text{independences that derive from graph:}\\
&&\text{$Y_{x_j} \perp\!\!\!\perp \{Z_{x_i}:i\neq j\}\mid Z_{x_j}, \qquad j=1,\ldots,k.$}\\
&=& \sum_{\substack{{z_1},...,{z_k} \\ \in \{1,...,n\}^{k}}} \min\Big\{{P({y_{1}}}\mid {z_1}_{x_1}, {x_{1}),...,}\\&&{P({y_{k}}}\mid {z_k}_{x_k}, {x_{k})}\Big\}\times \min\Big\{P({z_1}_{x_1}),...,P({z_k}_{x_k})\Big\}\\
&&\text{since for
$\forall x_j,\; Y_{x_j} \perp\!\!\!\perp {X} \mid Z_{x_j}$.}\\
&=& \sum_{\substack{{z_1},...,{z_k} \\ \in \{1,...,n\}^{k}}} \min\Big\{{P({y_{1}}}\mid {z_1}, {x_{1}),...,}\\&&{P({y_{k}}}\mid {z_k}, {x_{k})}\Big\}\times \min\Big\{P({z_1}_{x_1}),...,P({z_k}_{x_k})\Big\}\\
\end{eqnarray*}

Hence, we can say for $j$ $\in \{1,...,k\}$, 
\begin{eqnarray*}
PNS(k)&&\nonumber\\
\le \sum_{\substack{{z_1},...,{z_k} \\ \in \{1,...,n\}^{k}}}&&
{\min_j\{P({y_j}\mid {x_j},{z_j})\}}\\
&&\times \displaystyle{\min_j \{P({z_j}_{x_j})\} } 
\end{eqnarray*}
\end{proof}

\subsubsection{Proof of Theorem~\ref{nnk+x_p3}}
\begin{proof}
\begin{eqnarray*}
&&PSub(k,p) \\
&=& P({y_{1}}_{x_{1}},...,{y_{k}}_{x_{k}}, x_p)\\
&=& \sum_{\substack{{z_1},...,{z_k} \\ \in \{1,...,n\}^{k}}} P({y_{1}}_{x_{1}},...,{y_{k}}_{x_{k}},x_p,{z_1}_{x_1},...,{z_k}_{x_k})\\
&=& \sum_{\substack{{z_1},...,{z_k} \\ \in \{1,...,n\}^{k}}} P({y_{1}}_{x_{1}},...,{y_{k}}_{x_{k}},x_p \mid {z_1}_{x_1},...,{z_k}_{x_k})\\
&&\times P({z_1}_{x_1},...,{z_k}_{x_k})\\
&=& \sum_{\substack{{z_1},...,{z_k} \\ \in \{1,...,n\}^{k}}} P({y_{1}}_{x_{1}},...,{y_{k}}_{x_{k}}\mid {z_1}_{x_1},...,{z_k}_{x_k})\\
&&\times P(x_p \mid {z_1}_{x_1},...,{z_k}_{x_k}) \times P({z_1}_{x_1},...,{z_k}_{x_k})\\
&&\text{since $\left(Y_{x_1},\ldots,Y_{x_k}\right) \perp\!\!\!\perp X \mid \left(Z_{x_1},\ldots,Z_{x_k}\right)$}\\
&=& \sum_{\substack{{z_1},...,{z_k} \\ \in \{1,...,n\}^{k}}} P({y_{1}}_{x_{1}},...,{y_{k}}_{x_{k}}\mid {z_1}_{x_1},...,{z_k}_{x_k})\\
&&\times P({z_1}_{x_1},...,{z_k}_{x_k},x_p)\\
&\le& \sum_{\substack{{z_1},...,{z_k} \\ \in \{1,...,n\}^{k}}} \min\Big\{{P({y_{1}}_{x_{1}}\mid {z_1}_{x_1},...,{z_k}_{x_k}),...,}\\&&{P({y_{k}}_{x_{k}}\mid {z_1}_{x_1},...,{z_k}_{x_k})}\Big\}\\&&\times \min\Big\{P({z_1}_{x_1}),...,P({z_k}_{x_k}),P(x_p)\Big\}\\
&&\text{the equality sign holds when for $\forall j\in[1,k]$,}\\
&&P({y_j}_{x_j}\mid {z_1}_{x_1},\dots,{z_k}_{x_k}) = 0\\
&&\text{ or } P({z_j}_{x_j}) = 0 \text{ or } P(x_p) = 0\\
&=& \sum_{\substack{{z_1},...,{z_k} \\ \in \{1,...,n\}^{k}}} \min\Big\{{P({y_{1}}_{x_{1}}\mid {z_1}_{x_1}),...,}\\&&{P({y_{k}}_{x_{k}}\mid {z_k}_{x_k})}\Big\}\\
&&\times \min\Big\{P({z_1}_{x_1}),...,P({z_k}_{x_k}),P(x_p)\Big\}\\
&&\text{as for $j\in[1,k]$: $Y_{x_j} \perp\!\!\!\perp \{Z_{x_i}:i\neq j\}\mid Z_{x_j}$}\\
&=& \sum_{\substack{{z_1},...,{z_k} \\ \in \{1,...,n\}^{k}}} \min\Big\{{P({y_{1}}}\mid {z_1}_{x_1}, {x_{1}),...,}\\&&{P({y_{k}}}\mid {z_k}_{x_k}, {x_{k})}\Big\}\\
&&\times \min\Big\{P({z_1}_{x_1}),...,P({z_k}_{x_k}),P(x_p)\Big\}\\
&&\text{since for
$\forall x_j,\; Y_{x_j} \perp\!\!\!\perp {X} \mid Z_{x_j}$.}\\
&=& \sum_{\substack{{z_1},...,{z_k} \\ \in \{1,...,n\}^{k}}} \min\Big\{{P({y_{1}}}\mid {z_1}, {x_{1}),...,}\\&&{P({y_{k}}}\mid {z_k}, {x_{k})}\Big\}\\
&&\times \min\Big\{P({z_1}_{x_1}),...,P({z_k}_{x_k}),P(x_p)\Big\}
\end{eqnarray*}
Hence, we can say for $j$ $\in \{1,...,k\}$, 
\begin{eqnarray*}
PSub(k,p)&&\\ 
\le \displaystyle \sum_{\substack{{z_1},...,{z_k} \\ \in \{1,...,n\}^{k}}}
&&{\min_j\{P({y_j}\mid {x_j},{z_j})\}} \\
&&\times{\displaystyle \min_j \{P({z_j}_{x_j}), P(x_p)\} }
\end{eqnarray*}
\end{proof}

\subsubsection{Proof of Theorem~\ref{nnk+y_q3}}
\begin{proof}
\begin{eqnarray*}
&&PRep(k,q) \\
&=& P({y_{1}}_{x_{1}},...,{y_{k}}_{x_{k}}, y_q)\\
&\le& P({y_{1}}_{x_{1}},...,{y_{k}}_{x_{k}})\\
&&\text{the equality sign holds when } P(y_q) = 0\\
&\le& \displaystyle \sum_{\substack{{z_1},...,{z_k} \\ \in \{1,...,n\}^{k}}}{\min_j\{P({y_j}\mid {x_j},{z_j})\}} \times \displaystyle{\min_j \{P({z_j}_{x_j})\} }
\end{eqnarray*}
Hence, we can say for $j$ $\in \{1,...,k\}$, 
\begin{eqnarray*}
PRep(k,q)&& \\
\le \displaystyle \sum_{\substack{{z_1},...,{z_k} \\ \in \{1,...,n\}^{k}}}
&&{\min_j\{P({y_j}\mid {x_j},{z_j})\}} \times \displaystyle{\min_j \{P({z_j}_{x_j})\} }\\
\end{eqnarray*}
\end{proof}

\subsubsection{Proof of Theorem~\ref{nnk+x_p+y_q3}}
\begin{proof}
\begin{eqnarray*}
&&PN(k,p,q)\\ 
&=& P({y_{1}}_{x_{1}},...,{y_{k}}_{x_{k}}, x_p, y_q)\\
&=& P({y_{1}}_{x_{1}},...,{y_{k}}_{x_{k}}, {y_{q}}_{x_{p}}, x_p)\\
&=& \sum_{\substack{{z_1},...,{z_k},{z_p} \\ \in \{1,...,n\}^{k+1}}} {P({y_{1}}_{x_{1}},...,{y_{k}}_{x_{k}},{y_{q}}_{x_{p}},x_p,{z_1}_{x_1},...,}\\
&&{{z_k}_{x_k},{z_p}_{x_p})}\\
&=& \sum_{\substack{{z_1},...,{z_k},{z_p} \\ \in \{1,...,n\}^{k+1}}} P({y_{1}}_{x_{1}},...,{y_{k}}_{x_{k}},{y_{q}}_{x_{p}},x_p \mid {z_1}_{x_1},...,\\
&&{z_k}_{x_k},{z_p}_{x_p})\\
&&\times P({z_1}_{x_1},...,{z_k}_{x_k},{z_p}_{x_p})\\
&=& \sum_{\substack{{z_1},...,{z_k},{z_p} \\ \in \{1,...,n\}^{k+1}}} P({y_{1}}_{x_{1}},...,{y_{k}}_{x_{k}},{y_{q}}_{x_{p}}\mid {z_1}_{x_1},...,{z_k}_{x_k},\\
&&{z_p}_{x_p})\\
&&\times P(x_p \mid {z_1}_{x_1},...,{z_k}_{x_k},{z_p}_{x_p}) \\
&&\times P({z_1}_{x_1},...,{z_k}_{x_k},{z_p}_{x_p})\\
&&\text{since } \left(Y_{x_1},\ldots,Y_{x_k},Y_{x_p}\right) \perp\!\!\!\perp X \mid (Z_{x_1},\ldots,Z_{x_k},\\
&&Z_{x_p})
\end{eqnarray*}
\begin{eqnarray*}
&=& \sum_{\substack{{z_1},...,{z_k},{z_p} \\ \in \{1,...,n\}^{k+1}}} P({y_{1}}_{x_{1}},...,{y_{k}}_{x_{k}},{y_{q}}_{x_{p}}\mid {z_1}_{x_1},...,{z_k}_{x_k},\\
&&{z_p}_{x_p})\\
&&\times P({z_1}_{x_1},...,{z_k}_{x_k},{z_p}_{x_p},x_p)\\
&\le& \sum_{\substack{{{z_1},...,{z_k},{z_p}} \\ {\in \{1,...,n\}^{k+1}}}} \min\Big\{{P({y_{1}}_{x_{1}}\mid {z_1}_{x_1},...,{z_k}_{x_k},{z_p}_{x_p}),}\\
&&{...,P({y_{k}}_{x_{k}}\mid {z_1}_{x_1},...,{z_k}_{x_k},{z_p}_{x_p})}, \\
&&{P({y_{q}}_{x_{p}}\mid {z_1}_{x_1},...,{z_k}_{x_k},{z_p}_{x_p})}\Big\}\\
&&\times \min\Big\{P({z_1}_{x_1}),...,P({z_k}_{x_k}),P({z_p}_{x_p}),P(x_p)\Big\}\\
&&\text{the equality sign holds when for $\forall j\in[1,k]$,}\\
&&P({y_j}_{x_j}\mid {z_1}_{x_1},\dots,{z_k}_{x_k},{z_p}_{x_p}) = 0 \\
&&\text{ or } P({y_q}_{x_p}\mid {z_1}_{x_1},\dots,{z_k}_{x_k},{z_p}_{x_p}) = 0\\
&&\text{ or } P({z_j}_{x_j}) = 0 \text{ or } P({z_p}_{x_p}) = 0 \text{ or } P(x_p) = 0\\
&=& \sum_{\substack{{z_1},...,{z_k},{z_p} \\ \in \{1,...,n\}^{k+1}}} \min\Big\{{P({y_{1}}_{x_{1}}\mid {z_1}_{x_1}),...,}\\&&{P({y_{k}}_{x_{k}}\mid {z_k}_{x_k}),P({y_q}_{x_p}\mid{z_p}_{x_p})}\Big\}\\
&&\times \min\Big\{P({z_1}_{x_1}),...,P({z_k}_{x_k}),P({z_p}_{x_p}),P(x_p)\Big\}\\
&&\text{as for $j\in\{1,\dots,k,p\}$: $Y_{x_j} \perp\!\!\!\perp \{Z_{x_i}:i\neq j\}\mid Z_{x_j}$}\\
&=& \sum_{\substack{{z_1},...,{z_k},{z_p} \\ \in \{1,...,n\}^{k+1}}} \min\Big\{{P({y_{1}}}\mid {z_1}_{x_1}, {x_{1}),...,}\\&&{P({y_{k}}}\mid {z_k}_{x_k}, {x_{k}),P({y_q}\mid{z_p}_{x_p},{x_p})}\Big\}\\
&&\times \min\Big\{P({z_1}_{x_1}),...,P({z_k}_{x_k}),P({z_p}_{x_p}),P(x_p)\Big\}\\
&&\text{since for
$\forall x_j,\; Y_{x_j} \perp\!\!\!\perp {X} \mid Z_{x_j}$.}\\
&=& \sum_{\substack{{z_1},...,{z_k},{z_p} \\ \in \{1,...,n\}^{k+1}}} \min\Big\{{P({y_{1}}}\mid {z_1}, {x_{1}),...,}\\&&{P({y_{k}}}\mid {z_k}, {x_{k}),P({y_q}\mid{z_p},{x_p})}\Big\}\\
&&\times \min\Big\{P({z_1}_{x_1}),...,P({z_k}_{x_k}),P({z_p}_{x_p}),P(x_p)\Big\}\\
\end{eqnarray*}
Hence, we can say for $j$ $\in \{1,...,k\}$, 
\begin{eqnarray*}
PN(k,p,q)&&\\
\le \displaystyle \sum_{\substack{{z_1},...,{z_k},{z_p} \\ \in \{1,...,n\}^{k+1}}}
&&{\min_j\{P({y_j}\mid {x_j},{z_j}),P({y_q}\mid {x_p},{z_p})\}}\\
&&\times {\displaystyle \min_j \{P({z_j}_{x_j}), P({z_p}_{x_p}), P(x_p)\} }\\
\end{eqnarray*}
\end{proof}

\subsubsection{Proof of Theorem~\ref{nnk4}}
\begin{proof}
\begin{eqnarray*}
&&PNS(k) \\
&=& P({y_{1}}_{x_{1}},...,{y_{k}}_{x_{k}})\\
&=& \sum_{\substack{{z_1},...,{z_k} \\ \in \{1,...,n\}^{k}}} P({y_{1}}_{z_{1}},...,{y_{k}}_{z_{k}},{z_1}_{x_1},...,{z_k}_{x_k})\\
&=& \sum_{\substack{{z_1}\ne...\ne{z_k}}} P({y_{1}}_{z_{1}},...,{y_{k}}_{z_{k}})\times P({z_1}_{x_1},...,{z_k}_{x_k})\\
&\le& \sum_{\substack{{z_1}\ne...\ne{z_k}}} \min\{P({y_{1}}_{z_{1}}),...,P({y_{k}}_{z_{k}})\}\\&&\times \min\{P({z_1}_{x_1}),...,P({z_k}_{x_k})\}\\
&&\text{the equality sign holds when for $\forall j\in[1,k]$,}\\
&& P({y_j}_{z_j}) = 0 \text{ or } P({z_j}_{x_j}) = 0\\
&=& \sum_{\substack{{z_1}\ne...\ne{z_k}}} \min\{P({y_{1}}\mid{z_{1}}),...,P({y_{k}}\mid{z_{k}})\}\\&&\times \min\{P({z_1}\mid{x_1}),...,P({z_k}\mid{x_k})\}
\end{eqnarray*}
Hence, we can say for $j \in \{1,...,k\}$,
\begin{eqnarray*}
PNS(k)&& \\
\le \displaystyle \sum_{\substack{{z_1}\neq...\neq{z_k} \\ \in \{1,...,n\}^{k}}}
&&{\min_j\{P({y_j}\mid {z_j})\}} \times \displaystyle{\min_j \{P({z_j}\mid {x_j})\} }
\end{eqnarray*}
\end{proof}

\subsubsection{Proof of Theorem~\ref{nnk+x_p4}}
\begin{proof}
\begin{eqnarray*}
&&PSub(k,p) \\
&=& P({y_{1}}_{x_{1}},...,{y_{k}}_{x_{k}}, x_p)\\
&=& \sum_{\substack{{z_1},...,{z_k} \\ \in \{1,...,n\}^{k}}} P({y_{1}}_{z_{1}},...,{y_{k}}_{z_{k}},x_p,{z_1}_{x_1},...,{z_k}_{x_k})\\
&=& \sum_{\substack{{z_1}\ne...\ne{z_k}}} P({y_{1}}_{z_{1}},...,{y_{k}}_{z_{k}})\times P({z_1}_{x_1},...,{z_k}_{x_k})\\
&&\times P(x_p)\\
&&\text{as ${X} \perp\!\!\!\perp \{Y_{z_j},Z_{x_j}\}$}\\
&\le& \sum_{\substack{{z_1}\ne...\ne{z_k}}} \min\{P({y_{1}}_{z_{1}}),...,P({y_{k}}_{z_{k}})\}\\
&&\times \min\{P({z_1}_{x_1}),...,P({z_k}_{x_k})\}\times P(x_p)\\
&&\text{the equality sign holds when for $\forall j\in[1,k]$,}\\
&& P({y_j}_{z_j}) = 0 \text{ or } P({z_j}_{x_j}) = 0\\
&=& \sum_{\substack{{z_1}\ne...\ne{z_k}}} \min\{P({y_{1}}\mid{z_{1}}),...,P({y_{k}}\mid{z_{k}})\}\\
&&\times \min\{P({z_1}\mid{x_1}),...,P({z_k}\mid{x_k})\}\times P(x_p)\\
\end{eqnarray*}
Hence, we can say for $j \in \{1,...,k\}$,
\begin{eqnarray*}
&&PSub(k,p) \\
&\le& \displaystyle \sum_{\substack{{z_1}\neq...\neq{z_k} \\ \in \{1,...,n\}^{k}}}{\min_j\{P({y_j}\mid {z_j})\}}\\
&&\times \displaystyle{\min_j \{P({z_j}\mid {x_j})\} }\times P(x_p)
\end{eqnarray*}
\end{proof}

\subsubsection{Proof of Theorem~\ref{nnk+y_q4}}
\begin{proof}
\begin{eqnarray*}
&&PRep(k,q)\\
&=& P({y_{1}}_{x_{1}},...,{y_{k}}_{x_{k}}, y_q) \\
&\le& P({y_{1}}_{x_{1}},...,{y_{k}}_{x_{k}})\\
&&\text{the equality sign holds when } P(y_q) = 0\\
&\le& \displaystyle \sum_{\substack{{z_1}\neq...\neq{z_k} \\ \in \{1,...,n\}^{k}}}{\min_j\{P({y_j}\mid {z_j})\}} \times {\min_j \{P({z_j}\mid {x_j})\} }
\end{eqnarray*}
Hence, we can say for $j \in \{1,...,k\}$,
\begin{eqnarray*}
PRep(k,q)&& \\
\le \displaystyle \sum_{\substack{{z_1}\neq...\neq{z_k} \\ \in \{1,...,n\}^{k}}}
&&{\min_j\{P({y_j}\mid {z_j})\}} \times {\min_j \{P({z_j}\mid {x_j})\} }
\end{eqnarray*}
\end{proof}

\subsubsection{Proof of Theorem~\ref{nnk+x_p+y_q4}}
\begin{proof}
\begin{eqnarray*}
&&PN(k,p,q)\\
&=& P({y_{1}}_{x_{1}},...,{y_{k}}_{x_{k}}, x_p, y_q)\\
&=& P({y_{1}}_{x_{1}},...,{y_{k}}_{x_{k}}, {y_q}_{x_p}, x_p)\\
&=& \sum_{\substack{{z_1},...,{z_k},{z_p} \\ \in \{1,...,n\}^{k+1}}} P({y_{1}}_{z_{1}},...,{y_{k}}_{z_{k}},{{y_q}_{z_p}},{z_1}_{x_1},...,{z_k}_{x_k},\\
&&{z_p}_{x_p},x_p)\\
&=& \sum_{\substack{{z_1}\ne...\ne{z_k}\ne{z_p}}} P({y_{1}}_{z_{1}},...,{y_{k}}_{z_{k}},{y_q}_{z_p})\\
&&\times P({z_1}_{x_1},...,{z_k}_{x_k},{z_p}_{x_p})\times P(x_p)\\
&&\text{as ${X} \perp\!\!\!\perp \{Y_{z_j},Z_{x_j}\}$}\\
&\le& \sum_{\substack{{z_1}\ne...\ne{z_k}\ne{z_p}}} \min\{P({y_{1}}_{z_{1}}),...,P({y_{k}}_{z_{k}}),P({y_q}_{z_p})\}\\
&&\times \min\{P({z_1}_{x_1}),...,P({z_k}_{x_k}),P({z_p}_{x_p})\}\times P(x_p)\\
&&\text{the equality sign holds when for $\forall j\in[1,k]$,}\\
&& P({y_j}_{z_j}) = 0 \text{ or } P({y_q}_{z_p}) = 0 \\
&&\text{ or } P({z_j}_{x_j}) = 0 \text{ or } P({z_p}_{x_p}) = 0\\
\end{eqnarray*}
\begin{eqnarray*}
&=& \sum_{\substack{{z_1}\ne...\ne{z_k}\ne{z_p}}} \min\{P({y_{1}}\mid{z_{1}}),...,P({y_{k}}\mid{z_{k}}),\\
&&P({y_q}\mid{z_p})\}\\
&&\times \min\{P({z_1}\mid{x_1}),...,P({z_k}\mid{x_k}),P({z_p}\mid{x_p})\}\\
&&\times P(x_p)
\end{eqnarray*}

Hence, we can say for $j \in \{1,...,k\}$,
\begin{eqnarray*}
PN(k,p,q)&& \\
\le \displaystyle \sum_{\substack{{z_1}\neq...\neq{z_k}\neq{z_p} \\ \in \{1,...,n\}^{k+1}}}
&&{\min_j\{P({y_j}\mid {z_j}),P({y_q}\mid {z_p})\}} \\
&&\times {\displaystyle\min_j \{P({z_j}\mid {x_j}),P({z_p}\mid {x_p})\} } \\
&&\times P(x_p)
\end{eqnarray*}
\end{proof}

\subsection{Examples}
\subsubsection{Example 1:~\ref{ex1}}
From Table~\ref{tb1}, we can collect all data we want.

Since each entry in the table represents the count $n(x,y,z)$, for each fixed value $Z=z$, we have
\[
P(x\mid z)
=
\frac{\sum_y n(x,y,z)}
{\sum_{x,y}n(x,y,z)},
\]
\[
P(y\mid x,z)
=
\frac{n(x,y,z)}
{\sum_y n(x,y,z)},
\]
and
\[
P(y,x\mid z)
=
\frac{n(x,y,z)}
{\sum_{x,y}n(x,y,z)}.
\]
These quantities satisfy
\[
P(y,x\mid z)
=
P(y\mid x,z)P(x\mid z).
\]

For \(z_1\), we have
\begin{align*}
&P(z_1)=\frac{2}{5},\\
&P(x_1\mid z_1)=\frac{1}{6},
\quad
P(x_2\mid z_1)=\frac{1}{3},
\quad
P(x_3\mid z_1)=\frac{1}{2}.
\end{align*}
\begin{align*}
P(y_1\mid x_1,z_1) &= \frac{23}{30},&
P(y_1\mid x_2,z_1) &= \frac{1}{30},\\
P(y_1\mid x_3,z_1) &= \frac{9}{30},&
P(y_2\mid x_1,z_1) &= \frac{6}{30},\\
P(y_2\mid x_2,z_1) &= \frac{28}{30},&
P(y_2\mid x_3,z_1) &= \frac{20}{30},\\
P(y_3\mid x_1,z_1) &= \frac{1}{30},&
P(y_3\mid x_2,z_1) &= \frac{1}{30},\\
P(y_3\mid x_3,z_1) &= \frac{1}{30}.&
\end{align*}

\begin{align*}
P(x_1,y_1\mid z_1) &= \frac{23}{180},&
P(x_2,y_1\mid z_1) &= \frac{2}{180},\\
P(x_3,y_1\mid z_1) &= \frac{27}{180},&
P(x_1,y_2\mid z_1) &= \frac{6}{180},\\
P(x_2,y_2\mid z_1) &= \frac{56}{180},&
P(x_3,y_2\mid z_1) &= \frac{60}{180},\\
P(x_1,y_3\mid z_1) &= \frac{1}{180},&
P(x_2,y_3\mid z_1) &= \frac{2}{180},\\
P(x_3,y_3\mid z_1) &= \frac{3}{180}.&
\end{align*}

For \(z_2\), we have
\begin{align*}
&P(z_2)=\frac{1}{5},\\
&P(x_1\mid z_2)=\frac{1}{2},
\quad
P(x_2\mid z_2)=\frac{1}{6},
\quad
P(x_3\mid z_2)=\frac{1}{3}.
\end{align*}
\begin{align*}
P(y_1\mid x_1,z_2) &= \frac{1}{30},&
P(y_1\mid x_2,z_2) &= \frac{3}{30},\\
P(y_1\mid x_3,z_2) &= \frac{1}{30},&
P(y_2\mid x_1,z_2) &= \frac{25}{30},\\
P(y_2\mid x_2,z_2) &= \frac{25}{30},&
P(y_2\mid x_3,z_2) &= \frac{1}{30},\\
P(y_3\mid x_1,z_2) &= \frac{4}{30},&
P(y_3\mid x_2,z_2) &= \frac{2}{30},\\
P(y_3\mid x_3,z_2) &= \frac{28}{30}.&
\end{align*}

\begin{align*}
P(x_1,y_1\mid z_2) &= \frac{3}{180},&
P(x_2,y_1\mid z_2) &= \frac{3}{180},\\
P(x_3,y_1\mid z_2) &= \frac{2}{180},&
P(x_1,y_2\mid z_2) &= \frac{75}{180},\\
P(x_2,y_2\mid z_2) &= \frac{25}{180},&
P(x_3,y_2\mid z_2) &= \frac{2}{180},\\
P(x_1,y_3\mid z_2) &= \frac{12}{180},&
P(x_2,y_3\mid z_2) &= \frac{2}{180},\\
P(x_3,y_3\mid z_2) &= \frac{56}{180}.&
\end{align*}

For \(z_3\), we have
\begin{align*}
&P(z_3)=\frac{2}{5},\\
&P(x_1\mid z_3)=\frac{2}{3},
\quad
P(x_2\mid z_3)=\frac{1}{6},
\quad
P(x_3\mid z_3)=\frac{1}{6}.
\end{align*}
\begin{align*}
P(y_1\mid x_1,z_3) &= \frac{24}{30},&
P(y_1\mid x_2,z_3) &= \frac{22}{30},\\
P(y_1\mid x_3,z_3) &= \frac{2}{30},&
P(y_2\mid x_1,z_3) &= \frac{3}{30},\\
P(y_2\mid x_2,z_3) &= \frac{1}{30},&
P(y_2\mid x_3,z_3) &= \frac{1}{30},\\
P(y_3\mid x_1,z_3) &= \frac{3}{30},&
P(y_3\mid x_2,z_3) &= \frac{7}{30},\\
P(y_3\mid x_3,z_3) &= \frac{27}{30}.&
\end{align*}

\begin{align*}
P(x_1,y_1\mid z_3) &= \frac{96}{180},&
P(x_2,y_1\mid z_3) &= \frac{22}{180},\\
P(x_3,y_1\mid z_3) &= \frac{2}{180},&
P(x_1,y_2\mid z_3) &= \frac{12}{180},\\
P(x_2,y_2\mid z_3) &= \frac{1}{180},&
P(x_3,y_2\mid z_3) &= \frac{1}{180},\\
P(x_1,y_3\mid z_3) &= \frac{12}{180},&
P(x_2,y_3\mid z_3) &= \frac{7}{180},\\
P(x_3,y_3\mid z_3) &= \frac{27}{180}.&
\end{align*}

After applying the back-door adjustment formula, which gives
\begin{align*}
P({y_j}_{x_j}\mid z)=P(y_j\mid x_j,z).
\end{align*}
Substituting this equality into Theorems~\ref{nnk1} yields a new theorem that requires only observational data, without the need for experimental data. The new equation is listed below:
\begin{align*}
\sum_{z}{\max \left \{
\begin{array}{cc}
0, \\
\\
\displaystyle \sum_{j = 1}^{k}P(y_{j}\mid x_{j}, z) - k + 1, \\
\\
\text{For }i \in  \{1, ..., k\}:\\
\displaystyle \sum_{\substack{1 \le j \le k \\ j \ne i}}
\Big[P({y_{j}}\mid {x_{j}},z)+\\+P({x_{j}\mid z})-P({x_{j}, y_{j}}\mid z)\Big]+\\
+ P({x_{i}, y_{i}}\mid z) - k + 1
\end{array}
\right \}\nonumber
} \times P(z) \\ \le PNS(k)
\end{align*}
\begin{align*}
\sum_{z}{\min \left \{
\begin{array}{cc}
\displaystyle \sum_{j = 1}^{k}P({x_{j}, y_{j}}\mid z) +\\
+ \displaystyle \sum_{j = k+1}^{n}P({x_{j}}\mid z), \\
\\
\text{For }j \in \{1,...,k\}:\\
P({y_{j}}\mid {x_{j}},z),\\
\\
\text{For }m\in\{1,...,k-1\},\\
t_j\in\{1,...,k\}:\\
\displaystyle \frac{1}{m}\Big[\sum_{j = 0}^{m}P({y_{t_j}}\mid {x_{t_j}},z) -\\- P({x_{t_j}, y_{t_j}\mid z})\Big]
\end{array} 
\right \}\nonumber
} \times P(z) \\ \ge PNS(k)
\end{align*}

By applying the new bounds, we can derive

\begin{align*}
\sum_z {\max \left \{
\begin{array}{cc}
0, \\
\displaystyle P(y_{1}\mid x_{1}, z) +P(y_{2}\mid x_{2}, z)\\
+P(y_{3}\mid x_{3}, z)- 2, \\
\Big[P({y_{2}}\mid {x_{2}},z)+P({x_{2}\mid z})\\
-P({x_{2}, y_{2}}\mid z)\Big]+
\Big[P({y_{3}}\mid {x_{3}},z)\\
+P({x_{3}\mid z})-P({x_{3}, y_{3}}\mid z)\Big]\\
+ P({x_{1}, y_{1}}\mid z) - 2,\\
\Big[P({y_{1}}\mid {x_{1}},z)+P({x_{1}\mid z})\\
-P({x_{1}, y_{1}}\mid z)\Big]+
\Big[P({y_{3}}\mid {x_{3}},z)\\
+P({x_{3}\mid z})-P({x_{3}, y_{3}}\mid z)\Big]\\
+ P({x_{2}, y_{2}}\mid z) - 2,\\
\Big[P({y_{1}}\mid {x_{1}},z)+P({x_{1}\mid z})\\
-P({x_{1}, y_{1}}\mid z)\Big]+
\Big[P({y_{2}}\mid {x_{2}},z)\\
+P({x_{2}\mid z})-P({x_{2}, y_{2}}\mid z)\Big]\\
+ P({x_{3}, y_{3}}\mid z) - 2
\end{array}
\right \}\nonumber
} \times P(z)
\\ \le PNS(k)
\end{align*}

\begin{align*}
\sum_{z}{\min \left \{
\begin{array}{cc}
P({x_{1}, y_{1}}\mid z) +P({x_{2}, y_{2}}\mid z)\\
+P({x_{3}, y_{3}}\mid z),\\
P({y_{1}}\mid {x_{1}},z),\\
P({y_{2}}\mid {x_{2}},z),\\
P({y_{3}}\mid {x_{3}},z),\\
\Big[P({y_{1}}\mid {x_{1}},z) -P({x_{1}, y_{1}\mid z})\\
+P({y_{2}}\mid {x_{2}},z) -P({x_{2}, y_{2}\mid z})\Big],\\
\Big[P({y_{1}}\mid {x_{1}},z) -P({x_{1}, y_{1}\mid z})\\
+P({y_{3}}\mid {x_{3}},z) -P({x_{3}, y_{3}\mid z})\Big],\\
\Big[P({y_{2}}\mid {x_{2}},z) -P({x_{2}, y_{2}\mid z})\\
+P({y_{3}}\mid {x_{3}},z) -P({x_{3}, y_{3}\mid z})\Big],\\
\displaystyle \frac{1}{2}\Big[P({y_{1}}\mid {x_{1}},z) -P({x_{1}, y_{1}\mid z})\\
+P({y_{2}}\mid {x_{2}},z)-P({x_{2}, y_{2}\mid z})\\
+P({y_{3}}\mid {x_{3}},z) -P({x_{3}, y_{3}\mid z})\Big]
\end{array} 
\right \}\nonumber
} \times P(z) \\ \ge PNS(k)
\end{align*}

After applying all the required data we obtained before, we can derive

\begin{align*}
0
\leq
PNS(3)
\leq
0.033.
\end{align*}

\subsubsection{Example 2:~\ref{ex2}}
From Table~\ref{tb2}, we can collect all data we want.

Since each cell in the table represents the joint count $n(x,y,z)$, we obtain
\[
P(z\mid x)
=
\frac{\sum_y n(x,y,z)}
{\sum_{z,y}n(x,y,z)}.
\]
That is, for a fixed $X=x$, we first sum over $Y$ within each level $Z=z$ and then divide by the total number of individuals with $X=x$.

Similarly,
\[
P(y\mid z)
=
\frac{\sum_x n(x,y,z)}
{\sum_{x,y}n(x,y,z)}.
\]
That is, for a fixed $Z=z$, we sum over $X$ for each outcome $Y=y$ and then divide by the total number of individuals in the stratum $Z=z$.

\begin{align*}
P(z_1\mid x_1) &= \frac{28}{30},&
P(z_2\mid x_1) &= \frac{1}{30},\\
P(z_3\mid x_1) &= \frac{1}{30},&
P(z_1\mid x_2) &= \frac{19}{30},\\
P(z_2\mid x_2) &= \frac{8}{30},&
P(z_3\mid x_2) &= \frac{3}{30},\\
P(z_1\mid x_3) &= \frac{1}{30},&
P(z_2\mid x_3) &= \frac{7}{30},\\
P(z_3\mid x_3) &= \frac{22}{30}.&
\end{align*}

\begin{align*}
P(y_1\mid z_1) &= \frac{16}{30},&
P(y_2\mid z_1) &= \frac{13}{30},\\
P(y_3\mid z_1) &= \frac{1}{30},&
P(y_1\mid z_2) &= \frac{2}{30},\\
P(y_2\mid z_2) &= \frac{27}{30},&
P(y_3\mid z_2) &= \frac{1}{30},\\
P(y_1\mid z_3) &= \frac{6}{30},&
P(y_2\mid z_3) &= \frac{3}{30},\\
P(y_3\mid z_3) &= \frac{21}{30}.&
\end{align*}

Following the example, we can apply Thm.~\ref{nnk4}:
\begin{align*}
\min \left \{
\begin{array}{cc}
{PNS}_{{UB}},\\
\\
{\min\{P({y_1}\mid {z_1}),P({y_2}\mid {z_2}),P({y_3}\mid {z_3})\}} \times \\\displaystyle{\min\{P({z_1}\mid {x_1}),P({z_2}\mid {x_2}),P({z_3}\mid {x_3})\} }\\
+\\
+{\min\{P({y_1}\mid {z_1}),P({y_2}\mid {z_3}),P({y_3}\mid {z_2})\}} \times \\\displaystyle{\min\{P({z_1}\mid {x_1}),P({z_3}\mid {x_2}),P({z_2}\mid {x_3})\} }\\
+\\
+{\min\{P({y_1}\mid {z_2}),P({y_2}\mid {z_1}),P({y_3}\mid {z_3})\}} \times \\\displaystyle{\min\{P({z_2}\mid {x_1}),P({z_1}\mid {x_2}),P({z_3}\mid {x_3})\} }\\
+\\
+{\min\{P({y_1}\mid {z_2}),P({y_2}\mid {z_3}),P({y_3}\mid {z_1})\}} \times \\\displaystyle{\min\{P({z_2}\mid {x_1}),P({z_3}\mid {x_2}),P({z_1}\mid {x_3})\} }\\
+\\
+{\min\{P({y_1}\mid {z_3}),P({y_2}\mid {z_1}),P({y_3}\mid {z_2})\}} \times \\\displaystyle{\min\{P({z_3}\mid {x_1}),P({z_1}\mid {x_2}),P({z_2}\mid {x_3})\} }\\
+\\
+{\min\{P({y_1}\mid {z_3}),P({y_2}\mid {z_2}),P({y_3}\mid {z_1})\}} \times \\\displaystyle{\min\{P({z_3}\mid {x_1}),P({z_2}\mid {x_2}),P({z_1}\mid {x_3})\} }
\end{array} 
\right \}\nonumber
\\ \ge PNS(3)
\end{align*}

By applying the data, we have:
\begin{align*}
&\min \left\{
\begin{array}{c}
0.507,\\
\displaystyle
\min\left\{\frac{16}{30},\frac{27}{30},\frac{21}{30}\right\}
\times
\min\left\{\frac{28}{30},\frac{8}{30},\frac{22}{30}\right\}\\[2mm]
\displaystyle
+
\min\left\{\frac{16}{30},\frac{3}{30},\frac{1}{30}\right\}
\times
\min\left\{\frac{28}{30},\frac{3}{30},\frac{7}{30}\right\}\\[2mm]
\displaystyle
+
\min\left\{\frac{2}{30},\frac{13}{30},\frac{21}{30}\right\}
\times
\min\left\{\frac{1}{30},\frac{19}{30},\frac{22}{30}\right\}\\[2mm]
\displaystyle
+
\min\left\{\frac{2}{30},\frac{3}{30},\frac{1}{30}\right\}
\times
\min\left\{\frac{1}{30},\frac{3}{30},\frac{1}{30}\right\}\\[2mm]
\displaystyle
+
\min\left\{\frac{6}{30},\frac{13}{30},\frac{1}{30}\right\}
\times
\min\left\{\frac{1}{30},\frac{19}{30},\frac{7}{30}\right\}\\[2mm]
\displaystyle
+
\min\left\{\frac{6}{30},\frac{27}{30},\frac{1}{30}\right\}
\times
\min\left\{\frac{1}{30},\frac{8}{30},\frac{1}{30}\right\}
\end{array}
\right\}
\\
=&\min \left\{
\begin{array}{c}
0.507,\\
\\
\displaystyle\frac{16}{30}\frac{8}{30} + \frac{1}{30}\frac{3}{30} + \frac{2}{30}\frac{1}{30} + \frac{1}{30}\frac{1}{30} + \frac{1}{30}\frac{1}{30} \\
\\
\displaystyle+ \frac{1}{30}\frac{1}{30}
\end{array}
\right\}\\
=&0.151.
\end{align*}

Hence, we can conclude:
\begin{align*}
0 \leq \mathrm{PNS}(3)\ \leq 0.151.
\end{align*}

\subsection{Simulation Algorithm}
We used the following algorithm to generate samples and conduct the
simulations. $Q$ selects the causal quantity; $G$ and $Z$ select the structure.
Every theorem in the paper is one choice of $(Q, G, Z)$. Following \citet{mueller2021causeseffectslearningindividual}, we use the same procedure to generate the conditional probability tables.
\begin{algorithm}[!b]
\caption{Generate simulation data for \(Q\) under \(G\)}
\label{alg:simulation}

\textbf{Input}: Number of output samples \(n\)\\
\phantom{\textbf{Input}:} Causal diagram \(G\)\\
\phantom{\textbf{Input}:} Variable to condition on \(Z\)\\
\phantom{\textbf{Input}:} Quantity \(Q\), where\\
\phantom{\textbf{Input}:}
\(\quad Q\in\{\mathrm{PNS}(k),\mathrm{PSub}(k,p),\)\\
\phantom{\textbf{Input}:}
\(\qquad\quad \mathrm{PRep}(k,q),\mathrm{PN}(k,p,q)\}\)\\
\textbf{Output}: List of 4-tuples containing the general\\
\phantom{\textbf{Output}:} lower bound, graph-based lower bound,\\
\phantom{\textbf{Output}:} graph-based upper bound, and general\\
\phantom{\textbf{Output}:} upper bound

\begin{algorithmic}[1]

\FOR{\(i\gets1\) to \(n\)}

    \STATE
    \(\mathrm{cpt}\gets
    \textsc{generate-cpt}(G,\operatorname{Exp}(1))\)

    \STATE
    \textit{// General (Shu et al.) lower and upper bounds}

    \STATE
    \((\mathrm{lb},\mathrm{ub})
    \gets
    \textsc{q-bounds}(Q,\mathrm{cpt})\)

    \STATE
    \textit{// Graph-based lower and upper bounds}

    \STATE
    \((\mathrm{lb}_G,\mathrm{ub}_G)
    \gets
    \textsc{q-graph}(Q,\mathrm{cpt},G,Z)\)

    \STATE
    \(\textsc{append-result}
    (\mathrm{lb},\mathrm{lb}_G,\mathrm{ub}_G,\mathrm{ub})\)

\ENDFOR

\STATE \textbf{return} output list

\end{algorithmic}
\end{algorithm}

\end{document}